\documentclass{article} 
\usepackage{iclr2027_conference,times}

\usepackage[utf8]{inputenc} 
\usepackage[T1]{fontenc}    
\usepackage[hidelinks]{hyperref}       
\usepackage{url}            
\usepackage{booktabs}       
\usepackage{multirow}       
\usepackage{enumitem}       
\usepackage{amsfonts}       
\usepackage{nicefrac}       
\usepackage{microtype}      
\usepackage{xcolor}         
\usepackage{graphicx}       

\usepackage{amsmath}
\usepackage{amsthm}
\usepackage{algorithm}
\usepackage{algpseudocode}
\usepackage{cleveref}
\usepackage{glossaries}
\glsdisablehyper
\AddToHook{cmd/appendix/before}{%
    \crefalias{section}{appendix}%
    \crefalias{subsection}{appendix}
}

\newcommand{\drift}{b}
\newcommand{\priordrift}{h}
\newcommand{\postdrift}{f}
\newcommand{\odedrift}{\bar f}
\newcommand{\vol}{g}
\newcommand{\flow}{\Psi}
\newcommand{\obs}{X}
\newcommand{\gauge}{\Omega}

\newcommand{\tr}{\operatorname{tr}}
\newcommand{\gl}{\operatorname{GL}}

\newcommand{\skewset}{\mathfrak{so}}
\newcommand{\ltri}{\operatorname{ltri}}
\newcommand{\KL}[2]{D_{\mathrm{KL}}\left(#1 \,\middle\|\, #2\right)}
\newcommand{\E}{\mathbb{E}}
\newcommand{\loss}{\mathcal{L}}

\newcommand{\diff}{{\mathrm{d}}}

\newcommand{\Norm}{\mathcal{N}}
\newcommand{\reals}{\mathbb{R}}

\newacronym{ode}{ODE}{Ordinary Differential Equation}
\newacronym{sde}{SDE}{Stochastic Differential Equation}
\newacronym{pde}{PDE}{Partial Differential Equation}
\newacronym{fpk}{FPK}{Fokker--Planck--Kolmogorov}
\newacronym{vi}{VI}{Variational Inference}
\newacronym{elbo}{ELBO}{Evidence Lower Bound}
\newacronym{kl}{KL}{Kullback--Leibler}

\newtheorem{theorem}{Theorem}

\newtheorem{proposition}[theorem]{Proposition}

\newtheorem{corollary}[theorem]{Corollary}

\crefname{theorem}{Theorem}{Theorems}
\crefname{proposition}{Proposition}{Propositions}
\crefname{lemma}{Lemma}{Lemmas}
\crefname{corollary}{Corollary}{Corollaries}
\crefname{definition}{Definition}{Definitions}
\crefname{remark}{Remark}{Remarks}
\AddToHook{env/proposition/begin}{\crefalias{theorem}{proposition}}
\AddToHook{env/lemma/begin}{\crefalias{theorem}{lemma}}
\AddToHook{env/corollary/begin}{\crefalias{theorem}{corollary}}
\AddToHook{env/definition/begin}{\crefalias{theorem}{definition}}
\AddToHook{env/remark/begin}{\crefalias{theorem}{remark}}

\title{Gaussian Flow Dynamics: \\ Simulation-Free Neural SDE Learning Beyond One-Time Marginals}

\author{%
  Grigory Bartosh \\
  Informatics Institute \\
  University of Amsterdam \\
  \texttt{g.bartosh@uva.nl} \\
  \And
  Christian A. Naesseth \\
  Informatics Institute \\
  University of Amsterdam \\
  \texttt{christian.a.naesseth@uva.nl} \\
}

\newif\ifpreprint
\preprinttrue 
\ifpreprint\iclrfinalcopy\fi
\begin{document}

\maketitle
\ifpreprint\lhead{}\renewcommand{\headrulewidth}{0pt}\vspace{-\baselineskip}\fi

\begin{abstract}
Simulation-free training of latent \glspl{sde} relies on a variational posterior process whose one-time marginals are tractable, typically Gaussian. Such marginals, however, do not determine the underlying dynamics: many processes share the same marginals while differing in their temporal structure, and existing parameterizations fix this structure implicitly, which restricts the posterior family and biases the learned model. We introduce Gaussian flow dynamics, which construct stochastic processes directly from smoothly evolving Gaussian marginals while making the marginal-preserving, or gauge, degrees of freedom explicit and parameterizable. The construction admits state-dependent diffusion coefficients and recovers every linear \gls{sde} with additive noise and a non-degenerate Gaussian initial distribution. Building on it, we propose Gauge Matching, a simulation-free method for latent \gls{sde} learning that combines Gaussian flow dynamics with the \gls{sde} Matching objective. Gauge Matching costs at most quadratically in the latent dimension per step, like \gls{sde} Matching, but learns the temporal structure of the posterior beyond its one-time marginals. It comes within a nat of Helmholtz-\gls{sde}, which computes the gauge from the prior Jacobian at cubic cost, on the linear benchmark where the exact posterior is known, matches it on nonlinear systems, and applies where Helmholtz-\gls{sde} does not, to state-dependent noise.
\end{abstract}
\glsresetall
\ifpreprint\vspace{-\baselineskip}\fi

\section{Introduction}

Neural or latent \glspl{sde} model a time series as noisy observations of an unobserved state that evolves according to a learnable \gls{sde} \citep{movellan2002monte,tzen2019neural,hodgkinson2020stochastic,li2020scalable}, a natural choice for irregularly sampled, inherently stochastic dynamics, from neuroscience to fluid dynamics \citep{course2023state,smith2026closing}. Training them requires the posterior over latent trajectories, which \gls{vi} approximates with another \gls{sde}, the \emph{posterior process}, whose drift depends on the data. Simulation-based methods integrate this process numerically at every gradient step \citep{li2020scalable}; simulation-free methods \citep{course2023amortized,bartosh2025sde} instead parameterize its one-time marginals, which can be sampled at any time, so that the objective becomes an expectation over a random time and state and a gradient step costs the same for any length of the series.

The price of this efficiency is that the marginals underdetermine the process, and a method that parameterizes marginals still commits to one of the many processes that share them, implicitly. For Gaussian marginals, \citet{course2023amortized,course2023state} take the affine drift whose linear part is symmetric, and \gls{sde} Matching \citep{bartosh2025sde} takes the drift induced by a particular square root of the covariance. \citet{smith2026closing} recently showed that this choice, the \emph{gauge}, is not benign: a fixed gauge restricts the variational family to a strict subset of the processes available to simulation-based methods, which degrades both inference and the learned dynamics when the posterior is uncertain. Their remedy, Helmholtz-\gls{sde}, keeps the marginals of \gls{sde} Matching and adds to its drift the marginal-preserving component of the difference between the prior and posterior drifts; the correction needs the Jacobian of the prior drift, costs at least cubically in the latent dimension, and requires state-independent noise.

We make the gauge part of the parameterization. We introduce \emph{Gaussian flow dynamics}, which build a stochastic process from a prescribed path of Gaussian marginals. Factorizing the square root of the covariance into a Cholesky factor and an orthogonal matrix separates the marginals from the gauge, which enters the dynamics only through a skew-symmetric generator that can be parameterized directly, at only quadratic cost and with state-dependent noise. We prove that the family is complete: it recovers every linear \gls{sde} with state-independent noise and a non-degenerate Gaussian initial distribution, and every marginal-preserving affine field is generated by such a skew-symmetric matrix. Building on this, we propose \emph{Gauge Matching}, which combines Gaussian flow dynamics with the \gls{sde} Matching objective, so that the gauge is selected by the objective rather than by the parameterization. We characterize the selected gauge, in closed form for state-independent noise, and show that Helmholtz-\gls{sde} computes it for a linearized prior process with state-independent diffusion. Experimentally, the learned gauge removes most of the gap of the fixed gauges to the true posterior process in the linear-Gaussian (conjugate) setting, where that gap is available in closed form, matches Helmholtz-\gls{sde} on nonlinear systems at lower cost, improves on it when the prior drift is known and its linearization is poor, and extends to state-dependent noise, where Helmholtz-\gls{sde} does not apply. \Cref{sec:relatedwork} discusses related work.

\section{Background}

\subsection{Dynamics and Their Marginals}
Let $z_t$ solve the \gls{ode} $\diff z_t = \drift(t,z_t) \diff t$ on $[0, T]$ with \emph{vector field} $\drift : [0,T] \times \reals^d \to \reals^d$. Its \emph{flow map} sends an initial condition $z_0$ to the solution $z_t$ at time $t$; conversely, a continuously differentiable flow map $\Phi : [0,T] \times \reals^d \to \reals^d$ with $\Phi(t, \cdot)$ invertible for every $t$ determines its vector field: the trajectories $z_t = \Phi(t, z_0)$ solve the \gls{ode} with
\begin{align}
    \drift(t, z) &= \partial_t \Phi(t, z_0) \big|_{z_0 = \Phi^{-1}(t, z)},
    \label{eq:flowvectorfield}
\end{align}
where the inverse is taken with respect to the second argument \citep[Fundamental Theorem on Time-Dependent Flows, Theorem 9.48]{lee2012smooth}. If the initial condition is random with density $p(0, z)$, then under regularity assumptions the solution has a density $p(t, z)$ satisfying the continuity equation
\begin{align}
    \partial_t p(t, z) &= -\nabla \cdot [\drift(t,z) p(t, z)],
    \label{eq:continuity}
\end{align}
where $\nabla \cdot v = \sum_{i=1}^d \partial_{z^i} v^i$ is the divergence of a vector field $v = (v^1, \ldots, v^d)^\top$. Adding noise gives a diffusion process $Z_t$, defined by the \gls{sde} $\diff Z_t = \drift(t,Z_t) \diff t + \vol(t,Z_t)\diff W_t$ with $Z_0 \sim p(0, z)$, \emph{drift coefficient} $\drift$, \emph{diffusion coefficient} $\vol : [0,T] \times \reals^d \to \reals^{d\times d'}$ and a $d'$-dimensional Brownian motion $W_t$. Its marginal densities satisfy the \gls{fpk} equation
\begin{align}
    \partial_t p(t, z) = -\nabla \cdot [\drift(t,z) p(t, z) ] + \frac{1}{2} \nabla \cdot\nabla \cdot[\vol(t,z) \vol(t,z)^\top p(t, z)] ,
    \label{eq:fpk}
\end{align}
where the divergence of a matrix-valued function is taken column-wise.

\subsection{Simulation-Free Latent SDE Learning}
\label{sec:simfree}

\paragraph{Latent SDEs}
A latent \gls{sde} model explains observations $\obs = (x_{t_1}, \ldots, x_{t_N})$ at times $0 \le t_1 < \cdots < t_N \le T$ through an unobserved latent process: the \emph{prior process} is a diffusion, and the observations are conditionally independent given the latent path, $x_{t_i} \sim p_\theta(x \mid z_{t_i})$. Since the marginal likelihood $p_\theta(\obs)$ is intractable, the model is trained with \gls{vi}, using a \emph{posterior process} whose drift coefficient depends on the observations but whose diffusion coefficient is shared with the prior process,
\begin{align}
    \text{prior:}\quad \diff Z_t &= \priordrift_\theta(t, Z_t) \diff t + \vol_\theta(t, Z_t) \diff W_t, \quad Z_0 \sim p_\theta(0, z), \label{eq:prior} \\
    \text{posterior:}\quad \diff Z_t &= \postdrift_\phi(t, Z_t, \obs) \diff t + \vol_\theta(t, Z_t) \diff W_t, \quad Z_0 \sim q_\phi(0, z \mid \obs). \label{eq:posterior}
\end{align}
Writing $q_\phi(t, z \mid \obs)$ for the marginal density of the posterior process, the negative \gls{elbo} is
\begin{align}
    -\log p_\theta(\obs) \le \loss &= \KL{q_\phi(0, z \mid \obs)}{p_\theta(0, z)} + \frac{1}{2} \int_0^T \E_{q_\phi(t, z \mid \obs)} \left[ \left\| r_{\theta, \phi}(t, Z_t, \obs) \right\|_2^2 \right] \diff t \nonumber \\
    &\quad - \sum_{i=1}^N \E_{q_\phi(t_i, z \mid \obs)} \left[ \log p_\theta(x_{t_i} \mid Z_{t_i}) \right], \label{eq:elbo}
\end{align}
where the residual $r_{\theta, \phi}$ is defined by $\vol_\theta(t, z) r_{\theta, \phi}(t, z, \obs) = \priordrift_\theta(t, z) - \postdrift_\phi(t, z, \obs)$; for simplicity we assume throughout that $\vol_\theta(t, z)$ is square and invertible, so that $r_{\theta, \phi}$ is unique. The first two terms together are the \gls{kl} divergence between the path measures of the posterior and prior processes, given by Girsanov's theorem \citep{oksendal2003stochastic} and finite because the two processes share a diffusion coefficient. We call the second term the \emph{path-space \gls{kl} term}; it is the only term of $\loss$ that depends on the posterior process through more than its marginals.

\paragraph{Simulation-Free Estimation}
If $q_\phi(t, z \mid \obs)$ can be sampled directly for any $t$, every term of \cref{eq:elbo} is an expectation that can be estimated by sampling a time $t$ uniformly on $[0, T]$ and a state $Z_t \sim q_\phi(t, z \mid \obs)$, without simulating the posterior process \citep{course2023amortized,bartosh2025sde}. \gls{sde} Matching \citep{bartosh2025sde} achieves this in three steps: (i) define the marginals through a map of noise, $Z_t = F_\phi(t, \varepsilon, \obs)$ with $\varepsilon \sim \Norm(0, I)$; (ii) read $F_\phi$ as a flow map with $\varepsilon$ in the role of $z_0$, and obtain from \cref{eq:flowvectorfield} the vector field $\odedrift_\phi$ of the \gls{ode} that generates these marginals; (iii) convert the \gls{ode} into an \gls{sde} with the diffusion coefficient of the prior process, preserving the marginals \citep{song2021score},
\begin{align}
    \postdrift_{\phi, \theta}(t, z, \obs) &= \odedrift_\phi(t, z, \obs) + \frac{1}{2} \Sigma_\theta(t, z) \nabla_z \log q_\phi(t, z \mid \obs) + \frac{1}{2} \nabla \cdot \Sigma_\theta(t, z),
    \label{eq:odetosde}
\end{align}
where $\Sigma_\theta = \vol_\theta \vol_\theta^\top$; we write $\postdrift_\phi$ for short when the dependence on $\theta$ is not needed. \citet{bartosh2025sde} take $F_\phi$ affine in $\varepsilon$, so that the posterior marginals are Gaussian. The marginals, however, fix the posterior process only up to a marginal-preserving vector field, so this construction implicitly selects one compatible drift among many; making the selection explicit is the subject of the next section.

\section{Gaussian Flow Dynamics}
\label{sec:gaussianflowdynamics}

\emph{Gaussian flow dynamics} are deterministic and stochastic dynamics constructed directly from a time-dependent affine transformation of a standard Gaussian, so that their marginals are the Gaussian laws it defines; the flow-map perspective exposes the degrees of freedom that are invisible from the marginals alone. Throughout this section we write $q$ for the marginal densities, anticipating their role as posterior marginals in \cref{sec:gaugematching}, and suppress observations and parameters.

\paragraph{Gaussian Flow}
We define the \emph{Gaussian flow map} $\flow$ as an affine transformation of Gaussian noise,
\begin{align}
    \flow(t, \varepsilon) &= m(t)+S(t)\varepsilon, \quad \varepsilon \sim \Norm(0,I),
    \label{eq:flowmap}
\end{align}
for continuously differentiable $m : [0,T] \to \reals^d$ and invertible, continuously differentiable ${S : [0,T] \to \reals^{d\times d}}$. The random variable $Z_t = \flow(t, \varepsilon)$ has density
\begin{align}
    q(t, z) = \Norm(z \mid m(t), S(t)S(t)^\top),
    \label{eq:gaussianmarginals}
\end{align}
and by \cref{eq:flowvectorfield} its trajectories solve the \gls{ode}
\begin{align}
    \diff Z_t &= \odedrift(t,Z_t) \diff t, \qquad
    \odedrift(t,z) = \dot m(t)+\dot S(t) S(t)^{-1}(z-m(t)),
    \label{eq:gaussianflow}
\end{align}
with $Z_0 \sim \Norm(m(0),S(0)S(0)^\top)$, which we call the \emph{Gaussian flow}. Its vector field is affine in $z$, and its marginals \cref{eq:gaussianmarginals} solve the continuity equation \cref{eq:continuity} for this vector field (\cref{sec:proofcontinuity}).

\paragraph{Gaussian Flow Diffusion}
We define the \emph{Gaussian flow diffusion} as the \gls{sde} induced by $\flow$ and a diffusion coefficient $\vol(t,z)$ via the conversion \cref{eq:odetosde},
\begin{align}
    \diff Z_t &= \postdrift(t,Z_t) \diff t+ \vol(t,Z_t)\diff W_t, \quad Z_0 \sim \Norm(m(0),S(0)S(0)^\top),
    \label{eq:gaussianflowdiffusion} \\
    \postdrift(t,z) &= \dot m(t)+\left(\dot S(t)S(t)^{-1} -\frac{1}{2}\Sigma(t,z) S(t)^{-\top}S(t)^{-1}\right)(z-m(t))+\frac{1}{2}\nabla \cdot \Sigma(t, z), \nonumber
\end{align}
with $\Sigma(t, z) = \vol(t,z)\vol(t,z)^\top$. Its marginals are again the Gaussians \cref{eq:gaussianmarginals}, which solve the \gls{fpk} equation \cref{eq:fpk} for this drift and diffusion coefficient (\cref{sec:proofFPK}); it is the posterior process that \gls{sde} Matching associates with the affine map $F = \flow$. When $\Sigma(t, z) \equiv \Sigma(t)$ is state-independent, the drift coefficient is affine in $z$ and the process is Gaussian; for a state-dependent diffusion coefficient its joint distributions are in general not Gaussian, although its marginals are.

\paragraph{Representation of Linear SDEs}
Conversely, every linear \gls{sde} with state-independent noise,
\begin{align}
    \diff Z_t &= \left(A(t) Z_t + u(t)\right) \diff t+ \vol(t)\diff W_t, \qquad Z_0 \sim \Norm(m_0, S_0 S_0^\top), \label{eq:linearSDE}
\end{align}
is a Gaussian flow diffusion. We write $\gl(d)$ for the group of invertible $d \times d$ matrices.
\begin{theorem}[Representation of Linear \glspl{sde}]
    \label{thm:representation}
    Let $A$, $u$, $\vol$ in \cref{eq:linearSDE} be continuous on $[0,T]$ and $S_0\in \gl(d)$. Then there exist unique continuously differentiable $m : [0, T] \to \reals^d$ and $S : [0, T] \to \gl(d)$ with $m(0)=m_0$ and $S(0)=S_0$ such that the Gaussian flow diffusion \cref{eq:gaussianflowdiffusion} with flow map $m(t) + S(t) \varepsilon$ and diffusion coefficient $\vol$ coincides with \cref{eq:linearSDE} (proof in \cref{sec:proofgaussian}).
\end{theorem}
Replacing $S_0$ by another square root $S_0 R$ of the initial covariance only relabels the trajectories, since $\flow(t, R \varepsilon)$ defines the same process; up to this relabeling, the representation is unique.

\paragraph{Marginals and Gauge}
By \cref{thm:representation}, distinct linear \glspl{sde} with the same marginal covariance path are represented by distinct square roots $S(t)$: the square root carries more information than the covariance. To make this explicit, we factorize
\begin{align}
    S(t) &= L(t) O(t),
    \label{eq:factorization}
\end{align}
where $L(t)$ is the Cholesky factor of the covariance, lower triangular with positive diagonal, and $O(t)$ is orthogonal; both inherit continuous differentiability from $S(t)$. The pair $(m, L)$ determines the marginals, while $O$ leaves them invariant. To see how $O$ enters the dynamics, define $\gauge(t) = \dot O(t) O(t)^\top$, which is skew-symmetric because $O O^\top \equiv I$; we call it the \emph{gauge generator}. Then $\dot S S^{-1} = \dot L L^{-1} + L \gauge L^{-1}$, and since the drift coefficients in \cref{eq:gaussianflow,eq:gaussianflowdiffusion} depend on $S$ only through $\dot S S^{-1}$ and $S S^\top = L L^\top$, the orthogonal factor enters the dynamics only through the affine vector field
\begin{align}
    v(t, z) &= L(t) \gauge(t) L(t)^{-1} (z - m(t)),
    \label{eq:gaugevectorfield}
\end{align}
added to the drift coefficient of the process with the same marginals and $O \equiv I$. Since the \gls{fpk} equation \cref{eq:fpk} depends on the drift coefficient only through $\nabla \cdot [\postdrift(t, z) q(t, z)]$, adding a field $v$ to the drift coefficient leaves the marginals unchanged if and only if $v$ is \emph{$q$-divergence-free}, $\nabla \cdot [v(t, z) q(t, z)] = 0$ \citep{smith2026closing}; the fields \cref{eq:gaugevectorfield} are exactly the affine ones.
\begin{proposition}[Affine Gauge Fields]
\label{prop:gauge}
Let $q(t, z) = \Norm(z \mid m(t), L(t) L(t)^\top)$ with $L(t)$ invertible. An affine vector field $v(t, z) = M(t) z + c(t)$ is $q$-divergence-free if and only if $v(t, z) = L(t) \gauge(t) L(t)^{-1} (z - m(t))$ for a skew-symmetric $\gauge(t)$, in which case $\gauge(t) = L(t)^{-1} M(t) L(t)$ (proof in \cref{sec:proofgauge}; the Lyapunov form of this statement is Proposition 3 of \citealp{smith2026closing}).
\end{proposition}
Consequently, for fixed marginals \cref{eq:gaussianmarginals} and diffusion coefficient $\vol$, letting $O$ range over all continuously differentiable orthogonal functions of time produces exactly the \glspl{sde} with these marginals and this $\vol$ whose drift differs from that of $O \equiv I$ by an affine field; for state-independent $\vol$, these are all the Gaussian processes with these marginals and this $\vol$ (\cref{thm:representation}). Following \citet{smith2026closing}, we call $\gauge$ the \emph{gauge}. Existing methods commit to a particular $\gauge$; we call their choices \emph{fixed gauges}: the \emph{square root gauge} of \gls{sde} Matching \citep{bartosh2025sde}, the \emph{symmetric gauge} of \citet{course2023amortized}, and the simplest, $O \equiv I$, that is, $\gauge \equiv 0$, which we call the \emph{Cholesky gauge} (\cref{sec:fixedgauges}).

\paragraph{Parameterization}
The dynamics depend on $O$ only through $\gauge$, and sampling from the marginals requires only $m$ and $L$, so the orthogonal factor never has to be formed: we parameterize Gaussian flow dynamics directly by $m(t)$, $L(t)$ and $\gauge(t)$, with neural networks that output $d^2 + d$ numbers, the entries of $m$, the nonzero entries of $L$ with a positive transform on its diagonal, and the entries of $\gauge$ above its diagonal. In the whitened state $\eta = L(t)^{-1}(z - m(t))$, the drift coefficient \cref{eq:gaussianflowdiffusion} reads
\begin{align}
    \postdrift(t, z) &= \dot m(t) + \dot L(t) \eta + L(t) \gauge(t) \eta - \frac{1}{2} \Sigma(t, z) L(t)^{-\top} \eta + \frac{1}{2} \nabla \cdot \Sigma(t, z).
    \label{eq:driftwhitened}
\end{align}
Its evaluation requires $\dot m$ and $\dot L$ (from forward-mode differentiation), two triangular solves and a few matrix--vector products, $\mathcal{O}(d^2)$ in total, also for state-dependent diffusion coefficients whenever $\Sigma(t, z) w$ and $\nabla \cdot \Sigma(t, z)$ cost $\mathcal{O}(d^2)$, as for the diagonal structures common in latent \gls{sde} models \citep{li2020scalable,bartosh2025sde}. Alternatively, $O(t)$ itself can be parameterized, for instance as a Householder flow \citep{tomczak2016improving}; see \cref{sec:fixedgauges}.

\section{Gauge Matching}
\label{sec:gaugematching}

\emph{Gauge Matching} is \gls{sde} Matching with a Gaussian flow diffusion as posterior process, so that its marginals and its gauge are both learned, while the objective \cref{eq:elbo} and its simulation-free estimator are unchanged. After specifying the method, we characterize the gauge that the objective selects and relate it to the correction of \citet{smith2026closing}.

\paragraph{Posterior Process and Training}
Given observations $\obs$, the posterior process is the Gaussian flow diffusion \cref{eq:gaussianflowdiffusion} with the diffusion coefficient $\vol_\theta$ of the prior process, parameterized as in \cref{sec:gaussianflowdynamics} by a mean $m_\phi(t, \obs)$, a Cholesky factor $L_\phi(t, \obs)$ and a gauge generator $\gauge_\phi(t, \obs)$, neural networks of $t$ and a representation of $\obs$ such as the output of a sequence model \citep{bartosh2025sde}; its marginals are $q_\phi(t, z \mid \obs) = \Norm(z \mid m_\phi, L_\phi L_\phi^\top)$ and its drift coefficient is \cref{eq:driftwhitened} with $\Sigma_\theta = \vol_\theta \vol_\theta^\top$. \gls{sde} Matching is the member of this family with diagonal $L_\phi$, for which its square root gauge has $\gauge_\phi \equiv 0$; the fixed gauges of \cref{sec:fixedgauges} prescribe $\gauge_\phi$ as a function of $L_\phi$, $\dot L_\phi$ and $\Sigma_\theta$; Gauge Matching leaves $\gauge_\phi$ free. Training minimizes \cref{eq:elbo} jointly over $\theta$ and $\phi$ with the simulation-free estimator of \cref{sec:simfree}: the path-space \gls{kl} term is estimated at a single time $t \sim \mathcal{U}[0, T]$ and the reconstruction term at a single index $i \sim \mathcal{U}\{1, \ldots, N\}$ (\cref{alg:gaugematching} in \cref{sec:expdetails}). Sampling $Z_t = m_\phi + L_\phi \varepsilon$ with $\varepsilon \sim \Norm(0, I)$ makes the whitened state in \cref{eq:driftwhitened} simply $\eta = \varepsilon$, and an iteration costs the network evaluations, one forward-mode derivative of $m_\phi$ and $L_\phi$ in $t$, and $\mathcal{O}(d^2)$ linear algebra, whatever the number of observations and the length of the time interval.

\paragraph{The Optimal Gauge}
The gauge enters \cref{eq:elbo} only through the path-space \gls{kl} term; this term alone therefore selects the posterior process among those compatible with the learned marginals. To see which, fix the prior process and the marginals. The integrand at time $t$ depends on the gauge only through $\gauge_\phi(t, \obs)$, so the minimization decouples across time and we study it at a fixed $t$, suppressing observations and parameters as in \cref{sec:gaussianflowdynamics}. Let $\postdrift_0$ denote the drift coefficient of the Cholesky gauge; by \cref{eq:gaugevectorfield}, the posterior process with gauge $\gauge$ has drift coefficient $\postdrift_0 + L \gauge L^{-1}(z - m)$, and with $\| v \|_{\Sigma^{-1}}^2 = v^\top \Sigma^{-1} v$ the squared residual in \cref{eq:elbo} is $\| r \|_2^2 = \| \priordrift - \postdrift \|_{\Sigma^{-1}}^2$, a quadratic function of $\gauge$ whose minimizer can be found explicitly.
\begin{proposition}[Optimal Gauge]
\label{prop:optimalgauge}
Fix $t \in [0, T]$ and $m$, $L$, $\dot m$, $\dot L$ at $t$. Let $\priordrift(t, \cdot)$ and $\vol(t, \cdot)$ be continuously differentiable, with $\priordrift$, its Jacobian $\nabla_z \priordrift$ (entries $\partial \priordrift_i / \partial z_j$), $\Sigma = \vol \vol^\top$ and $\nabla \cdot \Sigma$ of at most polynomial growth and $\Sigma$ uniformly positive definite. The integrand of the path-space \gls{kl} term,
\begin{align}
    \mathcal{J}(\gauge) &= \frac{1}{2} \E_{q(t, z)} \left[ \left\| \priordrift(t, Z_t) - \postdrift_0(t, Z_t) - L \gauge L^{-1}(Z_t - m) \right\|_{\Sigma(t, Z_t)^{-1}}^2 \right], \label{eq:Jgauge}
\end{align}
is a strictly convex quadratic function of the skew-symmetric matrix $\gauge$ and has a unique minimizer $\gauge^\star$, the \emph{optimal gauge}. If $\Sigma$ is state-independent, $\gauge^\star$ is the unique solution of
\begin{align}
    G \gauge^\star + \gauge^\star G &= L^\top \left( \Sigma^{-1} \bar J - \bar J^\top \Sigma^{-1} \right) L, \quad G = L^\top \Sigma^{-1} L, \quad \bar J = \E_{q(t, z)} \left[ \nabla_z \priordrift(t, Z_t) \right] - \dot L L^{-1}.
    \label{eq:optimalgauge}
\end{align}
\end{proposition}
The proof is in \cref{sec:proofoptimalgauge}. To interpret \cref{eq:optimalgauge}, note that $\bar J$ is, up to a term that cancels on the right-hand side, the Jacobian of the residual $\priordrift - \postdrift_0$ averaged over the posterior marginal, and that the right-hand side is the skew-symmetric part of $\Sigma^{-1} \bar J$ transformed by $L$: the objective selects the gauge that reproduces the rotational component of the prior dynamics as seen from the posterior marginals, which the marginals themselves cannot represent. For the Ornstein--Uhlenbeck spiral of \cref{sec:exp:gauge}, the stationary linear \gls{sde} with drift $(-\alpha I + \omega J) z$ and $J = \bigl(\begin{smallmatrix} 0 & -1 \\ 1 & 0 \end{smallmatrix}\bigr)$, \cref{eq:optimalgauge} gives $\gauge^\star = \omega J$, for which the posterior process coincides with the prior process. The same computation tells us when a fixed gauge is already optimal (\cref{sec:optimalgaugeconsequences}): the Cholesky gauge if and only if $\Sigma^{-1} \bar J$ is symmetric, and, for isotropic $\Sigma$, the symmetric gauge if and only if $\E_{q(t, z)}[\nabla_z \priordrift(t, Z_t)]$ is symmetric, as for a gradient prior drift; neither holds for prior dynamics with a rotational component.

\paragraph{Relation to Helmholtz-SDE}
Helmholtz-\gls{sde} \citep{smith2026closing} starts from a fixed-gauge posterior process with a state-independent diffusion coefficient and adds to its drift coefficient the $q$-divergence-free component of the residual $\priordrift - \postdrift$, which minimizes the path-space \gls{kl} term over all marginal-preserving fields; to make it computable, the residual is replaced by its Taylor expansion about the mean $m(t)$. At order $\ell = 1$, used in all their experiments, the correction is affine, hence a gauge field by \cref{prop:gauge}, and the corrected gauge is, irrespective of the fixed gauge it starts from, the solution of \cref{eq:optimalgauge} with the Jacobian at the mean, $\nabla_z \priordrift(t, m)$, in place of its expectation (\cref{sec:optimalgaugeconsequences}). The two gauges coincide whenever $\nabla_z \priordrift$ is affine in $z$, as for the quadratic Lorenz drift, and differ otherwise: Helmholtz-\gls{sde} finds the optimal gauge of the linearized prior process, whereas Gauge Matching minimizes the exact objective over the same family; only corrections of order $\ell \ge 2$ leave it, and a state-dependent gauge generator recovers those as well, still simulation-free (\cref{sec:statedependentgauge}). The larger difference is computational. Solving \cref{eq:optimalgauge} at every evaluation requires the Jacobian of the prior drift network and $\mathcal{O}(d^3)$ linear algebra, and for a state-dependent diffusion coefficient it becomes a linear system whose coefficients must be estimated numerically (\cref{sec:optimalgaugeconsequences}). Gauge Matching learns $\gauge_\phi$ by stochastic gradient descent, at $\mathcal{O}(d^2)$, without derivatives of the prior drift coefficient and for state-dependent diffusion coefficients (\cref{tab:methods}). Since $\mathcal{J}$ is strictly convex in $\gauge$, the optimal gauge is a unique target at every $(t, \obs)$; $\gauge_\phi$ may still have local optima in $\phi$, as any network does, but in \cref{sec:experiments} gradient descent reaches the optimal gauge.

\begin{table}[t]
    \centering
    \caption{Posterior processes with Gaussian marginals in simulation-free latent \gls{sde} learning. Cost is that of evaluating the drift coefficient at one $(t, z)$ beyond the network evaluations, for a full $d \times d$ covariance (\gls{sde} Matching: diagonal or triangular $S$). The Cholesky gauge ($\gauge \equiv 0$) is Gauge Matching without the gauge.}
    \label{tab:methods}
    \small
    \begin{tabular}{llll}
        \toprule
        Method & Gauge & Cost & State-dependent $\vol$ \\
        \midrule
        \citet{course2023amortized,course2023state} & symmetric, fixed & $\mathcal{O}(d^3)$ & no \\
        \gls{sde} Matching \citep{bartosh2025sde} & square root, fixed & $\mathcal{O}(d^2)$ & yes \\
        Cholesky gauge (ablation) & $\gauge \equiv 0$, fixed & $\mathcal{O}(d^2)$ & yes \\
        Helmholtz-\gls{sde} \citep{smith2026closing} & order-one correction, computed & $\mathcal{O}(d^3)$ & no \\
        Gauge Matching (ours) & learned & $\mathcal{O}(d^2)$ & yes \\
        \bottomrule
    \end{tabular}
\end{table}

\paragraph{Exactness in the Conjugate Setting}
When the prior process is linear and the observation model is linear-Gaussian, the true posterior process is a Gaussian process, the Kalman--Bucy smoother \citep{Särkkä_Solin_2019}, which the square root and symmetric gauges cannot represent \citep{smith2026closing}. The family with a learned gauge can.
\begin{corollary}[Conjugate Setting]
\label{cor:conjugate}
Let the prior process be the linear \gls{sde} \cref{eq:linearSDE} with $S_0 \in \gl(d)$ and $\Sigma(t)$ positive definite, and let $p_\theta(x \mid z) = \Norm(x \mid C z, R)$ with $R$ positive definite. Then (i) between consecutive observation times, the true posterior process is a Gaussian flow diffusion with diffusion coefficient $\vol$, continuously differentiable $m$, $L$ and continuous $\gauge$; and (ii) the infimum of the \gls{kl} divergence from Gauge Matching posterior processes to the true one is zero (proof in \cref{sec:proofconjugate}).
\end{corollary}
At each observation time, $m$ and $L$ of the true posterior process have kinks and $\gauge$ jumps, whereas a neural network in $t$ is continuously differentiable; the infimum is therefore approached rather than attained, as for Helmholtz-\gls{sde} \citep[Theorem 1]{smith2026closing}.

\section{Experiments}
\label{sec:experiments}

We ask whether gradient descent on the path-space \gls{kl} term finds the optimal gauge of \cref{prop:optimalgauge}; whether the learned gauge closes the gap of the fixed gauges in the conjugate setting and how it compares with the exact gauge of Helmholtz-\gls{sde} there; whether it matches Helmholtz-\gls{sde} on nonlinear systems and extends beyond its scope; and how it scales with $d$. All posterior processes share the \gls{sde} Matching architecture \citep{bartosh2025sde} and differ only in $L_\phi$ and the gauge: the square root gauge with a diagonal $L_\phi$, which is \gls{sde} Matching; the Cholesky ($\gauge_\phi \equiv 0$) and symmetric gauges with a full $L_\phi$; Helmholtz-\gls{sde} at order one, which we reimplement; and the learned gauge of Gauge Matching (\cref{sec:expdetails}). We call each independent observation sequence $\obs$ a \emph{trial}.

\begin{figure}[t]
    \centering
    \includegraphics[width=\linewidth]{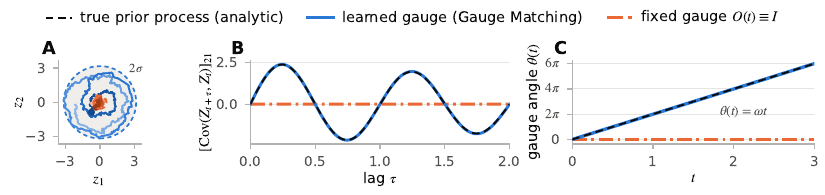}
    \caption{Learning the gauge of the stationary Ornstein--Uhlenbeck spiral ($\alpha = 0.2$, $\omega = 2\pi$) without observations. \emph{A:} one sample path of each fitted model, learned gauge (blue) and fixed gauge $O(t) \equiv I$ (orange), over the true $2\sigma$ disk (shaded) and each model's own $2\sigma$ circle (dashed): the fixed gauge cannot circulate and shrinks the marginals instead. \emph{B:} lagged cross-covariance $[\mathrm{Cov}(Z_{t+\tau}, Z_t)]_{21}$ of the true and the two fitted processes. \emph{C:} rotation angle $\theta(t)$ of the learned $O(t)$ against $\omega t$.}
    \label{fig:gauge}
\end{figure}

\begin{figure}[t]
    \centering
    \includegraphics[width=\linewidth]{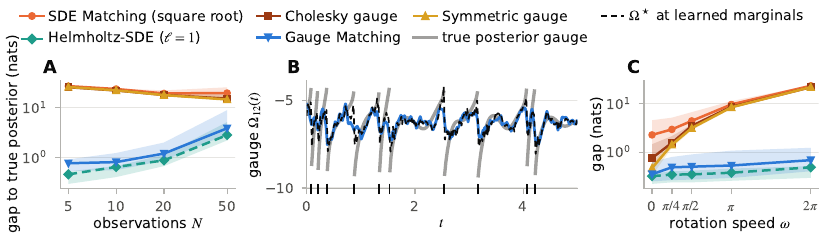}
    \caption{Gauge Matching in the conjugate setting ($\sigma = 0.3$; median and interquartile range over 16 trials). \emph{A:} \gls{kl} divergence to the true posterior process against the number of observations $N$ at $\omega = 2\pi$; the learned gauge is $4$--$34$ times below the fixed gauges and within a nat of Helmholtz-\gls{sde}, whose gauge is exact here. \emph{B:} learned gauge $\gauge_{12}(t)$ for one trial with $N = 10$, against the optimal gauge \cref{eq:optimalgauge} at the learned marginals and the gauge of the true posterior process, which jumps at the observation times (ticks). \emph{C:} \gls{kl} divergence against the rotation speed $\omega$ at $N = 10$.}
    \label{fig:conjugate}
\end{figure}

\subsection{Linear Dynamics}
\label{sec:exp:gauge}
\paragraph{Learning the Gauge}
We take as prior process the two-dimensional linear \gls{sde} \cref{eq:linearSDE} with $A = -\alpha I + \omega J$, $\Sigma = I$ and the stationary initial distribution $\Norm(0, I / (2 \alpha))$, the Ornstein--Uhlenbeck spiral of \citet{smith2026closing}. Its drift is a contracting gradient field plus a rotation that preserves the stationary law, so the marginals carry no information about $\omega$: the temporal structure lies in the gauge alone. Without observations, the posterior process should recover the prior process. We fit a Gaussian flow diffusion by minimizing the first two terms of \cref{eq:elbo}, with $O(t)$ a Householder flow with rotation angle $\theta(t)$. The learned gauge recovers the prior process: the marginals converge to the stationary law, $\theta(t)$ to $\omega t$, and the path-space \gls{kl} term drops to $0.007$ nats over $[0, 3]$ (\cref{fig:gauge}). The fixed gauge $O(t) \equiv I$, which represents all three fixed gauges here (\cref{sec:fixedgauges}), cannot rotate and shrinks the marginals to compensate, from the true variance $2.5$ to $0.080$, at a cost of about $18$ nats.
\paragraph{Conjugate Setting}
We next add observations $\Norm(x \mid C z + c, \sigma^2 I)$ with a random matrix $C$ and offset $c$ at $N$ random times in $[0, T]$, with $T = 5$ and $\alpha = 0.2$, and fix the generative model at the truth. Since all drift coefficients are affine, the \gls{kl} divergence of each fitted posterior process to the true posterior process, which we call the \emph{gap}, is available in closed form (\cref{sec:expdetails}). For $\omega = 2\pi$ and $\sigma = 0.3$ (\cref{fig:conjugate}A), the three fixed gauges incur a gap of $14$--$26$ nats, within about $2$ nats of the best their variational families can do; the Cholesky gauge improves little on the square root gauge, so it is the gauge, not the covariance, that is missing. Gauge Matching reduces the gap to $0.8$--$3.9$ nats. Helmholtz-\gls{sde}, whose gauge is exact here and costs nothing to compute (the prior Jacobian is the constant $A$), reaches $0.5$--$2.8$ nats with the same network and budget: knowing the optimal gauge is worth $0.2$--$1$ nat over learning it, and what remains for both is the error of a smooth posterior at the kinks of the true one (\cref{cor:conjugate}). The learned $\gauge_\phi(t)$ tracks the optimal gauge \cref{eq:optimalgauge} at the learned marginals and, up to the jumps, the gauge of the true posterior process (\cref{fig:conjugate}B); as $\omega$ grows (\cref{fig:conjugate}C), the gap of the fixed gauges rises from $0.6$--$2.3$ to $22$--$23$ nats, linearly in $\omega T$ \citep[Theorem 4]{smith2026closing}, while that of the learned gauge stays below $0.8$.
\begin{table}[t]
    \centering
    \caption{Noisy Lorenz attractor \citep{smith2026closing}, $d = 4$; mean $\pm 2$ standard errors over 5 seeds on held-out trials; arrows give the better direction. \emph{Negative \gls{elbo}}: per trial, nats. \emph{Latent LL}: log-density of the true state under the posterior predictive. \emph{Autocorr.\ error}: distance between the lagged correlation matrices of the learned and the true prior process. \emph{Gauge excess}: $\mathcal{J}(\gauge) - \mathcal{J}(\gauge^\star)$ of \cref{eq:Jgauge} at the method's own gauge and marginals, nats per unit time (\cref{sec:expdetails}).}
    \label{tab:lorenz}
    \small
    \setlength{\tabcolsep}{5pt}
    \begin{tabular}{lcccc}
        \toprule
        Posterior process & Negative \gls{elbo} $\downarrow$ & Latent LL $\uparrow$ & Autocorr.\ error $\downarrow$ & Gauge excess $\downarrow$ \\
        \midrule
        \gls{sde} Matching (square root) & $77.9 \pm 0.6$ & $-0.22 \pm 0.01$ & $0.134 \pm 0.009$ & $2.84 \pm 0.08$ \\
        Cholesky gauge & $75.5 \pm 0.8$ & $-0.20 \pm 0.01$ & $0.111 \pm 0.007$ & $2.52 \pm 0.09$ \\
        Symmetric gauge & $76.9 \pm 0.8$ & $-0.22 \pm 0.01$ & $0.111 \pm 0.006$ & $2.61 \pm 0.11$ \\
        Helmholtz-\gls{sde} ($\ell = 1$) & $53.0 \pm 0.6$ & $-0.06 \pm 0.01$ & $0.052 \pm 0.006$ & $0.001 \pm 0.000$ \\
        Gauge Matching & $53.2 \pm 0.7$ & $-0.06 \pm 0.01$ & $0.056 \pm 0.009$ & $0.044 \pm 0.003$ \\
        \midrule
        Gauge Matching, diagonal $L_\phi$ & $55.3 \pm 1.9$ & $-0.10 \pm 0.02$ & $0.064 \pm 0.006$ & $0.022 \pm 0.001$ \\
        \bottomrule
    \end{tabular}
\end{table}

\subsection{Nonlinear Dynamics}
\label{sec:exp:nonlinear}
\paragraph{Noisy Lorenz Attractor}
We follow \citet{smith2026closing}: the stochastic Lorenz system on $[0, 5]$, with process noise large enough that the posterior does not collapse onto the true trajectory, is observed every $\Delta t = 0.25$ through a noisy linear map, for 1024 training trials; the latent \gls{sde} has $d = 4$, a neural prior drift and a fixed diffusion coefficient, and the prior and observation model are learned jointly with the posterior process. In \cref{tab:lorenz}, the fixed gauges reach a negative \gls{elbo} of $75$--$78$ nats per trial, whether with diagonal or with full covariance. Gauge Matching and Helmholtz-\gls{sde} reach $53$ nats; the difference lies in the path-space \gls{kl} term, and their marginals assign a higher likelihood to the true latent state. The learned dynamics improve too: the autocorrelation error of the prior process drops from $0.11$--$0.13$ to $0.05$--$0.06$, and only these two models reproduce the outward spiral within each lobe of the attractor (\cref{fig:lorenzsamples}). The gauge excess confirms that gradient descent finds the optimal gauge, to within $0.05$ nats per unit time against $2.5$--$2.8$ for the fixed gauges, and the last row shows that the gain is the gauge, not the covariance: a diagonal $L_\phi$ with a learned gauge reaches $55.3$ nats, a full $L_\phi$ without one stays at $75.5$. The picture holds across observation spacings and noise levels (\cref{sec:expdetails}), and the agreement of Gauge Matching and Helmholtz-\gls{sde} is what \cref{sec:gaugematching} predicts: the fitted drift networks are nearly affine over the posterior support.
\paragraph{Known Nonlinear Drift}
Gauge Matching and Helmholtz-\gls{sde} differ when the prior drift is known and strongly nonlinear: on a Stuart--Landau oscillator with shear and the generative model fixed at the truth (\cref{tab:hopf} in \cref{sec:expdetails}), the gauge excess of the learned gauge is $10$ times smaller than that of the order-one correction, and its negative \gls{elbo} is lower by $1.6 \pm 0.3$ nats, on every seed.
\paragraph{State-Dependent Noise}
The stochastic Lotka--Volterra system \citep{vadillo2019comparing,bartosh2025sde} has multiplicative noise, $\vol(z) = \sigma \operatorname{diag}(z)$, which Helmholtz-\gls{sde} does not support; in log coordinates the noise becomes additive and the drift acquires exponentials. We learn a latent \gls{sde} from sparse, noisy observations in each coordinate system, with the true diffusion coefficient in log coordinates and a learned state-dependent one in the original ones, and evaluate the learned prior process by the period of its cycle, the predator's lag behind the prey and the autocorrelation error (\cref{tab:lv} in \cref{sec:expdetails}). In log coordinates the posterior is narrow relative to the curvature of the drift, so the order-one correction is already optimal: Helmholtz-\gls{sde} and Gauge Matching agree, as \cref{sec:gaugematching} predicts, with gauge excess $0.000$ and $0.004$ against $0.215$ for the fixed gauge. In the original coordinates, where Helmholtz-\gls{sde} does not apply, Gauge Matching reduces the period error from $22\%$ to $13\%$ and improves all three metrics on every seed (\cref{fig:lv}). The same holds in $d = 8$: on a rotating Ornstein--Uhlenbeck process with multiplicative noise $\vol(z) = \sigma \operatorname{diag}(\sqrt{1 + z_i^2})$, Gauge Matching lowers the negative \gls{elbo} of the fixed gauges by $6.6 \pm 0.8$ nats with the generative model known and $4.8 \pm 0.3$ with a learned prior, at an $18$--$29$ times smaller gauge excess (\cref{tab:rousd}).
\paragraph{Motion Capture}
On the motion capture benchmark of \citet{gan2015deep}, 50-dimensional recordings with 300 dense, nearly noise-free observations each, Gauge Matching with $d = 6$ reaches a test mean squared error of $4.52 \pm 0.29$, on par with \gls{sde} Matching ($4.50 \pm 0.32$) and, like it, behind the simulation-based latent \gls{sde} of \citet{li2020scalable} ($4.03 \pm 0.20$): the posterior concentrates on the dense observations, so the gauge has little to add, and nothing is lost.

\paragraph{Scaling in the Latent Dimension}
\label{sec:exp:dimsweep}
The learned gauge adds $d(d-1)/2$ network outputs, $2016$ at $d = 64$; Helmholtz-\gls{sde} pays $d$ forward-mode passes through the prior drift network and an $\mathcal{O}(d^3)$ solve of \cref{eq:optimalgauge}. On a rotating Ornstein--Uhlenbeck process in $d \in \{2, 4, \ldots, 64\}$ dimensions with the generative model known, the gap of the learned gauge stays within a factor of two of that of Helmholtz-\gls{sde}, which is exact here, and $4$--$27$ times below the fixed gauges at every $d$: gradient descent finds the optimal gauge in $2016$ dimensions (\cref{fig:dimsweep} in \cref{sec:expdetails}). With a learned generative model and the state fully observed, the fixed gauges are $3$--$9$ nats worse than Helmholtz-\gls{sde} at $d = 8$ and, for a full $L_\phi$, at every $d \le 32$ on every seed, while Gauge Matching recovers most of the difference. Per training iteration on a GPU, Gauge Matching costs $1.05$--$1.5$ times as much as \gls{sde} Matching for $d \le 128$, Helmholtz-\gls{sde} $1.3$--$3.2$ times.
\section{Related Work}
\label{sec:relatedwork}

\paragraph{Variational Inference for Latent SDEs}
Simulation-based \gls{vi} parameterizes the posterior process by its drift coefficient and optimizes the gauge implicitly, at the cost of one numerical integration per gradient step: with an affine drift the marginals follow from moment \glspl{ode} \citep{archambeau2007gaussian,archambeau2007variational,duncker2019learning,wildner2021moment,kohs2021variational,verma2024variational,hu2025sing,park2025amortized}; with a neural drift the process is simulated and differentiated through \citep{movellan2002monte,ryder2018black,tzen2019neural,hodgkinson2020stochastic,li2020scalable,kidger2021efficient}, as in latent \glspl{ode} \citep{chen2018neural,rubanova2019latent}.

\paragraph{Simulation-Free Methods and the Gauge}
Diffusions with prescribed marginals go back to \citet{brigo2000sdes,brigo2002general} and were used for variational path estimation by \citet{sutter2016variational}. Simulation-free \gls{vi} with Gaussian marginals \citep{course2023amortized,course2023state,bartosh2025sde} fixes the gauge (\cref{sec:fixedgauges}); \citet{kiyohara2025neural} use a mean-field family, and \citet{neklyudov2023action,zhang2024trajectory} learn simulation-free dynamics in observation space. \citet{smith2026closing} identified a post-hoc gauge correction (\cref{sec:gaugematching}) and briefly mention learning $\gauge$ in their Appendix A.5 without developing or evaluating it; they suggest this sacrifices \gls{sde} Matching's closed-form structure, whereas our $O(t)$ construction retains its explicit closed-form flow map.

\paragraph{Marginal-Preserving Fields}
The same freedom appears in generative modeling, where the probability-flow \gls{ode} \citep{song2021score} is one of many transports between the marginals of a diffusion model: \citet{horvat2024gauge} analyze this freedom, \citet{petrovic2025curly} exploit it in flow matching \citep{lipman2023flow} to learn non-gradient dynamics, and \citet{hua2025simulation} parameterize a density and a dynamics jointly so that the continuity equation holds by construction. Helmholtz--Hodge decompositions also structure Gaussian-process models of vector fields \citep{berlinghieri2023gaussian,xu2024hhd}. There the divergence-free component is a modeling choice; here it is a coordinate of the variational family that the marginals leave free. Diffusion models with learned noising processes \citep{singhal2023where,bartosh2024neural,nielsen2024diffenc,sahoo2024diffusion,bartosh2024neuralflow,du2024doob} are latent \glspl{sde} with one observation and a learned posterior.

\section{Conclusion}
We introduced Gaussian flow dynamics, which turn the marginal-preserving freedom of a process with Gaussian marginals into an explicit skew-symmetric generator, at quadratic cost and for state-dependent diffusion coefficients, and Gauge Matching, which lets the \gls{sde} Matching objective select this generator. The learned gauge removes most of the gap of the fixed gauges in the conjugate setting, matches Helmholtz-\gls{sde} on nonlinear systems at lower cost and up to $d = 64$, improves on it when the prior drift is known and strongly nonlinear, and extends to state-dependent noise, where Helmholtz-\gls{sde} does not apply. A state-dependent gauge generator extends the construction to general nonlinear marginal-preserving fields, beyond the polynomial gauges and state-independent diffusion of Helmholtz-\gls{sde} (\cref{sec:statedependentgauge}); its evaluation and non-Gaussian marginals are left for future work.



\bibliography{references}

@article{brigo2000sdes,
  title={On {SDEs} with marginal laws evolving in finite-dimensional exponential families},
  author={Brigo, Damiano},
  journal={Statistics \& Probability Letters},
  volume={49},
  number={2},
  pages={127--134},
  year={2000},
  publisher={Elsevier}
}

@article{brigo2002general,
  title={The general mixture-diffusion {SDE} and its relationship with an uncertain-volatility option model with volatility-asset decorrelation},
  author={Brigo, Damiano},
  journal={arXiv preprint arXiv:0812.4052},
  note={Working paper dated September 2002},
  year={2002}
}

@article{sutter2016variational,
  title={A variational approach to path estimation and parameter inference of hidden diffusion processes},
  author={Sutter, Tobias and Ganguly, Arnab and Koeppl, Heinz},
  journal={Journal of Machine Learning Research},
  volume={17},
  number={190},
  pages={1--37},
  year={2016}
}

@article{course2023state,
  title={State estimation of a physical system with unknown governing equations},
  author={Course, Kevin and Nair, Prasanth B},
  journal={Nature},
  volume={622},
  number={7982},
  pages={261--267},
  year={2023},
  publisher={Nature Publishing Group UK London}
}

@article{course2023amortized,
  title={Amortized reparametrization: efficient and scalable variational inference for latent {SDEs}},
  author={Course, Kevin and Nair, Prasanth B},
  journal={Advances in Neural Information Processing Systems},
  volume={36},
  pages={78296--78318},
  year={2023}
}

@inproceedings{bartosh2024neural,
  author    = {Bartosh, Grigory and Vetrov, Dmitry and Naesseth, Christian A.},
  title     = {Neural Diffusion Models},
  booktitle = {Proceedings of the 41st International Conference on Machine Learning},
  series    = {Proceedings of Machine Learning Research},
  volume    = {235},
  pages     = {3073--3095},
  publisher = {PMLR},
  year      = {2024}
}

@article{bartosh2024neuralflow,
  title={Neural flow diffusion models: Learnable forward process for improved diffusion modelling},
  author={Bartosh, Grigory and Vetrov, Dmitry and Naesseth, Christian A},
  journal={Advances in Neural Information Processing Systems},
  volume={37},
  pages={73952--73985},
  year={2024}
}

@inproceedings{bartosh2025sde,
  author    = {Bartosh, Grigory and Vetrov, Dmitry and Naesseth, Christian A.},
  title     = {{SDE} {M}atching: Scalable and Simulation-Free Training of Latent Stochastic Differential Equations},
  booktitle = {Proceedings of the 42nd International Conference on Machine Learning},
  series    = {Proceedings of Machine Learning Research},
  volume    = {267},
  pages     = {3054--3070},
  publisher = {PMLR},
  year      = {2025}
}

@article{smith2026closing,
  title={Closing the Approximation Gap in Simulation-free Latent {SDEs}},
  author={Smith, Henry D and Trippe, Brian L and Linderman, Scott W},
  journal={arXiv preprint arXiv:2606.16138},
  year={2026}
}

@book{Särkkä_Solin_2019,
  title={Applied Stochastic Differential Equations},
  author={S{\"a}rkk{\"a}, Simo and Solin, Arno},
  publisher={Cambridge University Press},
  address={Cambridge},
  series={Institute of Mathematical Statistics Textbooks},
  year={2019}
}

@article{tomczak2016improving,
  title={Improving variational auto-encoders using {H}ouseholder flow},
  author={Tomczak, Jakub M and Welling, Max},
  journal={arXiv preprint arXiv:1611.09630},
  year={2016}
}

@book{grove2002classical,
  title={Classical groups and geometric algebra},
  author={Grove, Larry C},
  series={Graduate Studies in Mathematics},
  volume={39},
  year={2002},
  publisher={American Mathematical Society}
}

@book{lee2012smooth,
  title={Introduction to Smooth Manifolds},
  author={Lee, John M},
  series={Graduate Texts in Mathematics},
  volume={218},
  edition={2},
  publisher={Springer},
  address={New York},
  year={2012}
}

@book{oksendal2003stochastic,
  title={Stochastic Differential Equations: An Introduction with Applications},
  author={{\O}ksendal, Bernt},
  edition={6},
  publisher={Springer},
  address={Berlin},
  year={2003}
}

@inproceedings{li2020scalable,
  title={Scalable gradients for stochastic differential equations},
  author={Li, Xuechen and Wong, Ting-Kam Leonard and Chen, Ricky T. Q. and Duvenaud, David},
  booktitle={International Conference on Artificial Intelligence and Statistics},
  pages={3870--3882},
  year={2020},
  organization={PMLR}
}

@inproceedings{song2021score,
  title={Score-based generative modeling through stochastic differential equations},
  author={Song, Yang and Sohl-Dickstein, Jascha and Kingma, Diederik P and Kumar, Abhishek and Ermon, Stefano and Poole, Ben},
  booktitle={International Conference on Learning Representations},
  year={2021}
}

@inproceedings{gan2015deep,
  title={Deep temporal sigmoid belief networks for sequence modeling},
  author={Gan, Zhe and Li, Chunyuan and Henao, Ricardo and Carlson, David E and Carin, Lawrence},
  booktitle={Advances in Neural Information Processing Systems},
  volume={28},
  year={2015}
}

@article{wang2007gaussian,
  title={Gaussian process dynamical models for human motion},
  author={Wang, Jack M and Fleet, David J and Hertzmann, Aaron},
  journal={IEEE Transactions on Pattern Analysis and Machine Intelligence},
  volume={30},
  number={2},
  pages={283--298},
  year={2008}
}

@inproceedings{hu2025sing,
  title={{SING}: {SDE} inference via natural gradients},
  author={Hu, Amber and Smith, Henry and Linderman, Scott},
  booktitle={Advances in Neural Information Processing Systems},
  year={2025}
}

@inproceedings{kingma2015adam,
  title={Adam: A method for stochastic optimization},
  author={Kingma, Diederik P and Ba, Jimmy},
  booktitle={International Conference on Learning Representations},
  year={2015}
}

@article{vadillo2019comparing,
  title={Comparing stochastic {L}otka--{V}olterra predator-prey models},
  author={Vadillo, Fernando},
  journal={Applied Mathematics and Computation},
  volume={360},
  pages={181--189},
  year={2019}
}

@inproceedings{archambeau2007gaussian,
  title={Gaussian process approximations of stochastic differential equations},
  author={Archambeau, C{\'e}dric and Cornford, Dan and Opper, Manfred and Shawe-Taylor, John},
  booktitle={Gaussian Processes in Practice},
  pages={1--16},
  year={2007},
  organization={PMLR}
}

@inproceedings{archambeau2007variational,
  title={Variational inference for diffusion processes},
  author={Archambeau, C{\'e}dric and Opper, Manfred and Shen, Yuan and Cornford, Dan and Shawe-Taylor, John},
  booktitle={Advances in Neural Information Processing Systems},
  volume={20},
  year={2007}
}

@inproceedings{duncker2019learning,
  title={Learning interpretable continuous-time models of latent stochastic dynamical systems},
  author={Duncker, Lea and Bohner, Gergo and Boussard, Julien and Sahani, Maneesh},
  booktitle={International Conference on Machine Learning},
  year={2019}
}

@inproceedings{verma2024variational,
  title={Variational {G}aussian process diffusion processes},
  author={Verma, Prakhar and Adam, Vincent and Solin, Arno},
  booktitle={International Conference on Artificial Intelligence and Statistics},
  year={2024}
}

@inproceedings{kohs2021variational,
  title={Variational inference for continuous-time switching dynamical systems},
  author={K{\"o}hs, Lukas and Alt, Bastian and Koeppl, Heinz},
  booktitle={Advances in Neural Information Processing Systems},
  year={2021}
}

@inproceedings{wildner2021moment,
  title={Moment-based variational inference for stochastic differential equations},
  author={Wildner, Christian and Koeppl, Heinz},
  booktitle={International Conference on Artificial Intelligence and Statistics},
  year={2021}
}

@inproceedings{park2025amortized,
  title={Amortized control of continuous state space {F}eynman-{K}ac model for irregular time series},
  author={Park, Byoungwoo and Lee, Hyungi and Lee, Juho},
  booktitle={International Conference on Learning Representations},
  year={2025}
}

@inproceedings{ryder2018black,
  title={Black-box variational inference for stochastic differential equations},
  author={Ryder, Tom and Golightly, Andrew and McGough, A Stephen and Prangle, Dennis},
  booktitle={International Conference on Machine Learning},
  pages={4423--4432},
  year={2018},
  organization={PMLR}
}

@article{movellan2002monte,
  title={A {M}onte {C}arlo {EM} approach for partially observable diffusion processes: Theory and applications to neural networks},
  author={Movellan, Javier R and Mineiro, Paul and Williams, Ruth J},
  journal={Neural Computation},
  volume={14},
  number={7},
  pages={1507--1544},
  year={2002},
  publisher={MIT Press}
}

@article{hodgkinson2020stochastic,
  title={Stochastic normalizing flows},
  author={Hodgkinson, Liam and van der Heide, Chris and Roosta, Fred and Mahoney, Michael W},
  journal={arXiv preprint arXiv:2002.09547},
  year={2020}
}

@inproceedings{hua2025simulation,
  title={Simulation-free differential dynamics through neural conservation laws},
  author={Hua, Mengjian and Vanden-Eijnden, Eric and Chen, Ricky T. Q.},
  booktitle={Proceedings of the Forty-first Conference on Uncertainty in Artificial Intelligence},
  series={Proceedings of Machine Learning Research},
  volume={286},
  pages={1730--1744},
  publisher={PMLR},
  year={2025}
}

@article{tzen2019neural,
  title={Neural stochastic differential equations: Deep latent {G}aussian models in the diffusion limit},
  author={Tzen, Belinda and Raginsky, Maxim},
  journal={arXiv preprint arXiv:1905.09883},
  year={2019}
}

@inproceedings{kidger2021efficient,
  title={Efficient and accurate gradients for neural {SDEs}},
  author={Kidger, Patrick and Foster, James and Li, Xuechen and Lyons, Terry},
  booktitle={Advances in Neural Information Processing Systems},
  volume={34},
  pages={18747--18761},
  year={2021}
}

@inproceedings{chen2018neural,
  title={Neural ordinary differential equations},
  author={Chen, Ricky T. Q. and Rubanova, Yulia and Bettencourt, Jesse and Duvenaud, David K},
  booktitle={Advances in Neural Information Processing Systems},
  volume={31},
  year={2018}
}

@inproceedings{rubanova2019latent,
  title={Latent ordinary differential equations for irregularly-sampled time series},
  author={Rubanova, Yulia and Chen, Ricky T. Q. and Duvenaud, David K},
  booktitle={Advances in Neural Information Processing Systems},
  volume={32},
  year={2019}
}

@inproceedings{kiyohara2025neural,
  title={Neural stochastic flows: Solver-free modelling and inference for {SDE} solutions},
  author={Kiyohara, Naoki and Johns, Edward and Li, Yingzhen},
  booktitle={Advances in Neural Information Processing Systems},
  year={2025}
}

@inproceedings{neklyudov2023action,
  title={Action matching: Learning stochastic dynamics from samples},
  author={Neklyudov, Kirill and Brekelmans, Rob and Severo, Daniel and Makhzani, Alireza},
  booktitle={International Conference on Machine Learning},
  year={2023}
}

@inproceedings{zhang2024trajectory,
  title={Trajectory flow matching with applications to clinical time series modelling},
  author={Zhang, Xi and Pu, Yuan and Kawamura, Yuki and Loza, Andrew and Bengio, Yoshua and Shung, Dennis L and Tong, Alexander},
  booktitle={Advances in Neural Information Processing Systems},
  volume={37},
  year={2024}
}

@inproceedings{lipman2023flow,
  title={Flow matching for generative modeling},
  author={Lipman, Yaron and Chen, Ricky T. Q. and Ben-Hamu, Heli and Nickel, Maximilian and Le, Matt},
  booktitle={International Conference on Learning Representations},
  year={2023}
}

@inproceedings{horvat2024gauge,
  title={On gauge freedom, conservativity and intrinsic dimensionality estimation in diffusion models},
  author={Horvat, Christian and Pfister, Jean-Pascal},
  booktitle={International Conference on Learning Representations},
  year={2024}
}

@inproceedings{petrovic2025curly,
  title={Curly flow matching for learning non-gradient field dynamics},
  author={Petrovi{\'c}, Katarina and Atanackovic, Lazar and Moro, Viggo and Kapu{\'s}niak, Kacper and Ceylan, Ismail Ilkan and Bronstein, Michael and Bose, Avishek Joey and Tong, Alexander},
  booktitle={Advances in Neural Information Processing Systems},
  year={2025}
}

@inproceedings{berlinghieri2023gaussian,
  title={Gaussian processes at the helm(holtz): A more fluid model for ocean currents},
  author={Berlinghieri, Renato and Trippe, Brian L and Burt, David R and Giordano, Ryan James and Srinivasan, Kaushik and {\"O}zg{\"o}kmen, Tamay and Xia, Junfei and Broderick, Tamara},
  booktitle={International Conference on Machine Learning},
  year={2023}
}

@inproceedings{xu2024hhd,
  title={{HHD-GP}: Incorporating {H}elmholtz-{H}odge decomposition into {G}aussian processes for learning dynamical systems},
  author={Xu, Hao and Pan, Jia},
  booktitle={Advances in Neural Information Processing Systems},
  volume={37},
  year={2024}
}

@inproceedings{singhal2023where,
  title={Where to diffuse, how to diffuse, and how to get back: Automated learning for multivariate diffusions},
  author={Singhal, R. and Goldstein, M. and Ranganath, R.},
  booktitle={International Conference on Learning Representations},
  year={2023}
}

@inproceedings{nielsen2024diffenc,
  title={{DiffEnc}: Variational diffusion with a learned encoder},
  author={Nielsen, Beatrix M. G. and Christensen, Anders and Dittadi, Andrea and Winther, Ole},
  booktitle={International Conference on Learning Representations},
  year={2024}
}

@inproceedings{sahoo2024diffusion,
  title={Diffusion models with learned adaptive noise},
  author={Sahoo, Subham Sekhar and Gokaslan, Aaron and {De Sa}, Christopher and Kuleshov, Volodymyr},
  booktitle={Advances in Neural Information Processing Systems},
  volume={37},
  pages={105730--105779},
  year={2024}
}

@inproceedings{du2024doob,
  title={Doob's {L}agrangian: A sample-efficient variational approach to transition path sampling},
  author={Du, Yuanqi and Plainer, Michael and Brekelmans, Rob and Duan, Chenru and No{\'e}, Frank and Gomes, Carla P. and Aspuru-Guzik, Al{\'a}n and Neklyudov, Kirill},
  booktitle={Advances in Neural Information Processing Systems},
  volume={37},
  pages={65791--65822},
  year={2024}
}
\bibliographystyle{iclr2027_conference}

\appendix
\newpage
\section{Proofs for \texorpdfstring{\Cref{sec:gaussianflowdynamics}}{Section 3}}\label{sec:proofs}

This appendix proves that the Gaussian flow and the Gaussian flow diffusion have the marginals \cref{eq:gaussianmarginals} (\cref{prop:continuity,prop:fpk}), and then proves \cref{thm:representation,prop:gauge}.

\subsection{Gaussian Flow and the Continuity Equation}\label{sec:proofcontinuity}

\begin{proposition}
\label{prop:continuity}
The Gaussian flow \cref{eq:gaussianflow} satisfies the continuity equation $\partial_t q(t, z) = -\nabla \cdot [\odedrift(t,z) q(t, z)]$ for $q(t, z) = \Norm(z \mid m(t), S(t)S(t)^\top)$.
\end{proposition}

\begin{proof}
Since $q > 0$, the continuity equation \cref{eq:continuity} for $q$ is equivalent to the equation obtained by dividing it by $q(t, z)$. Using the product rule for the divergence, $\nabla \cdot [\odedrift q] = (\nabla \cdot \odedrift) q + \odedrift^\top \nabla q$, it becomes
\begin{align}
    \partial_t \log q(t, z) + \nabla \cdot \odedrift(t,z) + \odedrift(t,z)^\top \nabla \log q(t, z) &= 0. \label{eq:logcontinuity}
\end{align}
We compute the three terms in turn and show that they sum to zero. To lighten the notation we drop the argument $t$ from $m$, $S$ and their derivatives, and write $P = S S^\top$ for the covariance, so that $P^{-1} = S^{-\top} S^{-1}$, and $\bar z = z - m$.

\emph{First term.} Since $\log q(t, z) = -\frac{1}{2} \log \det P - \frac{1}{2} \bar z^\top P^{-1} \bar z + \text{const}$,
\begin{align*}
    \partial_t \log q(t, z) &= -\frac{1}{2} \tr\left(P^{-1} \dot P\right) + \dot m^\top P^{-1} \bar z - \frac{1}{2} \bar z^\top \partial_t \left(P^{-1}\right) \bar z,
\end{align*}
where we used the derivative of the log-determinant, $\partial_t \log \det P = \tr(P^{-1} \dot P)$, the product rule for the quadratic form, and the symmetry of $P^{-1}$. The two remaining derivatives, $\tr(P^{-1} \dot P)$ and $\partial_t(P^{-1})$, are, with $\dot P = \dot S S^\top + S \dot S^\top$,
\begin{align*}
    \tr\left(P^{-1} \dot P\right) &= \tr\left(S^{-\top} S^{-1} \dot S S^\top\right) + \tr\left(S^{-\top} S^{-1} S \dot S^\top\right) = 2 \tr\left(S^{-1} \dot S\right), \\
    \partial_t \left(P^{-1}\right) &= -P^{-1} \dot P P^{-1} = -S^{-\top} S^{-1} \dot S S^{-1} - S^{-\top} \dot S^\top S^{-\top} S^{-1},
\end{align*}
where the first line uses the cyclic property of the trace and the second the derivative of a matrix inverse. Substituting, and using that the two terms of $-\frac{1}{2} \bar z^\top \partial_t(P^{-1}) \bar z$ are transposes of each other and hence equal,
\begin{align*}
    \partial_t \log q(t, z) &= -\tr\left(S^{-1} \dot S\right) + \dot m^\top S^{-\top} S^{-1} \bar z + \bar z^\top S^{-\top} \dot S^\top S^{-\top} S^{-1} \bar z.
\end{align*}

\emph{Second term.} The vector field $\odedrift(t, z) = \dot m + \dot S S^{-1} \bar z$ is affine in $z$, so its divergence is the trace of its linear part,
\begin{align*}
    \nabla \cdot \odedrift(t,z) &= \tr\left(\dot S S^{-1}\right) = \tr\left(S^{-1} \dot S\right).
\end{align*}

\emph{Third term.} With $\nabla \log q(t, z) = -S^{-\top} S^{-1} \bar z$,
\begin{align*}
    \odedrift(t,z)^\top \nabla \log q(t, z) &= - \dot m^\top S^{-\top} S^{-1} \bar z - \bar z^\top S^{-\top} \dot S^\top S^{-\top} S^{-1} \bar z.
\end{align*}

Summing the three terms, the traces cancel, the terms linear in $\bar z$ cancel, and the quadratic forms cancel, so \cref{eq:logcontinuity} holds.
\end{proof}

\subsection{Gaussian Flow Diffusion and the FPK Equation}\label{sec:proofFPK}

\begin{proposition}
\label{prop:fpk}
The Gaussian flow diffusion \cref{eq:gaussianflowdiffusion} satisfies the \gls{fpk} equation $\partial_t q(t, z) = -\nabla \cdot [\postdrift(t,z) q(t, z)] + \frac{1}{2} \nabla \cdot\nabla \cdot[\vol(t,z) \vol(t,z)^\top q(t, z)]$ for $q(t, z) = \Norm(z \mid m(t), S(t)S(t)^\top)$.
\end{proposition}

\begin{proof}
The idea is that the two terms by which the drift coefficient \cref{eq:gaussianflowdiffusion} differs from the Gaussian flow \cref{eq:gaussianflow} are designed to cancel the diffusion term of the \gls{fpk} equation, so that the equation reduces to the continuity equation of \cref{prop:continuity}. Throughout, $\Sigma = \vol \vol^\top$.

\emph{Step 1: rewrite the drift coefficient.} By the conversion \cref{eq:odetosde} that defines the Gaussian flow diffusion,
\begin{align*}
     \postdrift(t,z) &= \odedrift(t,z) + \frac{1}{2}\Sigma(t,z) \nabla_z \log q(t, z) + \frac{1}{2}\nabla \cdot \Sigma(t, z),
\end{align*}
and, since $\nabla_z \log q = \nabla_z q / q$, the last two terms combine into
\begin{align*}
     \frac{1}{2} \Sigma(t,z) \frac{\nabla_z q(t, z)}{q(t,z)} + \frac{1}{2} \nabla \cdot \Sigma(t, z)
     &= \frac{1}{2 q(t, z)} \left(\Sigma(t,z)\nabla_z q(t, z) + q(t, z)\nabla \cdot \Sigma(t, z) \right) \\
     &= \frac{1}{2 q(t, z)} \nabla \cdot[\Sigma(t, z) q(t, z)],
\end{align*}
where the last equality is the product rule for the divergence applied to each column of the symmetric matrix $\Sigma q$. Hence
\begin{align*}
     \postdrift(t,z) q(t, z) &= \odedrift(t,z) q(t, z) + \frac{1}{2} \nabla \cdot[\Sigma(t, z) q(t, z)].
\end{align*}
\emph{Step 2: substitute into the \gls{fpk} equation.} Using the linearity of the divergence,
\begin{align*}
    &-\nabla \cdot [\postdrift(t,z) q(t, z)] + \frac{1}{2} \nabla \cdot\nabla \cdot[\Sigma(t,z) q(t, z)] \\
    &\qquad = -\nabla \cdot [\odedrift(t,z) q(t, z)] - \frac{1}{2} \nabla \cdot \nabla \cdot[\Sigma(t, z) q(t, z)] + \frac{1}{2} \nabla \cdot\nabla \cdot[\Sigma(t,z) q(t, z)] \\
    &\qquad = -\nabla \cdot [\odedrift(t,z) q(t, z)].
\end{align*}
The \gls{fpk} equation for $\postdrift$ and $\vol$ is therefore equivalent to the continuity equation for $\odedrift$, which $q(t, z) = \Norm(z \mid m(t), S(t)S(t)^\top)$ satisfies by \cref{prop:continuity}.
\end{proof}

\subsection{Proof of \texorpdfstring{\cref{thm:representation}}{Theorem 1}}\label{sec:proofgaussian}

We show that for continuous $A$, $u$, $\vol$ and $S_0 \in \gl(d)$ there are unique continuously differentiable $m$ and $S$ with $m(0) = m_0$, $S(0) = S_0$ such that the Gaussian flow diffusion with flow map $m + S \varepsilon$ and diffusion coefficient $\vol$ is the linear \gls{sde} \cref{eq:linearSDE}.

\begin{proof}
For a state-independent diffusion coefficient $\vol(t)$, with $\Sigma(t) = \vol(t) \vol(t)^\top$, the drift coefficient \cref{eq:gaussianflowdiffusion} of the Gaussian flow diffusion is
\begin{align*}
    \postdrift(t, z) &= \dot m(t)+\left(\dot S(t)S(t)^{-1} -\frac{1}{2}\Sigma(t) S(t)^{-\top}S(t)^{-1}\right)(z-m(t)).
\end{align*}
It coincides with the drift coefficient $A(t) z + u(t)$ of \cref{eq:linearSDE} for all $z$ if and only if the linear parts and the constant parts agree,
\begin{align}
    \dot S(t)S(t)^{-1} -\frac{1}{2}\Sigma(t) S(t)^{-\top}S(t)^{-1} &= A(t), \label{eq:statematrix}\\
    \dot m(t) - \left(\dot S(t)S(t)^{-1} -\frac{1}{2}\Sigma(t) S(t)^{-\top}S(t)^{-1}\right) m(t) &= u(t). \label{eq:driver}
\end{align}
The initial distributions agree by the choice $m(0) = m_0$, $S(0) = S_0$. It therefore remains to show that \cref{eq:statematrix,eq:driver} have unique continuously differentiable solutions $S : [0, T] \to \gl(d)$ and $m : [0, T] \to \reals^d$ with these initial values. We proceed in three steps: we construct $S$ from the covariance of \cref{eq:linearSDE} and an orthogonal factor, show that this $S$ is the unique solution, and then solve for $m$.

\emph{Step 1: construction of $S$.} Multiplied by $S(t)$ from the right, \cref{eq:statematrix} is equivalent to the matrix \gls{ode}
\begin{align}
    \dot S(t) &= A(t) S(t) + \frac{1}{2}\Sigma(t) S(t)^{-\top}. \label{eq:statematrixequiv}
\end{align}
Let $P_0 = S_0 S_0^\top \succ 0$ and let $P(t)$ solve the differential Lyapunov equation
\begin{align*}
    \dot P(t) &= A(t) P(t) + P(t) A(t)^\top + \Sigma(t), \quad P(0) = P_0,
\end{align*}
which is the covariance equation of the linear \gls{sde} \cref{eq:linearSDE}. For continuous $A$ and $\Sigma \succeq 0$ it has the unique solution
\begin{align*}
    P(t) &= \Phi_A(t, 0) P_0 \Phi_A(t,0)^\top + \int_0^t \Phi_A(t, s)\Sigma(s)\Phi_A(t,s)^\top \diff s,
\end{align*}
where $\Phi_A$ is the state-transition matrix of $A$, the solution of $\partial_t \Phi_A(t, s) = A(t) \Phi_A(t, s)$ with $\Phi_A(s, s) = I$ \citep{Särkkä_Solin_2019}. Since $P_0 \succ 0$ and $\Phi_A(t, 0)$ is invertible, $P(t) \succ 0$ for all $t$. Let $P(t) = L(t)L(t)^\top$ be its Cholesky factorization; since the Cholesky factor is a smooth function of a positive definite matrix, $L$ is continuously differentiable and invertible. Define
\begin{align*}
    \Omega(t) &= L(t)^{-1}\left[A(t) L(t) +\frac{1}{2}\Sigma(t) L(t)^{-\top}-\dot L(t)\right].
\end{align*}
To see that $\Omega$ is skew-symmetric, multiply by $L$ on the left and by $L^\top$ on the right:
\begin{align*}
    L \Omega L^\top + \left(L \Omega L^\top\right)^\top &= \left(A P + P A^\top + \Sigma\right) - \left(\dot L L^\top + L \dot L^\top\right) = \dot P - \dot P = 0,
\end{align*}
using the Lyapunov equation for $\dot P$ and $\dot P = \dot L L^\top + L \dot L^\top$; since $L$ is invertible, $\Omega + \Omega^\top = 0$. Now let $O(t)$ solve
\begin{align*}
    \dot O(t) &= \Omega(t) O(t), \quad O(0) = L(0)^{-1} S_0.
\end{align*}
The initial condition is orthogonal, since $O(0) O(0)^\top = L(0)^{-1} P_0 L(0)^{-\top} = I$, and skew-symmetry of $\Omega$ implies $\frac{\diff}{\diff t}(O O^\top) = 0$, so $O(t)$ is orthogonal for all $t$. Hence $S(t) = L(t) O(t)$ is invertible with $S(0) = S_0$, and
\begin{align*}
    \dot S &= \dot L O + L \Omega O = \dot L O + \left(A L + \tfrac{1}{2} \Sigma L^{-\top} - \dot L\right) O = A S + \tfrac{1}{2} \Sigma S^{-\top},
\end{align*}
since $L^{-\top} O = S^{-\top}$ for orthogonal $O$; so $S$ satisfies \cref{eq:statematrixequiv}. Here $\Omega = \dot O O^\top$ is exactly the gauge generator of \cref{sec:gaussianflowdynamics} for the factorization $S = L O$; \cref{sec:proofconjugate} uses this identification.

\emph{Step 2: uniqueness of $S$.} The right-hand side of \cref{eq:statematrixequiv} is locally Lipschitz in $S$ on the open set $\gl(d)$, so the solution with $S(0) = S_0$ is unique as long as it exists; the solution constructed in Step 1 exists on all of $[0, T]$, so it is the unique one.

\emph{Step 3: the mean.} Substituting \cref{eq:statematrix} into \cref{eq:driver} gives the linear \gls{ode}
\begin{align*}
    \dot m(t) &= A(t) m(t) + u(t), \quad m(0) = m_0,
\end{align*}
which has a unique continuously differentiable solution since $A$ and $u$ are continuous \citep{Särkkä_Solin_2019}.
\end{proof}

\subsection{Proof of \texorpdfstring{\cref{prop:gauge}}{Proposition 2}}\label{sec:proofgauge}

We show that an affine vector field is $q$-divergence-free for a Gaussian $q$ with Cholesky factor $L$ if and only if it is a gauge field $L \gauge L^{-1}(z - m)$ with skew-symmetric $\gauge$.

\begin{proof}
We write the $q$-divergence-free condition in whitened coordinates, where it becomes a polynomial identity in the whitened state, and read off the constraints on $M$ and $c$ degree by degree. Fix $t$ and drop it from the notation. Let $q(z) = \Norm(z \mid m, L L^\top)$ with $L$ invertible, and let $v(z) = M z + c$ be an affine vector field. By the product rule for the divergence,
\begin{align*}
    \nabla \cdot [v(z) q(z)] &= \left(\nabla \cdot v(z) + v(z)^\top \nabla \log q(z)\right) q(z),
\end{align*}
so $v$ is $q$-divergence-free if and only if $\nabla \cdot v(z) + v(z)^\top \nabla \log q(z) = 0$ for all $z$. Since $\nabla \cdot v(z) = \tr(M)$ and $\nabla \log q(z) = -L^{-\top} L^{-1}(z - m)$, in the whitened coordinate $\eta = L^{-1}(z - m)$ this condition reads
\begin{align}
    \tr(M) - \left(L^{-1} M L \, \eta + L^{-1}(M m + c)\right)^\top \eta &= 0 \quad \text{for all } \eta \in \reals^d. \label{eq:gaugecondition}
\end{align}
The left-hand side is a polynomial of degree two in $\eta$, and it vanishes identically if and only if each of its homogeneous parts vanishes. The quadratic part, $-\eta^\top L^{-1} M L \, \eta$, vanishes for all $\eta$ if and only if $L^{-1} M L$ is skew-symmetric. The linear part, $-(M m + c)^\top L^{-\top} \eta$, vanishes for all $\eta$ if and only if $M m + c = 0$. The constant part requires $\tr(M) = \tr(L^{-1} M L) = 0$, which skew-symmetry already implies.

Hence $v$ is $q$-divergence-free if and only if $c = -M m$ and $\gauge = L^{-1} M L$ is skew-symmetric, that is, if and only if $v(z) = L \gauge L^{-1}(z - m)$ for a skew-symmetric $\gauge$. Both directions of \cref{prop:gauge} follow, and $\gauge = L^{-1} M L$ is determined by $v$.
\end{proof}

The condition \cref{eq:gaugecondition} does not involve the diffusion coefficient. Since the continuity and \gls{fpk} equations depend on the drift coefficient only through $\nabla \cdot [\postdrift q]$, adding a $q$-divergence-free $v$ to the drift coefficient of \cref{eq:gaussianflow} or \cref{eq:gaussianflowdiffusion} leaves both equations, and hence the marginal laws, unchanged, for any $\vol$, state-dependent or not.

\section{Fixed Gauges of Existing Methods}\label{sec:fixedgauges}

This appendix identifies the gauges that existing simulation-free methods use, so that they can be placed in the family of \cref{sec:gaussianflowdynamics}. Existing methods with Gaussian marginals $q(t, z) = \Norm(z \mid m(t), P(t))$ and a state-independent diffusion coefficient use a posterior process with an affine drift coefficient $\postdrift(t, z) = A(t) z + u(t)$; we call $A(t)$ its \emph{state matrix}. Comparing with \cref{eq:gaussianflowdiffusion}, the state matrix has the form
\begin{align}
    A(t) &= K(t) - \frac{1}{2} \Sigma(t) P(t)^{-1}, \qquad \dot P(t) = K(t) P(t) + P(t) K(t)^\top, \label{eq:lyapunovK}
\end{align}
where $K(t) = \dot S(t) S(t)^{-1}$, and the Lyapunov equation on the right follows by differentiating $P = S S^\top$. Conversely, every solution $K(t)$ of this Lyapunov equation is realized by some square root $S(t)$ \citep{smith2026closing}. In terms of the factorization $S = L O$ of \cref{sec:gaussianflowdynamics}, $K(t) = \dot L(t) L(t)^{-1} + L(t) \gauge(t) L(t)^{-1}$ with $L(t)$ the Cholesky factor of $P(t)$. Each method therefore corresponds to a particular choice of the skew-symmetric $\gauge(t)$ as a function of $(P, \dot P, \Sigma)$. The choice $O \equiv I$, that is, $\gauge \equiv 0$ and $K = \dot L L^{-1}$, is the Cholesky gauge of \cref{sec:gaussianflowdynamics}.

\paragraph{Square Root Gauge}
\gls{sde} Matching \citep{bartosh2025sde} defines the posterior process through the affine map $F(t, \varepsilon) = m(t) + S(t) \varepsilon$ for a learned square root $S(t)$, so that $K = \dot S S^{-1}$ directly. Which gauge this selects depends on how $S$ is parameterized. For the symmetric square root $S = P^{1/2}$, \citet{smith2026closing} call the resulting choice the \emph{square root gauge}. In the coordinates $S = L O$ it corresponds to $\gauge = L^{-1} D L$ with $D = \frac{\diff}{\diff t} P^{1/2} \, P^{-1/2} - \dot L L^{-1}$. Both terms of $D$ solve the Lyapunov equation in \cref{eq:lyapunovK}, so their difference $D$ satisfies $D P + P D^\top = 0$; hence $D (z - m)$ is $q$-divergence-free, and $\gauge$ is skew-symmetric by \cref{prop:gauge}. The implementation of \citet{bartosh2025sde} uses a diagonal $S(t)$. The marginal covariance is then diagonal, the symmetric square root and the Cholesky factor coincide, and the square root and Cholesky gauges agree. Parameterizing a general $S(t)$ would in principle span every gauge, but it conflates the marginals with the gauge; the factorization $S = L O$ separates them.

\paragraph{Symmetric Gauge}
\citet{course2023amortized,course2023state} choose $A(t)$ to be the unique symmetric solution of $\dot P - \Sigma = A P + P A^\top$; we call this the \emph{symmetric gauge}. To express it in our coordinates, write $A = A_0 + L \gauge L^{-1}$, where $A_0 = \dot L L^{-1} - \frac{1}{2} \Sigma P^{-1}$ is the state matrix of the Cholesky gauge. Symmetry of $A$ is then equivalent to the Sylvester equation
\begin{align*}
    G_0 \gauge + \gauge G_0 &= -L^\top (A_0 - A_0^\top) L, \qquad G_0 = L^\top L,
\end{align*}
which has a unique skew-symmetric solution because $G_0$ is positive definite and the right-hand side is skew-symmetric; it has the form of \cref{eq:optimalgauge} with $\Sigma = I$. Unlike the square root and Cholesky gauges, the symmetric gauge depends on $\Sigma$, and it does not extend to state-dependent diffusion coefficients: when $\Sigma$ depends on $z$, the drift coefficient of the Cholesky gauge is no longer affine in $z$, and no state-independent $\gauge$ can make its Jacobian symmetric for all $z$ \citep{smith2026closing}.

\paragraph{Learned Gauge}
For all three fixed gauges, Cholesky, square root and symmetric, the gauge is a deterministic function of the marginals and $\Sigma$, chosen independently of the objective. In Gauge Matching, $\gauge(t)$ is instead a free function, and \cref{prop:gauge} shows that no other affine choice exists: every fixed gauge is a point in the family we optimize over. One special case matters for \cref{sec:exp:gauge}. When the marginals are stationary, $\dot P = 0$, the Cholesky and square root gauges give $K = 0$, and the symmetric gauge gives $K = 0$ as well whenever $\Sigma$ and $P$ commute, as in the isotropic setting there. The fixed gauge $O \equiv I$ used in \cref{sec:exp:gauge} therefore represents all three fixed gauges at once. \Cref{tab:methods} summarizes how the fixed gauges, the order-one Helmholtz correction of \cref{sec:gaugematching} and the learned gauge compare in cost and scope.

\paragraph{Householder Parameterization of the Gauge}
Instead of parameterizing the generator $\gauge$ directly, the orthogonal factor $O(t)$ of \cref{eq:factorization} can be parameterized itself, for instance as a Householder flow \citep{tomczak2016improving} of $n = 2\lfloor d/2 \rfloor$ reflections, $O(t) = \prod_{i=1}^{n} H(v_i(t))$ with $H(v) = I - 2 v v^\top / \|v\|^2$, which represents every rotation \citep[Theorem 6.6]{grove2002classical}; the gauge is then $\gauge = \dot O O^\top$, obtained by forward-mode differentiation of the $v_i$. This recovers an explicit flow map $\flow(t, \varepsilon) = m(t) + L(t) O(t) \varepsilon$ at the same asymptotic cost, at the price of a redundant parameterization, since $O$ and $O R$ for constant orthogonal $R$ give the same process. We use it in \cref{sec:exp:gauge}, where the gauge angle of $O(t)$ can be read off directly.

\subsection{State-Dependent Gauge Generators}\label{sec:statedependentgauge}

The gauge of \cref{sec:gaussianflowdynamics} is a skew-symmetric matrix $\gauge(t)$, and \cref{prop:gauge} shows that it spans every \emph{affine} marginal-preserving field. Letting the generator depend on the state extends the construction to nonlinear fields. We record the extension because it places the higher-order corrections of \citet{smith2026closing} inside the gauge framework and keeps the simulation-free structure of Gauge Matching; we do not evaluate it experimentally.

\paragraph{Marginal-Preserving Fields from Skew-Symmetric Fields}
Let $q(t, z) > 0$ be a smooth density and $\gauge(t, z)$ a continuously differentiable skew-symmetric matrix field. With the column-wise divergence of \cref{eq:fpk}, $(\nabla \cdot \gauge)_i = \sum_j \partial_j \gauge_{ji}$, define
\begin{align}
    v(t, z) &= \frac{1}{q(t, z)} \nabla \cdot \left[ q(t, z) \gauge(t, z) \right] = \nabla \cdot \gauge(t, z) - \gauge(t, z) \nabla_z \log q(t, z), \label{eq:skewfield}
\end{align}
where the second expression is the product rule, $\sum_j \partial_j (q \gauge_{ji}) = q (\nabla \cdot \gauge)_i + \sum_j \gauge_{ji} \partial_j q$, together with $\gauge_{ji} = -\gauge_{ij}$. Then $v$ is $q$-divergence-free,
\begin{align*}
    \nabla \cdot [v q] &= \sum_{i, j} \partial_i \partial_j \left( q \gauge_{ji} \right) = 0,
\end{align*}
because the second derivatives are symmetric in $(i, j)$ and $\gauge_{ji}$ is antisymmetric. Adding $v$ to the drift coefficient of any process with marginals $q$ therefore leaves the marginals unchanged, for any diffusion coefficient, as in \cref{sec:proofgauge}. For a state-independent $\gauge(t)$ and the Gaussian marginals \cref{eq:gaussianmarginals}, $\nabla \cdot \gauge = 0$ and $\nabla_z \log q = -L^{-\top} L^{-1}(z - m)$, so $v = \gauge L^{-\top} L^{-1}(z - m)$, which is the gauge field \cref{eq:gaugevectorfield} after the change of variable $\gauge \mapsto L \gauge L^\top$. Conversely, every $q$-divergence-free field is of the form \cref{eq:skewfield}: $q v$ is a divergence-free field on $\reals^d$ and hence the divergence of an antisymmetric matrix field, $q v = \nabla \cdot B$ with $B^\top = -B$, by the Poincar\'e lemma applied to the closed $(d-1)$-form that $q v$ defines \citep[Chapter 17]{lee2012smooth}; $\gauge = B / q$ is skew-symmetric and recovers $v$. In dimension $d = 1$ the only skew-symmetric matrix is zero, and indeed no nonzero marginal-preserving field exists.

\paragraph{Whitened Form}
For the Gaussian marginals \cref{eq:gaussianmarginals} it is convenient to define the generator in the whitened state $\eta = L(t)^{-1}(z - m(t))$. If $\gauge(t, z) = L(t) \tilde\gauge(t, \eta) L(t)^\top$ for a skew-symmetric field $\tilde\gauge(t, \eta)$, then \cref{eq:skewfield} becomes
\begin{align}
    v(t, z) &= L(t) \, u(t, \eta), \qquad u(t, \eta) = \nabla_\eta \cdot \tilde\gauge(t, \eta) + \tilde\gauge(t, \eta) \eta, \label{eq:whitenedskewfield}
\end{align}
and $u$ is divergence-free with respect to the standard Gaussian, $\nabla_\eta \cdot u - \eta^\top u = 0$. The drift coefficient of the posterior process is \cref{eq:driftwhitened} with $L \gauge \eta$ replaced by $L u$,
\begin{align}
    \postdrift(t, z) &= \dot m + \dot L \eta + L \left( \nabla_\eta \cdot \tilde\gauge(t, \eta) + \tilde\gauge(t, \eta) \eta \right) - \frac{1}{2} \Sigma(t, z) L^{-\top} \eta + \frac{1}{2} \nabla \cdot \Sigma(t, z), \label{eq:driftwhitenedsd}
\end{align}
which for a state-independent $\tilde\gauge = \gauge$ is \cref{eq:driftwhitened} itself. The proof of \cref{prop:fpk} applies verbatim, since it uses only that the added field is $q$-divergence-free. For a polynomial $u$ the generator is explicit, $\tilde\gauge = \nabla \psi - (\nabla \psi)^\top$ with $\psi = (\mathcal{N} + I)^{-1} u$ componentwise, where $\mathcal{N} = -\Delta + \eta^\top \nabla$ is the Ornstein--Uhlenbeck number operator, which acts diagonally on Hermite polynomials.

\paragraph{Flow Map and Sampling}
With a state-independent generator, the whitened flow $\dot O = \gauge O$ is a rotation, which preserves $\Norm(0, I)$; this is why the posterior process is a Gaussian flow diffusion with flow map $m + L O \varepsilon$. With a state-dependent generator, the whitened flow is the \gls{ode} $\dot \eta = u(t, \eta)$, whose solution map $O_t$ is a nonlinear diffeomorphism that still preserves the standard Gaussian, because $u$ is divergence-free with respect to it. The flow map is $\flow(t, \varepsilon) = m(t) + L(t) O_t(\varepsilon)$, the marginals are unchanged, and the joint distributions at several times are no longer Gaussian, even for a state-independent diffusion coefficient. Sampling and training remain simulation-free: since $O_t(\varepsilon) \sim \Norm(0, I)$, the state $Z_t = m + L \eta$ with $\eta \sim \Norm(0, I)$ is a sample from the marginal, and the drift coefficient \cref{eq:driftwhitenedsd} is evaluated at this $\eta$ without integrating the \gls{ode}. The map $O_t$ is never needed, just as $O(t)$ is never formed in \cref{sec:gaussianflowdynamics}.

\paragraph{Parameterization and Cost}
A network $U_\phi(t, \eta, \obs)$ with $d(d-1)/2$ outputs, placed above the diagonal of an otherwise zero matrix, gives $\tilde\gauge = U_\phi - U_\phi^\top$, with the same number of free entries at every state as the constant gauge. The only new quantity in \cref{eq:driftwhitenedsd} is the divergence $\nabla_\eta \cdot \tilde\gauge$, a vector of $d$ sums of first derivatives of the network. Computed exactly it costs $d$ forward-mode passes; the unbiased estimate $\nabla_\eta \cdot \tilde\gauge = \E_\xi[(\partial_\xi \tilde\gauge)^\top \xi]$ with $\xi \sim \Norm(0, I)$, where $\partial_\xi \tilde\gauge$ is the directional derivative of $\tilde\gauge$ along $\xi$, costs a single Jacobian--vector product, so that an iteration stays at the cost of a few network evaluations. The squared residual in \cref{eq:elbo} is quadratic in the estimated divergence, so the resulting estimator of the path-space \gls{kl} term is biased upward by the variance of the divergence estimate; the bias vanishes as the number of directions $\xi$ grows. We have verified the identities of this section numerically for a neural $\tilde\gauge$: \cref{eq:skewfield,eq:whitenedskewfield}, the invariance of $\Norm(0, I)$ under the flow of $u$, the divergence estimator, and the explicit generator of a polynomial field.

\paragraph{Relation to Higher-Order Helmholtz Corrections}
The order-$\ell$ correction of \citet{smith2026closing} is the $q$-divergence-free component of a polynomial residual of degree $\ell$, and for Gaussian $q$ it is itself a polynomial field of degree $\ell$. By the converse above, it is a field \cref{eq:whitenedskewfield} with a polynomial generator $\tilde\gauge$ of degree $\ell - 1$, and conversely every such generator gives a $q$-divergence-free polynomial field of degree at most $\ell$; the two spaces coincide (we checked that their dimensions agree for $d \in \{2, 3\}$ and $\ell \le 3$). Order one is the constant generator of \cref{sec:gaussianflowdynamics}. A state-dependent generator therefore spans the corrections of every order with the outputs of a single network, where the construction of \citet{smith2026closing} forms the derivative tensors of the prior drift and solves a linear system whose size grows as $d^{\ell + 1}$, and it applies to state-dependent diffusion coefficients, which their construction excludes. Which of these fields the objective selects is again the strictly convex problem of \cref{prop:optimalgauge} on a larger linear space, so the minimizer is unique. Whether it is worth learning is an empirical question that we leave open; \citet{smith2026closing} report no notable improvement from orders $\ell \ge 2$ in their experiments.

\section{Proofs for \texorpdfstring{\Cref{sec:gaugematching}}{Section 4}}\label{sec:proofsgm}

\subsection{Proof of \texorpdfstring{\cref{prop:optimalgauge}}{Proposition 3}}\label{sec:proofoptimalgauge}

The proof has two parts. We first show that $\mathcal{J}$ is a strictly convex quadratic on the skew-symmetric matrices, which gives existence and uniqueness of the minimizer for any diffusion coefficient. We then derive the stationarity condition, which for a state-independent diffusion coefficient reduces to \cref{eq:optimalgauge} by Stein's lemma. After the proof, \cref{sec:optimalgaugeconsequences} draws the consequences used in \cref{sec:gaugematching}: the case of a state-dependent diffusion coefficient, the optimality criterion for fixed gauges, the stationary linear example, and the identification of the order-one Helmholtz correction.

\begin{proof}
Fix $t$ and drop it from the notation, so that $q(z) = \Norm(z \mid m, L L^\top)$, and write $Z = m + L \varepsilon$ with $\varepsilon \sim \Norm(0, I)$. In these coordinates the gauge field \cref{eq:gaugevectorfield} is $L \gauge L^{-1}(Z - m) = L \gauge \varepsilon$, and with the residual of the Cholesky gauge $r_0(z) = \priordrift(z) - \postdrift_0(z)$ the objective reads
\begin{align*}
    \mathcal{J}(\gauge) &= \frac{1}{2} \E_\varepsilon \left[ \left\| r_0(Z) - L \gauge \varepsilon \right\|_{\Sigma(Z)^{-1}}^2 \right].
\end{align*}
Under the hypotheses of \cref{prop:optimalgauge}, the divergence term of $\postdrift_0$ is defined, $r_0$ has at most polynomial growth, $\Sigma(z)^{-1}$ is bounded because $\Sigma$ is uniformly positive definite, and $\mathcal{J}(\gauge)$ is finite for every $\gauge$.

\paragraph{Strict Convexity}
Expanding the square,
\begin{align*}
    \mathcal{J}(\gauge) &= \mathcal{J}(0) - \E_\varepsilon \left[ r_0(Z)^\top \Sigma(Z)^{-1} L \gauge \varepsilon \right] + \frac{1}{2} \E_\varepsilon \left[ \varepsilon^\top \gauge^\top L^\top \Sigma(Z)^{-1} L \gauge \varepsilon \right],
\end{align*}
a quadratic function of $\gauge$. Its quadratic part is nonnegative because $\Sigma(z)^{-1}$ is positive definite for every $z$, and it vanishes only if $L \gauge \varepsilon = 0$ almost surely, that is, only if $\gauge = 0$, since $\varepsilon$ has full support and $L$ is invertible. The quadratic part is therefore positive definite on the linear space $\skewset(d)$ of skew-symmetric matrices. Hence $\mathcal{J}$ is strictly convex and coercive on $\skewset(d)$ and has a unique minimizer $\gauge^\star$. This part of the argument uses no structure of $\Sigma$ beyond positive definiteness; in particular, it holds for state-dependent diffusion coefficients.

\paragraph{State-Independent Diffusion Coefficient}
Let $\Sigma$ be constant. Using $\E_\varepsilon[\varepsilon^\top X \varepsilon] = \tr(X)$ for any matrix $X$ and writing $\Xi = \E_\varepsilon[r_0(Z) \varepsilon^\top]$,
\begin{align*}
    \mathcal{J}(\gauge) &= \mathcal{J}(0) - \tr\left( \gauge \Xi^\top \Sigma^{-1} L \right) + \frac{1}{2} \tr\left( \gauge^\top G \gauge \right), \qquad G = L^\top \Sigma^{-1} L.
\end{align*}
The directional derivative of $\mathcal{J}$ in a skew-symmetric direction $E$ is $\tr( E^\top [ G \gauge - L^\top \Sigma^{-1} \Xi ] )$, where we used that $G$ is symmetric. A matrix $X$ satisfies $\tr(E^\top X) = 0$ for all $E \in \skewset(d)$ if and only if $X$ is symmetric, since the symmetric and skew-symmetric matrices are orthogonal complements under the trace inner product. The derivative therefore vanishes in every skew-symmetric direction if and only if $G \gauge - L^\top \Sigma^{-1} \Xi$ is symmetric, that is, if and only if it equals its transpose $-\gauge G - \Xi^\top \Sigma^{-1} L$, which rearranges to
\begin{align}
    G \gauge + \gauge G &= L^\top \Sigma^{-1} \Xi - \Xi^\top \Sigma^{-1} L. \label{eq:optimalgaugeC}
\end{align}
The right-hand side is skew-symmetric. Since $G$ is positive definite, the Sylvester operator $\gauge \mapsto G \gauge + \gauge G$ has eigenvalues $\lambda_i + \lambda_j > 0$, where $\lambda_i$ are the eigenvalues of $G$, and is therefore invertible on $\reals^{d \times d}$; so \cref{eq:optimalgaugeC} has exactly one solution. Transposing \cref{eq:optimalgaugeC} and using that its right-hand side is skew-symmetric shows that $-\gauge^\top$ solves it whenever $\gauge$ does, so by uniqueness the solution is skew-symmetric.

It remains to compute $\Xi$. Stein's lemma states that $\E_\varepsilon[\varphi(\varepsilon) \varepsilon^\top] = \E_\varepsilon[\nabla_\varepsilon \varphi(\varepsilon)]$ for continuously differentiable $\varphi$ of at most polynomial growth, where $\nabla_\varepsilon \varphi$ is the Jacobian with entries $\partial \varphi_i / \partial \varepsilon_j$, the convention used for $\nabla_z \priordrift$ throughout. Applied to $\varphi(\varepsilon) = r_0(m + L \varepsilon)$, it gives
\begin{align*}
    \Xi &= \E_\varepsilon\left[\nabla_z r_0(Z)\right] L = \left( \E_{q(z)}\left[\nabla_z \priordrift(Z)\right] - A_0 \right) L, \qquad A_0 = \dot L L^{-1} - \frac{1}{2} \Sigma L^{-\top} L^{-1},
\end{align*}
where $A_0$ is the state matrix of $\postdrift_0$ in \cref{eq:gaussianflowdiffusion}, whose divergence term vanishes for constant $\Sigma$. Substituting into \cref{eq:optimalgaugeC}, the term $\frac{1}{2} \Sigma L^{-\top} L^{-1}$ of $A_0$ drops out, since $\Sigma^{-1} \cdot \Sigma L^{-\top} L^{-1} = L^{-\top} L^{-1}$ is symmetric, and \cref{eq:optimalgauge} follows with $\bar J = \E_{q(z)}[\nabla_z \priordrift(Z)] - \dot L L^{-1}$.
\end{proof}

\subsection{Consequences of \texorpdfstring{\cref{prop:optimalgauge}}{Proposition 3}}\label{sec:optimalgaugeconsequences}

\paragraph{State-Dependent Diffusion Coefficient}
For a general $\Sigma(z)$ the minimizer is still the solution of a linear system, since $\mathcal{J}$ is a strictly convex quadratic. Writing $\gauge = \sum_{k} \gamma_k E_k$ in a basis $E_1, \ldots, E_{d(d-1)/2}$ of $\skewset(d)$, the stationarity conditions $\partial \mathcal{J} / \partial \gamma_j = 0$ read $\sum_k Q_{jk} \gamma_k = b_j$ with
\begin{align*}
    Q_{jk} &= \E_\varepsilon\left[ \varepsilon^\top E_j^\top L^\top \Sigma(Z)^{-1} L E_k \varepsilon \right], \qquad
    b_j = \E_\varepsilon\left[ r_0(Z)^\top \Sigma(Z)^{-1} L E_j \varepsilon \right],
\end{align*}
and $Q$ is positive definite by the strict convexity argument above. Unlike in the state-independent case, these coefficients do not reduce to a Jacobian: the factor $\Sigma(Z)^{-1}$ inside the expectations depends on $Z$, so Stein's lemma no longer gives the same simplification. $Q$ and $b$ therefore have to be estimated at every evaluation, by Monte Carlo sampling or, in low dimension, by numerical quadrature. Gauge Matching sidesteps this by learning $\gauge_\phi$ directly.

\paragraph{Optimality of Fixed Gauges}
The stationarity condition also tells us when a given gauge is optimal. Let $\gauge$ be any gauge, with state matrix $A_\gauge = A_0 + L \gauge L^{-1}$. By the computation above, $\gauge$ is optimal if and only if $G \gauge - L^\top \Sigma^{-1} \Xi$ is symmetric. Substituting $G = L^\top \Sigma^{-1} L$ and $\Xi = (\E_{q(z)}[\nabla_z \priordrift(Z)] - A_0) L$,
\begin{align*}
    G \gauge - L^\top \Sigma^{-1} \Xi &= L^\top \Sigma^{-1} \left( L \gauge L^{-1} + A_0 - \E_{q(z)}[\nabla_z \priordrift(Z)] \right) L \\
    &= L^\top \Sigma^{-1} \left( A_\gauge - \E_{q(z)}[\nabla_z \priordrift(Z)] \right) L,
\end{align*}
and since congruence by the invertible $L$ preserves symmetry, $\gauge$ is optimal if and only if
\begin{align}
    \Sigma^{-1} \left( A_\gauge - \E_{q(z)}\left[\nabla_z \priordrift(Z)\right] \right) \text{ is symmetric.} \label{eq:optimalitycriterion}
\end{align}
Two cases are used in \cref{sec:gaugematching}. For the Cholesky gauge, $A_\gauge = A_0$, and \cref{eq:optimalitycriterion} holds if and only if $\Sigma^{-1} \bar J$ is symmetric, since the term $\frac{1}{2} \Sigma L^{-\top} L^{-1}$ of $A_0$ drops out as before. For the symmetric gauge, $A_\gauge$ is symmetric by definition; if in addition $\Sigma = \sigma^2 I$ is isotropic, then \cref{eq:optimalitycriterion} holds if and only if $\E_{q(z)}[\nabla_z \priordrift(Z)]$ is symmetric. This is the case for a gradient prior drift coefficient $\priordrift = -\nabla U$, whose Jacobian $-\nabla^2 U$ is symmetric at every $z$.

\paragraph{The Stationary Linear SDE of \cref{sec:exp:gauge}}
Let $\priordrift(z) = A z$ with $A = -\alpha I + \omega J$, $J = \bigl(\begin{smallmatrix} 0 & -1 \\ 1 & 0 \end{smallmatrix}\bigr)$, $\Sigma = I$, and let the marginals be the stationary law $m = 0$, $L L^\top = I / (2 \alpha)$, so $\dot L = 0$ and $L = I / \sqrt{2 \alpha}$. Then $G = I / (2 \alpha)$, $\bar J = A$, and the right-hand side of \cref{eq:optimalgauge} is $(A - A^\top) / (2 \alpha) = \omega J / \alpha$, so $\gauge^\star = \omega J$. Here $A_0 = \dot L L^{-1} - \frac{1}{2} \Sigma L^{-\top} L^{-1} = -\alpha I$, so the drift coefficient of the optimal gauge is $A_0 z + L \gauge^\star L^{-1} z = -\alpha z + \omega J z = \priordrift(z)$: the optimal gauge recovers the prior process, and the path-space \gls{kl} term vanishes.

\paragraph{The Order-One Helmholtz Correction}
We now show that the order-one correction of Helmholtz-\gls{sde} \citep{smith2026closing} is the optimal gauge of the linearized prior process, as claimed in \cref{sec:gaugematching}. Helmholtz-\gls{sde} starts from a reference posterior process with a fixed gauge $\gauge_{\mathrm{ref}}$, whose drift coefficient $\postdrift_{\mathrm{ref}} = \postdrift_0 + L \gauge_{\mathrm{ref}} L^{-1}(z - m)$ is affine, and adds to it the $q$-divergence-free component of the residual $\priordrift - \postdrift_{\mathrm{ref}}$. This component is defined by the $q$-weighted Helmholtz decomposition of \citet{smith2026closing}: it is the field $v$ that minimizes $\E_{q(z)}[\| \priordrift(Z) - \postdrift_{\mathrm{ref}}(Z) - v(Z) \|_{\Sigma^{-1}}^2]$ over all $q$-divergence-free fields. At order $\ell = 1$, the residual is replaced by its linearization about $m$; since $\postdrift_{\mathrm{ref}}$ is already affine, this amounts to replacing $\priordrift$ by its linearization $\tilde\priordrift(z) = \priordrift(m) + \nabla_z \priordrift(m)(z - m)$.

Three observations identify the resulting gauge.
\begin{enumerate}[nosep,leftmargin=*]
    \item The Helmholtz decomposition of an affine field with respect to a Gaussian $q$ has affine components \citep{smith2026closing}, so the correction is an affine $q$-divergence-free field, that is, by \cref{prop:gauge}, a gauge field $L \gauge_{\mathrm{H}} L^{-1}(z - m)$.
    \item The correction minimizes over all $q$-divergence-free fields and, by the first observation, lies in the subset of gauge fields; it is therefore also the minimizer over gauge fields. Since $\postdrift_{\mathrm{ref}} + L \gauge_{\mathrm{H}} L^{-1}(z - m) = \postdrift_0 + L (\gauge_{\mathrm{ref}} + \gauge_{\mathrm{H}}) L^{-1}(z - m)$, the gauge $\gauge_{\mathrm{ref}} + \gauge_{\mathrm{H}}$ is the unique minimizer of $\mathcal{J}$ with $\priordrift$ replaced by $\tilde\priordrift$.
    \item $\nabla_z \tilde\priordrift \equiv \nabla_z \priordrift(m)$, so by \cref{prop:optimalgauge} this minimizer is the solution of \cref{eq:optimalgauge} with $\bar J$ replaced by $\nabla_z \priordrift(m) - \dot L L^{-1}$.
\end{enumerate}
The corrected gauge therefore does not depend on $\gauge_{\mathrm{ref}}$, and it coincides with $\gauge^\star$ whenever $\E_{q(z)}[\nabla_z \priordrift(Z)] = \nabla_z \priordrift(m)$, in particular whenever $\nabla_z \priordrift$ is affine in $z$. The order-$\ell \ge 2$ corrections of \citet{smith2026closing} are marginal-preserving fields of higher polynomial degree, which lie outside the family of \cref{sec:gaugematching}; \cref{sec:statedependentgauge} identifies them as gauge fields with a state-dependent generator.

\subsection{Proof of \texorpdfstring{\cref{cor:conjugate}}{Corollary 4}}\label{sec:proofconjugate}

We show that for a linear prior process and a linear-Gaussian observation model the true posterior process is a Gaussian flow diffusion between observation times, and that Gauge Matching posterior processes approach it in \gls{kl} divergence. Let $Z$ follow the linear \gls{sde} \cref{eq:linearSDE} with $Z_0 \sim \Norm(m_0, S_0 S_0^\top)$, $S_0 \in \gl(d)$, and $\Sigma(t) = \vol(t) \vol(t)^\top$ positive definite, and let $x_{t_i} \sim \Norm(C Z_{t_i}, R)$ with $R$ positive definite. We take the observation times to be interior, $0 < t_1 < \cdots < t_N < T$, and write $t_0 = 0$ and $t_{N+1} = T$; an observation at an endpoint introduces no kink and only simplifies the argument. The Gauge Matching posterior processes over which the infimum is taken are those with $m_\phi$, $L_\phi$ continuously differentiable and $\gauge_\phi$ continuous on $[0, T]$, the regularity that neural networks in $t$ provide.

The proof has three steps. We first describe the true posterior process and show that its covariance is positive definite. We then apply \cref{thm:representation} on each interval between observations, which gives the first claim of the corollary and an explicit formula for the jump of the gauge at an observation time. Finally, we smooth the resulting $m^\dagger$, $L^\dagger$, $\gauge^\dagger$ in small neighborhoods of the observation times and show that the \gls{kl} divergence of the smoothed process to the true posterior process can be made arbitrarily small. Throughout, a dagger marks quantities of the true posterior process, to distinguish them from the optimal gauge $\gauge^\star$ of \cref{prop:optimalgauge}.

\begin{proof}

\emph{Step 1: the true posterior process.} Conditioned on $\obs$, the latent path is a Markov process with initial law $q^\dagger(0, z) \propto p_\theta(0, z) \, p_\theta(x_{t_1}, \ldots, x_{t_N} \mid Z_0 = z)$ and a drift coefficient given by Doob's $h$-transform \citep[Appendix B]{bartosh2025sde}: on the open interval $(t_i, t_{i+1})$,
\begin{align*}
    \postdrift^\dagger(t, z) &= A(t) z + u(t) + \Sigma(t) \nabla_z \log p_\theta(x_{t_{i+1}}, \ldots, x_{t_N} \mid Z_t = z).
\end{align*}
The likelihood of the future observations given the state is an unnormalized Gaussian function of $z$, so the last term is affine in $z$. Its coefficients are continuous functions of the transition matrix and the process noise covariance of \cref{eq:linearSDE} \citep{Särkkä_Solin_2019}, hence continuous in $t$ on $[t_i, t_{i+1}]$. We conclude that $\postdrift^\dagger(t, z) = A^\dagger(t) z + u^\dagger(t)$ with $A^\dagger$, $u^\dagger$ continuous on each closed interval $[t_i, t_{i+1}]$. Across each $t_i$, however, the set of future observations changes, and with it $A^\dagger$ and $u^\dagger$ change discontinuously. The marginals of the true posterior process are the smoothing distributions $q^\dagger(t, z) = \Norm(z \mid m^\dagger(t), P^\dagger(t))$, which are continuous in $t$ on $[0, T]$.

We also need that $P^\dagger(t) \succ 0$ for every $t$. Conditioning on more information can only decrease a Gaussian covariance, so it suffices to exhibit a finer conditioning under which the covariance of $Z_t$ is still positive definite. Condition, in addition to $\obs$, on the latent states $Z_{t_j}$ at all observation times $t_j \neq t$; given these states, the observations $x_{t_j}$ carry no further information about $Z_t$. The covariance of $Z_t$ is then that of a Gaussian bridge of \cref{eq:linearSDE} over an interval of positive length, or of the unconditioned process before $t_1$ or after $t_N$; both are positive definite since $\Sigma \succ 0$ and $S_0 S_0^\top \succ 0$. If $t$ is itself an observation time, conditioning on $x_t$ adds the finite precision $C^\top R^{-1} C$, which keeps the covariance positive definite. Hence $P^\dagger(t) \succ 0$.

\emph{Step 2: representation between observations.} On $[t_i, t_{i+1}]$ the true posterior process is the linear \gls{sde} with coefficients $A^\dagger$, $u^\dagger$, $\vol$ started from $\Norm(m^\dagger(t_i), P^\dagger(t_i))$ with $P^\dagger(t_i) \succ 0$. By \cref{thm:representation} it is a Gaussian flow diffusion on this interval, whose marginal parameters are necessarily $m = m^\dagger$ and $L = L^\dagger$, the Cholesky factor of $P^\dagger$. Both are continuously differentiable on $[t_i, t_{i+1}]$: $m^\dagger$ and $P^\dagger$ solve $\dot m = A^\dagger m + u^\dagger$ and $\dot P = A^\dagger P + P A^{\dagger\top} + \Sigma$ with continuous coefficients, and the Cholesky factorization is a smooth map on positive definite matrices. By the proof in \cref{sec:proofgaussian}, the gauge generator of this representation is
\begin{align*}
    \gauge^\dagger(t) &= L^\dagger(t)^{-1} \left[ A^\dagger(t) L^\dagger(t) + \frac{1}{2} \Sigma(t) L^\dagger(t)^{-\top} - \dot L^\dagger(t) \right],
\end{align*}
which is continuous on $[t_i, t_{i+1}]$. This proves the first claim of \cref{cor:conjugate}.

Across an observation time, $m^\dagger$ and $L^\dagger$ are continuous, but in general their derivatives jump, and so does $\gauge^\dagger$. The jump is explicit. The last term of $\postdrift^\dagger$ is $\Sigma$ times the gradient of the log-likelihood of the future observations, so $A^\dagger = A - \Sigma \Lambda$, where $\Lambda$ is the negative Hessian of this log-likelihood in $z$, the precision of the future observations about $Z_t$. This precision includes $C^\top R^{-1} C$ for $t < t_i$, when $x_{t_i}$ is still in the future, and not for $t > t_i$. Writing $\Delta$ for the jump from $t_i^-$ to $t_i^+$, $\Lambda$ drops by $C^\top R^{-1} C$ and the state matrix therefore jumps by $\Delta A^\dagger = \Sigma C^\top R^{-1} C$, and the derivative $\dot P^\dagger = A^\dagger P^\dagger + P^\dagger A^{\dagger\top} + \Sigma$ jumps by $\Delta A^\dagger P^\dagger + P^\dagger \Delta A^{\dagger\top}$.

\emph{The jump of $\dot L^\dagger$.} Let $\ltri(X)$ denote the strictly lower triangular part of $X$ plus half its diagonal; differentiating $P = L L^\top$ and solving for the lower triangular $\dot L$ gives $\dot L = L \, \ltri(L^{-1} \dot P L^{-\top})$. Since $L^\dagger$ is continuous across $t_i$, only $\dot L^\dagger$ jumps. With $B = L^{\dagger -1} \Delta A^\dagger L^\dagger = L^{\dagger -1} \Sigma C^\top R^{-1} C L^\dagger$ we have $L^{\dagger -1} \Delta \dot P^\dagger L^{\dagger -\top} = B + B^\top$, so the jump of $\dot L^\dagger$ is $L^\dagger \, \ltri(B + B^\top)$. Substituting into the formula for $\gauge^\dagger$,
\begin{align*}
    \Delta \gauge^\dagger &= L^{\dagger -1} \Delta A^\dagger L^\dagger - L^{\dagger -1} \Delta \dot L^\dagger = B - \ltri(B + B^\top).
\end{align*}
This matrix is skew-symmetric, because $\ltri(X) + \ltri(X)^\top = X$ for symmetric $X$. It vanishes if and only if $B = \ltri(B + B^\top)$, that is, if and only if $B$ is lower triangular. For $\Sigma = I$ and an isotropic added precision $C^\top R^{-1} C \propto I$, as under the observation model of \citet{smith2026closing}, $B$ is diagonal and the gauge does not jump (\cref{sec:expdetails}).

\emph{Step 3: the infimum is zero.} Let $Q^\dagger$ be the law of the true posterior process and $Q_\phi$ that of a Gauge Matching posterior process with the same diffusion coefficient. By Girsanov's theorem,
\begin{align*}
    \KL{Q_\phi}{Q^\dagger} &= \KL{q_\phi(0, z \mid \obs)}{q^\dagger(0, z)} \\
    &\quad + \frac{1}{2} \int_0^T \E_{q_\phi(t, z \mid \obs)} \left[ \left\| \postdrift_\phi(t, Z_t, \obs) - \postdrift^\dagger(t, Z_t) \right\|_{\Sigma(t)^{-1}}^2 \right] \diff t,
\end{align*}
and this divergence equals the slack $\loss + \log p_\theta(\obs)$ of the bound \cref{eq:elbo}, so its infimum is zero if and only if the bound is tight. The integral on the second line has the form of the path-space \gls{kl} term of \cref{eq:elbo} with the prior drift coefficient replaced by $\postdrift^\dagger$; we refer to it as the path-space term of this divergence. We construct a family of Gauge Matching posterior processes $Q_\delta$, indexed by $\delta > 0$, with $\KL{Q_\delta}{Q^\dagger} \to 0$.

\emph{Smoothing.} Fix $\delta > 0$ smaller than half the smallest spacing between consecutive elements of $\{t_0, \ldots, t_{N+1}\}$, and let $U_\delta = \bigcup_{i=1}^N (t_i - \delta, t_i + \delta)$ be the union of $\delta$-neighborhoods of the observation times. Define $m_\delta$ and $L_\delta$ to agree with $m^\dagger$ and $L^\dagger$ outside $U_\delta$ and, on each component of $U_\delta$, to be the cubic Hermite interpolants that match the values and one-sided derivatives of $m^\dagger$, $L^\dagger$ at the endpoints of the component. Define $\gauge_\delta$ analogously by linear interpolation of $\gauge^\dagger$. Then $m_\delta$, $L_\delta$ are continuously differentiable and $\gauge_\delta$ is continuous on $[0, T]$, so $Q_\delta$, the Gauge Matching posterior process with these functions, is admissible.

\emph{Uniform bounds.} Since $m^\dagger$, $L^\dagger$ are Lipschitz and $\dot m^\dagger$, $\dot L^\dagger$, $\gauge^\dagger$ are bounded on $[0, T]$, the interpolants have derivatives $\dot m_\delta$, $\dot L_\delta$ and values $\gauge_\delta$ bounded by a constant independent of $\delta$. Moreover, $\sup_t \| L_\delta(t) - L^\dagger(t) \| = \mathcal{O}(\delta)$, and $\lambda_{\min}(P^\dagger(t))$ is bounded away from zero on $[0, T]$, since $P^\dagger$ is positive definite by Step 1 and continuous on the compact interval $[0, T]$. For small $\delta$, therefore, $L_\delta(t)$ is lower triangular with positive diagonal and $\| L_\delta(t)^{-1} \|$ is bounded uniformly in $t$ and $\delta$.

\emph{The divergence.} The drift coefficient of $Q_\delta$, which is \cref{eq:driftwhitened} evaluated at $(m_\delta, L_\delta, \gauge_\delta)$, agrees with $\postdrift^\dagger$ outside $U_\delta$, so the integrand of the path-space term vanishes there. On $U_\delta$, both drift coefficients are affine in $z$ with coefficients bounded uniformly in $\delta$, and the second moment $\E_{q_\delta(t, z \mid \obs)}[\| Z_t \|^2] = \| m_\delta(t) \|^2 + \tr(L_\delta(t) L_\delta(t)^\top)$ is bounded uniformly as well. The integrand is therefore bounded on $U_\delta$ by a constant $c$ independent of $\delta$, and since $U_\delta$ has length at most $2 N \delta$, the path-space term, with its factor $\frac{1}{2}$, is at most $c N \delta$. The initial term vanishes since $0 \notin U_\delta$. Hence $\KL{Q_\delta}{Q^\dagger} \le c N \delta \to 0$ as $\delta \to 0$, which proves the second claim of \cref{cor:conjugate}. The argument is that of \citet[Theorem 1]{smith2026closing}, with the gauge smoothed alongside the marginals.
\end{proof}

\section{Experimental Details}\label{sec:expdetails}

\Cref{alg:gaugematching} summarizes one training iteration of Gauge Matching, the simulation-free estimator of \cref{sec:simfree} applied to the posterior process of \cref{sec:gaugematching}.

\begin{algorithm}[H]
\caption{Gauge Matching (one training iteration)}
\label{alg:gaugematching}
\begin{algorithmic}[1]
\Require observations $\obs$; generative model $p_\theta(0, z)$, $\priordrift_\theta$, $\vol_\theta$, $p_\theta(x \mid z)$; posterior process $m_\phi$, $L_\phi$, $\gauge_\phi$
\State $\hat\loss \gets \KL{q_\phi(0, z \mid \obs)}{p_\theta(0, z)}$ \Comment{initial \gls{kl} term}
\State $t \sim \mathcal{U}[0, T]$, \; $\varepsilon \sim \Norm(0, I)$, \; $Z_t \gets m_\phi(t, \obs) + L_\phi(t, \obs) \varepsilon$
\State $\postdrift \gets \dot m_\phi + \dot L_\phi \varepsilon + L_\phi \gauge_\phi \varepsilon - \frac{1}{2} \Sigma_\theta(t, Z_t) L_\phi^{-\top} \varepsilon + \frac{1}{2} \nabla \cdot \Sigma_\theta(t, Z_t)$ \Comment{\cref{eq:driftwhitened} with $\eta = \varepsilon$}
\State solve $\vol_\theta(t, Z_t) \, r = \priordrift_\theta(t, Z_t) - \postdrift$, \; $\hat\loss \gets \hat\loss + \frac{T}{2} \| r \|_2^2$ \Comment{path-space \gls{kl} term}
\State $i \sim \mathcal{U}\{1, \ldots, N\}$, \; $\varepsilon' \sim \Norm(0, I)$, \; $Z_{t_i} \gets m_\phi(t_i, \obs) + L_\phi(t_i, \obs) \varepsilon'$
\State $\hat\loss \gets \hat\loss - N \log p_\theta(x_{t_i} \mid Z_{t_i})$ \Comment{reconstruction}
\State gradient step on $(\theta, \phi)$ with respect to $\hat\loss$
\end{algorithmic}
\end{algorithm}

All methods are implemented in a common code base (see the reproducibility statement). This appendix lists the choices shared by all experiments and then, for each experiment, the data, the model, deviations of the posterior process from the shared architecture, the training schedule, and the definition of every reported metric.

\textbf{Posterior architecture.} All posterior processes share the amortized architecture of \citet{bartosh2025sde}: a GRU processes the observations of a trial into per-observation embeddings (when the observation times are irregular, it also receives the times and their spacings); the embedding at a query time $t$ is a softmax-weighted average of these, with weights centered at $t$ over the observation times; and an MLP with two hidden layers of width 100 and SiLU activations maps the embedding and $t$ to the posterior parameters. For Gauge Matching these are the $d$ entries of $m_\phi$, the $d(d+1)/2$ entries of $L_\phi$, with a softplus transform on the diagonal, and the $d(d-1)/2$ entries of $\gauge_\phi$ above the diagonal, the last from a separate head of the same MLP.

\textbf{Baselines.} The three fixed gauges use the same network without the gauge head: \gls{sde} Matching, that is, the square root gauge, with a diagonal $L_\phi$, and the Cholesky and symmetric gauges with a full $L_\phi$. Helmholtz-\gls{sde} uses the network of the Cholesky gauge and, at every evaluation, sets $\gauge$ to the solution of \cref{eq:optimalgauge} with $\nabla_z \priordrift_\theta(t, m_\phi)$ in place of the expected Jacobian. By \cref{sec:optimalgaugeconsequences}, the resulting posterior process is the one that the order-one correction produces from any fixed gauge with these marginals, so no separate reference gauge is needed. The Jacobian is computed by $d$ forward-mode passes through the prior drift network, the Sylvester equation \cref{eq:optimalgauge} is solved by the Kronecker-sum formula for $d \le 4$ and by the eigendecomposition described in \cref{sec:expdetails:linear} for larger $d$, and gradients flow through the solve.

\textbf{Prior process.} The prior drift coefficient is an MLP with one hidden layer of width 100 and softplus activations, as in \citet{li2020scalable} and \citet{smith2026closing}. The diffusion coefficient is state-independent and diagonal unless stated otherwise.

\textbf{Optimization.} All models are trained with Adam \citep{kingma2015adam} and gradient clipping at norm 5, using one sample each of $t$, $\varepsilon$ and $i$ per trial and iteration as in \cref{alg:gaugematching}; in the gauge-learning, conjugate and Protocol 1 experiments the batch consists of samples rather than trials. Learning rates, batch sizes and iteration counts are listed per experiment. Unless stated otherwise, every experiment is repeated over 5 seeds, and we report the mean and twice the standard error over seeds.

\paragraph{Gap and Gauge Excess}
Two quantities recur in the experiments. The \emph{gap} of a fitted posterior process is its \gls{kl} divergence to the true posterior process (\cref{sec:exp:gauge}). For two processes with a common diffusion coefficient this divergence is, by Girsanov's theorem, an initial \gls{kl} term plus a path-space term (the display in Step 3 of \cref{sec:proofconjugate}). When both drift coefficients are affine, as for a linear prior process and every posterior process considered here, the integrand of the path-space term is a Gaussian expectation of a quadratic form, and we evaluate the gap in closed form on a time grid of step $5 \cdot 10^{-4}$.

The \emph{gauge excess} of \cref{tab:lorenz} is $\mathcal{J}(\gauge) - \mathcal{J}(\gauge^\star)$ in nats per unit time, where $\mathcal{J}$ is the integrand of \cref{prop:optimalgauge} at the marginals of the trained model, $\gauge$ is the gauge the method uses at each time $t$, and $\gauge^\star$ is the solution of \cref{eq:optimalgauge} with the expected Jacobian estimated from $10^3$ Monte Carlo samples; $\mathcal{J}$ itself is estimated from the same samples. The result is averaged over $t$ on a grid of $100$ points and over $64$ held-out trials; every gauge excess reported below uses this estimator. Since $\mathcal{J}$ is a strictly convex quadratic function of $\gauge$, the excess is nonnegative, and it is zero exactly when the method's gauge is the optimal one for its own marginals.

\subsection{Linear Dynamics}
\label{sec:expdetails:linear}

\paragraph{Learning the Gauge}
\textbf{Data.} The prior process is \cref{eq:linearSDE} with $A = -\alpha I + \omega J$, $u = 0$, $\vol = I$, $\alpha = 0.2$, $\omega = 2 \pi$ and the stationary initial law $\Norm(0, I / (2 \alpha))$ on $[0, T]$ with $T = 3$. There are no observations and hence no encoder.

\textbf{Posterior.} $m_\phi$, $L_\phi$ and the Householder vectors $v_i(t)$ of $O(t)$ are MLPs of $t$ with two hidden layers of width 64, and the gauge $\gauge = \dot O O^\top$ is obtained by forward-mode differentiation.

\textbf{Training.} We minimize the initial and path-space \gls{kl} terms of \cref{eq:elbo} with the estimator of \cref{sec:simfree} for $10^4$ iterations, with $2^{10}$ samples of $(t, \varepsilon)$ per iteration and learning rate $10^{-3}$. The initial term is needed to identify the initial law: without observations, any initial law is compatible with a zero path-space \gls{kl} term once the process has relaxed to stationarity.

\textbf{Metrics.} The gauge angle in \cref{fig:gauge} is $\theta(t) = \operatorname{atan2}(O_{21}(t), O_{11}(t))$, unwrapped, with $\theta(0)$ subtracted since the initial rotation is not identifiable. We call the path-space \gls{kl} term per unit time the \emph{rate}. The variance $0.080$ quoted in \cref{sec:exp:gauge} is that of the trained fixed-gauge process averaged over $t \in [T/6, 5T/6]$; we exclude the ends because the process has boundary layers there, where the initial \gls{kl} term at $t = 0$ and the absence of a future cost at $t = T$ let the variance rise to about $0.2$. Its rate averaged over $[0, T]$ is $6.06$ nats per unit time, which gives the $18$ nats over $[0, 3]$ quoted there. Both values agree with the closed-form optimum derived next ($0.080$ and $6.09$).

\textbf{Marginal distortion of the fixed gauge.} For stationary isotropic marginals $\Norm(0, v I)$ and $\gauge = 0$, \cref{eq:driftwhitened} gives the posterior drift coefficient $\postdrift_0(z) = -z / (2v)$, and with $Z = \sqrt{v} \, \varepsilon$ the rate is
\begin{align*}
    \frac{1}{2} \E \left[ \| A Z - \postdrift_0(Z) \|^2 \right] &= \frac{v}{2} \E \left[ \left\| \left( A + \tfrac{1}{2v} I \right) \varepsilon \right\|^2 \right] = \frac{v}{2} \tr \left( A^\top A + \tfrac{1}{v} A + \tfrac{1}{4v^2} I \right) \\
    &= (\alpha^2 + \omega^2) v - \alpha + \frac{1}{4v},
\end{align*}
using $A^\top A = (\alpha^2 + \omega^2) I$ and $\tr A = -2\alpha$. The minimizer is $v = 1 / (2 \sqrt{\alpha^2 + \omega^2})$ with minimal rate $\sqrt{\alpha^2 + \omega^2} - \alpha$, whereas at the stationary variance $v = 1 / (2\alpha)$ the rate is $\omega^2 / (2\alpha)$. For $\alpha = 0.2$ and $\omega = 2\pi$ the minimizing variance is $v = 0.080$ against the stationary $2.5$, and the corresponding rates are $6.09$ against $98.7$ nats per unit time; the learned gauge attains rate zero at the stationary variance.

The restriction to isotropic marginals is for the closed form only: the trained model is free to learn any $m_\phi$, $L_\phi$, and a numerical minimization over lower-triangular $L(t)$ confirms that the isotropic optimum is the global one.

\paragraph{Conjugate Setting}
\textbf{Data.} We follow \citet[Appendix C.2]{smith2026closing}: the prior process of \cref{sec:exp:gauge} with $T = 5$, $\alpha = 0.2$ and $\omega \in \{0, \pi/4, \pi/2, \pi, 2\pi\}$; observations $x_{t_i} \sim \Norm(C z_{t_i} + c, \sigma^2 I)$ at $N \in \{5, 10, 20, 50\}$ uniformly random times with $\sigma \in \{0.05, 0.3, 1.0\}$, where $C$ has standard normal entries normalized to unit determinant and $c$ is standard normal; 16 independent trials per configuration. The generative model is fixed at its true value, so only the posterior process is trained.

\textbf{Deviation from \citet{smith2026closing}.} Their observation noise is $\sigma^2 C C^\top$ rather than $\sigma^2 I$. Under their choice the precision $C^\top R^{-1} C = I / \sigma^2$ that an observation adds is isotropic, so the covariance of the true posterior process stays isotropic, its gauge is the constant $\omega J$, and the square root gauge fails for the reason of \cref{sec:exp:gauge} alone. Under isotropic noise the added precision $C^\top C / \sigma^2$ is anisotropic and the gauge of the true posterior process is time-dependent: by \cref{sec:proofconjugate} it jumps by $B - \ltri(B + B^\top)$ with $B = L^{\dagger -1} C^\top C L^\dagger / \sigma^2$ at every observation time, as \cref{fig:conjugate}B illustrates.

\textbf{True posterior process.} It is the Kalman--Bucy smoother \citep{Särkkä_Solin_2019}. We obtain its drift coefficient by Doob's $h$-transform as in \cref{sec:proofconjugate} on a grid of step $5 \cdot 10^{-3}$, and its gauge, shown in \cref{fig:conjugate}B, by \cref{prop:gauge} from its state matrix and the Cholesky factor of its covariance.

\textbf{Posterior.} Since the posterior process is fitted per trial, the encoder of the shared architecture is replaced by a learned per-trial embedding of dimension 100. The smoother has kinks at every observation time, which an MLP of $t$ alone resolves poorly (every posterior process stalled at a gap of about $100$ nats), so we add Fourier features of time, $(\sin(2 \pi k t / T), \cos(2 \pi k t / T))_{k \le 160}$, to the input $t$; with $k \le 8$ instead, the gap of Gauge Matching at $N = 10$, $\sigma = 0.3$, $\omega = 2\pi$ is $9.2$ nats rather than $1.5$ at the shorter budget below. The time derivatives $\dot m_\phi$, $\dot L_\phi$ pass through these features by forward-mode differentiation.

\textbf{Training.} $2^8$ samples of $t$ and of $i$ per trial and iteration; Adam with learning rate $10^{-3}$ decayed exponentially to $10^{-5}$ over $5 \cdot 10^4$ iterations for Helmholtz-\gls{sde} and Gauge Matching, and to $10^{-4}$ over $2 \cdot 10^4$ iterations for the three fixed gauges. The longer budget lowers the gap of Gauge Matching from $1.5$--$5.2$ to $0.8$--$3.9$ nats (\cref{fig:conjugate}A) and that of Helmholtz-\gls{sde} from $1.0$--$4.4$ to $0.5$--$2.8$ nats, so the two are compared at equal budget. We did not extend the fixed gauges: at $\sigma \ge 0.3$ they already lie within $0.04$--$2.2$ nats of the optima of their variational families (below), so their $14$--$26$ nats are a property of the family, not of the optimizer. Since the prior Jacobian is the constant $A$, the order-one correction of Helmholtz-\gls{sde} is exact and costs nothing beyond the Sylvester solve; the per-iteration cost comparison of \cref{sec:experiments} does not apply here.

\textbf{Metrics.} The gap is computed in closed form as above. As a reference for what training leaves on the table, we also minimize the closed-form gap directly over each variational family by L-BFGS on the grid, with a Gaussian kernel of width $0.02$ applied to the optimized functions standing in for the smoothness of a network; \cref{fig:conjugategridoptimum} shows these optima next to the trained models of \cref{fig:conjugategrid}. At $\sigma = 0.3$ and $\omega = 2\pi$ the optima are $0.3$--$2.2$ nats for the two gauge-corrected families and $12$--$26$ nats for the fixed gauges, which places Helmholtz-\gls{sde} $0.2$--$0.6$ nats, Gauge Matching $0.3$--$1.7$ nats and the fixed gauges $0.6$--$2.1$ nats above the optimum of their family.

\textbf{Remarks.} At $\sigma = 0.05$ the trained models lie $3$--$41$ nats (Helmholtz-\gls{sde} and Gauge Matching alike) and $4$--$54$ nats (fixed gauges) above their optima, so that row of \cref{fig:conjugategrid} measures the difficulty of fitting a sharply kinked smoother with a smooth network, not the effect of the gauge; the optima themselves (\cref{fig:conjugategridoptimum}) put the gauge-corrected families $2$--$4$ times below the fixed ones there as well. At $\omega = 0$ the fixed gauges retain a small gap in \cref{fig:conjugate}C, which comes from the covariance structure that the observations induce: the square root gauge, with its diagonal $L_\phi$, can represent the true posterior process only if its covariance is isotropic, which it is not for a generic $C$, while the Cholesky and symmetric gauges are close to the gauge-corrected models there (\cref{fig:conjugategrid}).

\begin{figure}[t]
    \centering
    \includegraphics[width=\linewidth]{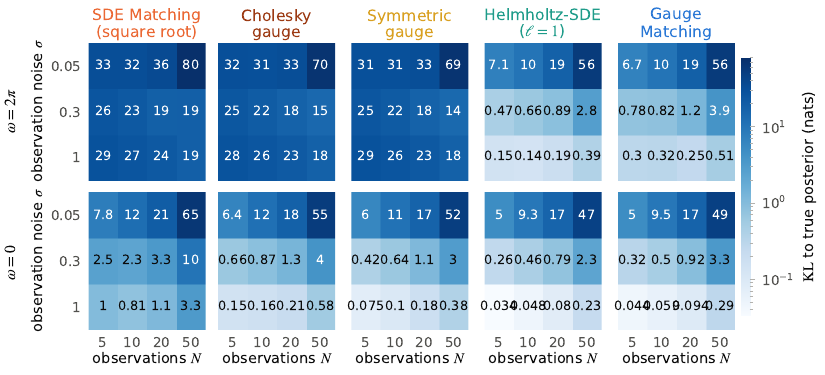}
    \caption{\gls{kl} divergence to the true posterior process in the conjugate setting for every combination of observation noise $\sigma$ and number of observations $N$, for the five posterior processes of \cref{sec:experiments}, at $\omega = 2 \pi$ (top) and $\omega = 0$ (bottom). Median over 16 trials. At $\omega = 0$ the prior drift is a gradient field, so the symmetric gauge is optimal and Helmholtz-\gls{sde} coincides with it; the learned gauge lies within a factor of $1.5$ of both, and the square root gauge retains a gap because the posterior covariance is not isotropic under a generic $C$. The residual error grows with $N$ and shrinks with $\sigma$: it is the error of a continuously differentiable posterior at the observation times. At $\sigma = 0.05$ every method is instead limited by optimization (\cref{fig:conjugategridoptimum} and the text).}
    \label{fig:conjugategrid}
\end{figure}

\begin{figure}[t]
    \centering
    \includegraphics[width=\linewidth]{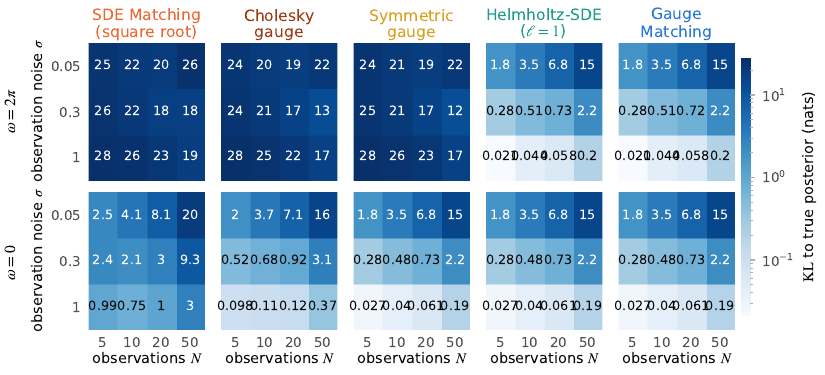}
    \caption{What the variational family admits versus what training reaches. Same grid as \cref{fig:conjugategrid}, but each entry is the minimum of the closed-form \gls{kl} divergence over the variational family of that gauge, found by L-BFGS directly on the time grid (functions smoothed by a Gaussian kernel of width $0.02$), rather than the divergence of a trained network. Median over 16 trials. The gap of the fixed gauges is a property of the family: their optima are $12$--$28$ nats at $\omega = 2\pi$, and the trained models of \cref{fig:conjugategrid} lie within $0.04$--$2.2$ nats of them at $\sigma \ge 0.3$. The gauge-corrected families admit $0.02$--$2.2$ nats there, and at $\sigma = 0.05$ they admit $1.8$--$15$ nats against the $5$--$56$ nats that training reaches. Helmholtz-\gls{sde} at order one and the learned gauge have the same optimum, since the order-one correction is exact for a linear prior drift; at $\omega = 0$ the symmetric gauge coincides with both.}
    \label{fig:conjugategridoptimum}
\end{figure}

\paragraph{Scaling in the Latent Dimension}
\textbf{Data.} The prior process is the $d$-dimensional rotating Ornstein--Uhlenbeck process $\diff Z = (-\alpha I + \omega K) Z \diff t + \diff W$ with $\alpha = 0.2$, $\omega = 2\pi$ and $T = 5$ as in the conjugate setting, where $K = (G - G^\top) / \|G - G^\top\|_2$ for a standard normal $d \times d$ matrix $G$: a random skew-symmetric matrix of unit spectral norm, which rotates $\lfloor d/2 \rfloor$ random planes at angular velocities of at most $\omega$. Since $A + A^\top = -2\alpha I$, the stationary law is $\Norm(0, I / (2\alpha))$ for every $K$, and the transition $e^{A s} = e^{-\alpha s} e^{\omega K s}$ with covariance $(1 - e^{-2\alpha s}) / (2\alpha) \, I$ lets us sample the latent state exactly at the observation times.

Observations are $x = C z + c + \Norm(0, \sigma^2 I)$ with $\sigma = 0.3$, $C$ a $\lceil d/2 \rceil \times d$ matrix with orthonormal rows (the transpose of the $Q$ factor of a standard normal $d \times \lceil d/2 \rceil$ matrix: a random half of the state is observed, the other half is inferred through the dynamics) and $c$ standard normal, at $N = 10$ distinct uniformly random times per trial on a grid of step $5 \cdot 10^{-3}$. One system $(K, C, c)$ is drawn per seed and shared by all trials and all five posterior processes; $d \in \{2, 4, 8, 16, 32, 64\}$, 3 seeds, the seed of the system being also that of the model.

\textbf{Protocol 1 (known generative model).} As in the conjugate setting, the prior process, the initial law and the observation model are fixed at their true values, and the posterior process is fitted to 16 trials at once with the architecture of the conjugate setting (per-trial embedding, Fourier time features) and heads for $m_\phi$, $L_\phi$ and, for Gauge Matching, the $d(d-1)/2$ entries of $\gauge_\phi$. Training: $2^8$ samples of $t$ and of $i$ per trial and iteration, Adam with learning rate $10^{-3}$ decayed exponentially to $10^{-5}$ over $3 \cdot 10^4$ iterations, the same budget for all five gauges. Metric: the closed-form gap to the Kalman--Bucy smoother in $d$ dimensions, as median and interquartile range over the $3 \times 16$ trials.

\textbf{Protocol 2 (learned generative model).} The protocol of the Lorenz system with a $d$-dimensional latent state: a Gaussian initial law with learned mean and diagonal covariance, the shared prior drift network, the fixed diffusion coefficient $\vol = I$, a linear-Gaussian observation model with learned $C$ and $c$ and the fixed noise $\sigma$, and the amortized GRU posterior of width 100; 1024 training and 2048 test trials; $3 \cdot 10^4$ iterations, batches of 256 trials, learning rate $10^{-3}$. Metrics on the test trials as in \cref{tab:lorenz}: the negative \gls{elbo} on a grid of 500 times with 4 noise samples, the latent log-likelihood of the noise-free $C z + c$ under the posterior predictive at the observation times, and the gauge excess.

\textbf{Sylvester solve.} The symmetric gauge, Helmholtz-\gls{sde} and the gauge excess solve \cref{eq:optimalgauge}. For $d \le 4$ we solve the $d^2 \times d^2$ Kronecker-sum system directly; for $d \ge 8$ we use the eigendecomposition $G = V \Lambda V^\top$ of the symmetric positive definite $G$, $\gauge = V \bigl[ (V^\top R V)_{ij} / (\lambda_i + \lambda_j) \bigr] V^\top$ with $R$ the right-hand side of \cref{eq:optimalgauge}, at cost $\mathcal{O}(d^3)$ per sample, and obtain its gradient from the adjoint Sylvester equation with the same eigendecomposition rather than by differentiating through the eigendecomposition, which is singular for repeated eigenvalues. The two solvers agree to machine precision.

\textbf{Timing.} \Cref{fig:dimsweep}C times one training iteration of the models of Protocol 2 (forward and backward pass, gradient clipping and Adam step, batch of 256 trials) for $d \in \{2, 4, \ldots, 128\}$ on one otherwise idle GPU, as the median and interquartile range over 50 iterations after 5 warm-up iterations with device synchronization, and on a CPU for $d \le 64$ (10 iterations).

\begin{figure}[t]
    \centering
    \includegraphics[width=\linewidth]{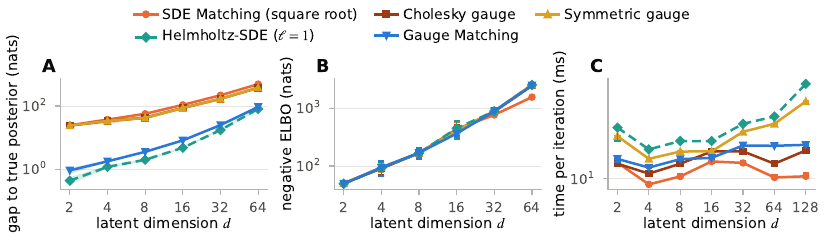}
    \caption{Scaling in the latent dimension on the rotating Ornstein--Uhlenbeck system. \emph{A:} gap to the true posterior process, generative model known (median and interquartile range over 16 trials and 3 seeds). \emph{B:} negative \gls{elbo} per trial, generative model learned (mean $\pm 2$ standard errors over 3 seeds). \emph{C:} GPU time per training iteration, models as in B up to $d = 128$ (batch of 256 trials; median and interquartile range over 50 iterations).}
    \label{fig:dimsweep}
\end{figure}

\textbf{Results.} With the generative model known (\cref{fig:dimsweep}A), the gap of Helmholtz-\gls{sde}, whose gauge is exact here, grows from $0.4$ nats at $d = 2$ to $81$ at $d = 64$ ($0.4$/$1.2$/$2.0$/$4.8$/$17$/$81$ at $d = 2$/$4$/$8$/$16$/$32$/$64$) and that of Gauge Matching from $0.9$ to $94$ ($0.9$/$1.8$/$3.6$/$8.2$/$25$/$94$), while the fixed gauges lie at $24$--$25$ and $381$--$496$ (the Cholesky and symmetric gauges agree at every $d$, since the symmetric gauge is nearly zero for a drift that is an isotropic contraction plus a rotation). The part of the gap due to the mean is the same for the two gauge-corrected processes at every $d$ ($0.06$ at $d = 2$, $63$ at $d = 64$; it is the training error of the fit to 16 trials). The part due to the covariance, which the gauge controls, is $0.4$--$17$ for Helmholtz-\gls{sde} and $0.9$--$32$ for Gauge Matching against $24$--$393$ for the fixed gauges: the learned gauge removes $89$--$97\%$ of the fixed gauges' covariance gap at every $d$, and its norm is within $2\%$ of that of the exact gauge.

With a learned generative model (\cref{fig:dimsweep}B) the five posterior processes reach the same negative \gls{elbo} within seed noise for $d \le 16$. The fitted prior drift recovers little of the rotation from ten partial observations per trial (the imaginary parts of the eigenvalues of its mean Jacobian are $0.02$--$1.1$ against the true $2\pi$), so it is close to a gradient field, for which every gauge is near-optimal, and the gauge excess of the fixed gauges is only $0.02$--$0.4$ nats per unit time. Pilot runs with $N = 20$ or $40$ observations, full observation of the state, or slower rotation ($\omega = \pi/2$, $\pi$) at $d = 8$ did not change this.

For $d \ge 32$ the diagonal posterior of \gls{sde} Matching reaches a lower negative \gls{elbo} than the four full-covariance posteriors, which agree with each other ($1554 \pm 78$ against $2444$--$2538$ at $d = 64$). The difference lies entirely in the reconstruction term: the learned observation matrix of every full-covariance posterior is rank deficient ($5$--$8$ of its $32$ singular values below $0.05$ at $d = 64$, $1$--$3$ of $16$ at $d = 32$, none for the diagonal posterior), so these models settle in a local optimum in which a subset of latent directions is not observed and reverts to the prior, which lowers the path-space \gls{kl} term and raises the reconstruction term. This is common to all four full-covariance gauges and not an effect of the gauge; training three times longer at $d = 64$ (one seed) halves the difference in training loss, and the same collapse appears in the identifiable setting below when the observation model is learned.

Per iteration (\cref{fig:dimsweep}C), the GPU step is dominated by kernel launches up to $d \approx 32$, so the Sylvester solve costs Helmholtz-\gls{sde} and the symmetric gauge a factor of $1.3$--$1.6$ over \gls{sde} Matching there and $2.2$ ($3.2$) at $d = 64$ ($128$), while Gauge Matching costs $1.05$--$1.5$; on the CPU the $\mathcal{O}(d^3)$ solve dominates from $d = 32$ on, with Helmholtz-\gls{sde} at $6.4$--$7.3$ times \gls{sde} Matching against $1.4$--$1.7$ for Gauge Matching.

\paragraph{Scaling in the Latent Dimension: Identifiable Setting}
Protocol 2 above measures the five posterior processes on a prior that the data barely identify. To measure them when the rotation is identified, we repeat it at the setting of \cref{tab:rousd}, with additive noise so that Helmholtz-\gls{sde} applies.

\textbf{Data.} The rotating Ornstein--Uhlenbeck process above with $\alpha = 0.5$, $\omega = \pi$ and $\vol = I$ (angular velocities of at most $\pi$, eight observations per period of the fastest plane), $d \in \{8, 16, 32\}$, observed in full: $C$ is a random $d \times d$ orthogonal matrix, $c$ standard normal, $N = 21$ distinct uniformly random observation times per trial on $[0, 5]$, noise $\sigma = 0.3$; 1024 training and 2048 test trials, one system per seed, 3 seeds.

\textbf{Model.} As in Protocol 2, except that the observation model $(C, c, \sigma)$ is fixed at its true value, as for the Lotka--Volterra system and \cref{tab:rousd}: the latent basis is then pinned by the data, and the learned prior drift, initial law and posterior process must explain every latent direction. Five posterior processes: \gls{sde} Matching, the Cholesky gauge, Helmholtz-\gls{sde}, Gauge Matching, and Gauge Matching with a diagonal $L_\phi$ (\cref{sec:ablationgauge}); the symmetric gauge is omitted, since it coincides with the Cholesky gauge on this prior (Protocol 1).

\textbf{Training.} $3 \cdot 10^4$ iterations, batches of 256 trials, learning rate $10^{-3}$, as in Protocol 2; at $d = 16$ and $32$ all five posterior processes were also trained for $6 \cdot 10^4$ iterations.

\textbf{Metrics.} As in Protocol 2. Since the seed-to-seed spread of the negative \gls{elbo} is that of the random system ($457$--$489$ nats at $d = 16$ for one gauge), we report paired differences to Helmholtz-\gls{sde} per seed as the median over the 2048 held-out trials of the per-trial difference (mean and $2$ standard errors over seeds). The median is used because a few held-out trials, for which the amortized posterior is far off, have a negative \gls{elbo} above $10^4$ nats (at most $15$ of $2048$ per seed, plus one non-finite trial for Gauge Matching): at $3 \cdot 10^4$ iterations this occurs only for the gauge-corrected models at $d = 32$, at $6 \cdot 10^4$ iterations for every full-covariance posterior and on single trials for the two diagonal ones. We also report the rotation rates of the learned drift, the imaginary parts of the eigenvalues of the skew part of its mean Jacobian over posterior samples, against the true $\omega |\lambda_j(K)|$.

\textbf{Results.} \Cref{fig:dimsweepident}. At $d = 8$ the negative \gls{elbo} is $218$--$219$ nats for the two fixed gauges against $212$--$214$ for the three gauge-corrected posterior processes; paired per trial and per seed, \gls{sde} Matching is $3$--$9$ nats and the Cholesky gauge $4$--$7$ nats above Helmholtz-\gls{sde} on every seed, Gauge Matching is $0.5$--$3$ above it and Gauge Matching with a diagonal $L_\phi$ within $1.5$ of it. The difference is the path-space \gls{kl} term ($112$--$113$ for the fixed gauges against $104$--$107$), of which Gauge Matching gives $3$--$5$ nats back in the reconstruction term. The gauge excess of the fixed gauges is $0.54$--$0.80$ nats per unit time against $0.01$--$0.02$ for the two learned gauges, whose $\|\gauge_\phi\|_F$ matches $\|\gauge^\star\|_F$ within $2\%$. Unlike in Protocol 2, the fitted prior rotates, at $83$--$86\%$ of the true rates for the gauge-corrected posteriors and $69$--$75\%$ for the fixed gauges, which as in \cref{tab:rousd} absorb part of the rotation into the drift they are trained with.

At $d = 16$ and $32$ the comparison at equal $L_\phi$ still favors the gauge, and more so at the longer budget: the Cholesky gauge is $1.4$--$11.6$ nats above Helmholtz-\gls{sde} at $3 \cdot 10^4$ iterations, and $5$--$21$ above Helmholtz-\gls{sde} and $5$--$30$ above Gauge Matching at $6 \cdot 10^4$, on every seed at both $d$. Gauge Matching with a diagonal $L_\phi$ is $3$--$14$ nats below \gls{sde} Matching at $d = 16$ (on two of three seeds at $3 \cdot 10^4$ iterations, on every seed at $6 \cdot 10^4$; gauge excess $0.04$--$0.05$ against $0.9$--$1.1$), but not consistently at $d = 32$: at $6 \cdot 10^4$ iterations it is $26$ nats above on one seed and $4$--$13$ below on the other two, its $\|\gauge_\phi\|_F$ exceeds $\|\gauge^\star\|_F$ by $7$--$9\%$, and its gauge head diverges near the end of the interval on one seed.

Across parameterizations, however, the diagonal $L_\phi$ is fitted faster than a full one. At $3 \cdot 10^4$ iterations \gls{sde} Matching is within $8$ nats of Helmholtz-\gls{sde} at $d = 16$, the sign depending on the seed, Gauge Matching with a diagonal $L_\phi$ has the lowest seed mean there, and the five models agree within $12$ nats at $d = 32$. At $6 \cdot 10^4$ iterations Gauge Matching with a diagonal $L_\phi$ is the best model at $d = 16$ on every seed, $8$--$9$ nats below Helmholtz-\gls{sde}, which \gls{sde} Matching and Gauge Matching tie within $5$ and $2$ nats; at $d = 32$ \gls{sde} Matching is $24$--$27$ nats below Helmholtz-\gls{sde} and $9$--$26$ below Gauge Matching on every seed, in test and training loss alike, although its gauge excess is $1.7$ nats per unit time, $8$ nats over the trajectory, against $0.15$ for the learned gauge.

\textbf{Remarks.} With the observation model learned instead of fixed, the same sweep gives a different picture (\cref{fig:dimsweepident}C): all four full-covariance posteriors converge to a rank-deficient observation matrix, with $3$--$5$ of $16$ ($d = 16$) and $13$--$15$ of $32$ ($d = 32$) singular values below $0.05$ against $0$--$2$ for the diagonal one, and their reconstruction term is $650$ and $1950$ nats against $218$ and $660$; the ordering of the gauges within the full-covariance family is unchanged (Cholesky $7.1$ and $1.1$ nats above Helmholtz-\gls{sde}, paired medians). This is the mechanism behind the lag of the full-covariance posteriors at $d \ge 32$ in Protocol 2, and the reason why the identifiable setting fixes $C$.

\begin{figure}[t]
    \centering
    \includegraphics[width=\linewidth]{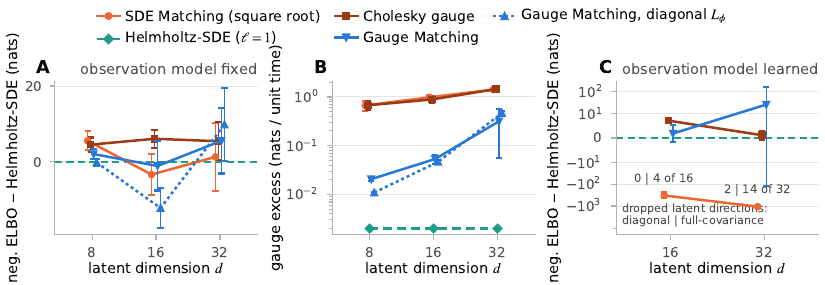}
    \caption{Dimension sweep at an identifiable setting (learned prior drift, full observation, $N = 21$, $\omega = \pi$; see the text). A: negative \gls{elbo} minus that of Helmholtz-\gls{sde}, paired per seed as the median over held-out trials of the per-trial difference (mean and $2$ standard errors over 3 seeds), with the observation model fixed. B: gauge excess (Helmholtz-\gls{sde} is zero by construction and drawn at the floor of the axis). C: as A, with the observation model learned: the full-covariance posteriors drop latent directions (numbers: mean over seeds for the diagonal $\mid$ full-covariance posteriors) and lose by hundreds of nats to the diagonal one.}
    \label{fig:dimsweepident}
\end{figure}

\subsection{Nonlinear Dynamics}

\paragraph{Noisy Lorenz Attractor}
\textbf{Data.} As in \citet[Appendix D.1]{smith2026closing}: the stochastic Lorenz system with parameters $(10, 28, 8/3)$, additive noise of scale $5$ and initial state $\Norm(0, I)$ is simulated on $[0, 5]$ by Euler--Maruyama with step $10^{-3}$; the standardized state is observed every $\Delta t = 0.25$ with Gaussian noise of standard deviation $0.3$; 1024 training and 2048 test trials.

\textbf{Model.} The latent \gls{sde} has $d = 4$; a Gaussian initial distribution with learned mean and diagonal covariance; the shared prior drift network; the fixed diffusion coefficient $\vol = I$ for all methods, including \gls{sde} Matching, which in \citet{bartosh2025sde} uses a state-dependent one; and a linear-Gaussian observation model $\Norm(x \mid C z + c, R)$ with learned $C$ and $c$ and fixed $R$.

\textbf{Training.} $3 \cdot 10^4$ iterations, batches of 256 trials, learning rate $10^{-3}$, for all methods.

\textbf{Metrics.} The \emph{negative \gls{elbo}} is \cref{eq:elbo} per test trial. Its initial \gls{kl} term is $0.5$--$0.6$ nats for every method; its path-space \gls{kl} term is $57.3$--$59.6$ for the three fixed gauges against $33.5$ and $33.7$ for Helmholtz-\gls{sde} and Gauge Matching, which pay about one nat more in the reconstruction term ($19.0$ and $18.9$ against $17.6$--$17.7$) for a $24$--$26$ nat lower path-space term.

The \emph{latent log-likelihood} follows \citet{smith2026closing}: writing $y^{(k)}(t_i)$ for the noise-free standardized true state of test trial $k$ at observation time $t_i$, and $m_\phi$, $P_\phi = L_\phi L_\phi^\top$ for the posterior marginal of that trial at $t_i$, it is $\frac{1}{N K} \sum_{k=1}^{K} \sum_{i=1}^{N} \log \Norm\bigl(y^{(k)}(t_i) \mid C m_\phi(t_i) + c, \, C P_\phi(t_i) C^\top + R\bigr)$ over the $N$ observation times and $K = 2048$ test trials, with the learned observation model $(C, c, R)$: the log-density of the true state under the posterior predictive.

The \emph{autocorrelation error} is computed from 2048 sample paths of the learned prior process started from the posterior at $t = 0$. For lags $\tau$ on a grid of step $0.05$ in $[0, 2]$, we form the $3 \times 3$ correlation matrix of the observed coordinates between times $t$ and $t + \tau$, average the squared Frobenius distance to the corresponding matrix of the true prior process over $t$ and $\tau$, and report the root mean square over entries. The Monte Carlo floor of this estimator at this sample size is $0.024$, and the true dynamics started from the fitted posterior at $t = 0$ score $0.030$; the gauge-corrected models, at $0.052$--$0.056$, are therefore within $0.03$ of what the estimator can resolve. \Cref{fig:lorenzsamples} shows sample paths of each learned prior process against the true one.

The \emph{gauge excess} is defined above. The \emph{training time} quoted in \cref{sec:exp:dimsweep} is the wall-clock time to $3 \cdot 10^4$ iterations, averaged over seeds; the five methods shared one GPU, so only ratios are meaningful: relative to \gls{sde} Matching, the Cholesky gauge, the symmetric gauge, Helmholtz-\gls{sde} and Gauge Matching took $1.06$, $1.27$, $1.43$ and $1.17$ times as long. On the Stuart--Landau oscillator at $\kappa = 3$, measured the same way, Helmholtz-\gls{sde} took $1.57$ times as long as Gauge Matching and Gauge Matching $1.29$ times as long as \gls{sde} Matching.

\textbf{Comparison with published numbers.} Our negative \gls{elbo} values are $6$--$10$ nats above those reported by \citet{smith2026closing} for the same methods, uniformly across methods, so the differences between \gls{sde} Matching and Helmholtz-\gls{sde} agree with theirs. The offset is not overfitting, since training and test values agree to within $1$ nat; it is consistent with a difference in batch size or data draw and, for \gls{sde} Matching, with our use of the fixed diffusion coefficient.

\textbf{Robustness to observation spacing and noise.} We repeat the experiment for the observation spacings $\Delta t \in \{0.125, 0.25, 0.5\}$, that is $N = 41$, $21$ or $11$ observations per trial on $[0, 5]$, and the observation noise levels $\sigma \in \{0.15, 0.3, 0.6\}$; $(0.25, 0.3)$ is the protocol of \cref{tab:lorenz}. The same 1024 training and 2048 test trajectories underlie every protocol, so only the sampling of the observations changes; the model, the training and the metrics are those of \cref{tab:lorenz}, with the fixed observation noise $R = \sigma^2 I$ set to the true value, for \gls{sde} Matching, the Cholesky gauge, Helmholtz-\gls{sde} and Gauge Matching (the symmetric gauge, which tracks the Cholesky gauge in \cref{tab:lorenz}, is omitted), three seeds per protocol.

The negative \gls{elbo} is a trimmed mean over the test trials that excludes trials more than ten interquartile ranges above the median; this excludes no trial in most runs and never changes a mean by more than $0.5$ nats except where the plain mean is infinite (see the remarks). \Cref{tab:lorenzrobustness} and \cref{fig:lorenzrobustness} report the results.

For $\Delta t \le 0.25$ and $\sigma \le 0.3$ the fixed gauges trail Helmholtz-\gls{sde} by $20$--$25$ nats in every protocol (their \emph{deficit}, the difference of the negative \glspl{elbo}), all of it in the path-space \gls{kl} term ($+25$ to $+29$ nats against $-1$ to $-6$ in the reconstruction term), while Gauge Matching is within $-0.7$ to $+1.8$ nats of Helmholtz-\gls{sde} in every protocol; the gauge excess of the fixed gauges is $2.5$--$3.1$ nats per unit time and that of Gauge Matching $0.04$--$0.07$, and the norm of the learned gauge matches that of the optimal gauge at its own marginals to within $1\%$.

The deficit of the fixed gauges does not grow with sparser or noisier observations: it disappears. At $\sigma = 0.6$ it is $18$--$21$ nats for $\Delta t = 0.125$ but $9$ nats for $\Delta t = 0.25$, and at $\Delta t = 0.5$ it is $-0.5$ to $+4.8$ nats, within two standard errors of zero in five of the six cases. The autocorrelation error explains why: it is $0.03$--$0.06$ for the gauge-corrected models and $0.04$--$0.13$ for the fixed gauges where the deficit is large, but $0.23$--$0.25$ for all four posterior processes at $\Delta t = 0.5$, so with eleven observations per trial the fitted prior drift no longer recovers the rotation within the lobes; the optimal gauge at the fixed-gauge marginals shrinks from $\|\gauge^\star\| \approx 12$ to $0.3$--$4$ and their gauge excess to below $0.8$ nats per unit time.

\textbf{Remarks.} On rare test trials ($1$--$6$ of $2048$ in twelve of the $96$ new runs) the amortized posterior returns a marginal covariance with one eigenvalue of order $10^{-9}$. For Helmholtz-\gls{sde} the Sylvester system of \cref{eq:optimalgauge} is then nearly singular, the gauge it returns has norm $10^{6}$--$10^{12}$, and the path-space \gls{kl} term of that trial is $10^{8}$--$10^{12}$ nats, which is why the plain mean of two runs is $10^{8}$ or infinite; the same trials are unremarkable for Gauge Matching, whose $\gauge_\phi$ is a bounded network output. Two training runs diverged, Helmholtz-\gls{sde} at $(0.5, 0.3)$ and \gls{sde} Matching at $(0.5, 0.6)$, in both cases a posterior standard deviation collapsing within a few iterations, reproducibly for the seed; they were replaced by a further seed. Nothing of this kind occurred for $\Delta t \le 0.25$.

\begin{figure}[t]
    \centering
    \includegraphics[width=\linewidth]{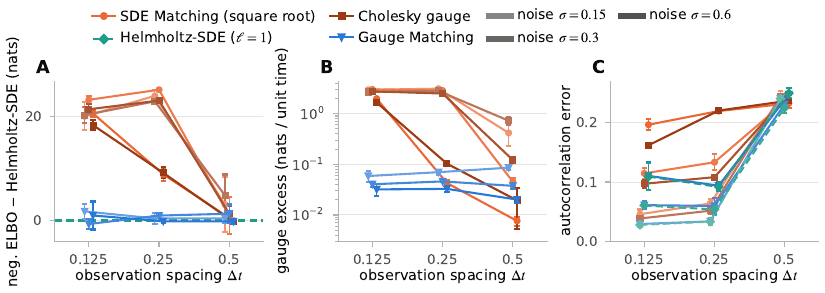}
    \caption{Noisy Lorenz attractor under nine observation protocols: spacing $\Delta t \in \{0.125, 0.25, 0.5\}$ ($N = 41, 21, 11$ observations per trial) and noise $\sigma \in \{0.15, 0.3, 0.6\}$ (light to dark), $d = 4$, the model and training of \cref{tab:lorenz}, three seeds per protocol. A: negative \gls{elbo} relative to Helmholtz-\gls{sde}, paired over seeds. B: gauge excess (log scale). C: autocorrelation error of the learned prior process. Mean $\pm 2$ standard errors over seeds.}
    \label{fig:lorenzrobustness}
\end{figure}

\begin{table}[t]
    \centering
    \caption{Noisy Lorenz attractor under nine observation protocols (spacing $\Delta t$, $N$ observations per trial, noise $\sigma$); the row $(0.25, 21, 0.3)$ is the protocol of \cref{tab:lorenz}, here over its first three seeds. Negative \gls{elbo} per test trial and gauge excess in nats per unit time, mean $\pm 2$ standard errors over three seeds; the gauge excess of the Cholesky gauge is $10\%$ below the square root's where the deficit is large, and that of Helmholtz-\gls{sde} is $0.000$--$0.005$.}
    \label{tab:lorenzrobustness}
    \small
    \setlength{\tabcolsep}{3pt}
    \begin{tabular}{ccccccccc}
        \toprule
        & & & \multicolumn{4}{c}{Negative \gls{elbo} $\downarrow$} & \multicolumn{2}{c}{Gauge excess $\downarrow$} \\
        \cmidrule(lr){4-7} \cmidrule(lr){8-9}
        $\Delta t$ & $N$ & $\sigma$ & Sq.\ root & Cholesky & Helmholtz & Gauge M. & Sq.\ root & Gauge M. \\
        \midrule
        $0.125$ & $41$ & $0.15$ & $45.8 \pm 2.1$ & $46.0 \pm 1.1$ & $25.7 \pm 0.7$ & $27.5 \pm 2.1$ & $2.97 \pm 0.02$ & $0.059 \pm 0.014$ \\
        $0.125$ & $41$ & $0.3$ & $98.2 \pm 0.4$ & $96.4 \pm 0.8$ & $75.0 \pm 0.3$ & $74.3 \pm 0.7$ & $3.02 \pm 0.09$ & $0.040 \pm 0.005$ \\
        $0.125$ & $41$ & $0.6$ & $164.6 \pm 0.1$ & $162.4 \pm 0.2$ & $144.1 \pm 0.8$ & $145.0 \pm 2.1$ & $1.98 \pm 0.06$ & $0.032 \pm 0.009$ \\
        $0.25$ & $21$ & $0.15$ & $60.8 \pm 1.1$ & $59.7 \pm 0.5$ & $36.8 \pm 0.5$ & $37.2 \pm 1.4$ & $3.12 \pm 0.04$ & $0.070 \pm 0.007$ \\
        $0.25$ & $21$ & $0.3$ & $77.8 \pm 0.5$ & $75.7 \pm 0.3$ & $52.6 \pm 0.4$ & $53.5 \pm 0.8$ & $2.86 \pm 0.10$ & $0.046 \pm 0.001$ \\
        $0.25$ & $21$ & $0.6$ & $92.7 \pm 0.6$ & $93.0 \pm 1.7$ & $83.9 \pm 0.7$ & $83.7 \pm 0.5$ & $0.04 \pm 0.01$ & $0.033 \pm 0.004$ \\
        $0.5$ & $11$ & $0.15$ & $49.3 \pm 5.2$ & $51.1 \pm 3.2$ & $46.3 \pm 1.0$ & $46.7 \pm 0.9$ & $0.42 \pm 0.19$ & $0.086 \pm 0.009$ \\
        $0.5$ & $11$ & $0.3$ & $41.6 \pm 1.2$ & $42.4 \pm 0.2$ & $40.6 \pm 1.4$ & $42.4 \pm 0.0$ & $0.05 \pm 0.01$ & $0.038 \pm 0.005$ \\
        $0.5$ & $11$ & $0.6$ & $47.6 \pm 0.5$ & $47.9 \pm 0.2$ & $48.1 \pm 0.5$ & $48.0 \pm 0.0$ & $0.01 \pm 0.00$ & $0.020 \pm 0.003$ \\
        \bottomrule
    \end{tabular}
\end{table}

\begin{figure}[t]
    \centering
    \includegraphics[width=\linewidth]{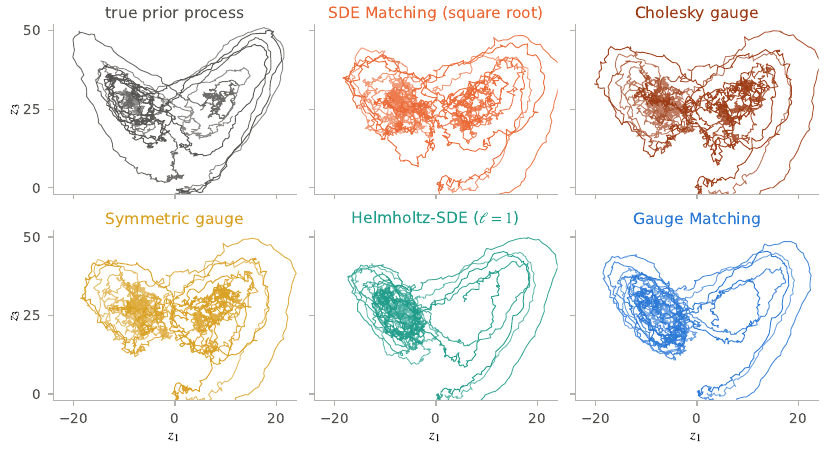}
    \caption{Three samples from the true prior process of the noisy Lorenz attractor and from the prior processes learned with each posterior process. The fixed gauges miss the outward spiral within each lobe of the attractor; the learned gauge and Helmholtz-\gls{sde} recover it.}
    \label{fig:lorenzsamples}
\end{figure}

\paragraph{Known Nonlinear Drift: Stuart--Landau Oscillator}
To isolate the effect of the expected Jacobian in \cref{eq:optimalgauge} from the learning of the prior, we take a two-dimensional oscillator with a cubic drift coefficient and fix the generative model at the truth, as in the conjugate setting, so that only the posterior process is learned.

\textbf{Why the drift must be known.} With a learned prior drift, Gauge Matching and Helmholtz-\gls{sde} coincide on every system we tried, including van der Pol and Stuart--Landau oscillators: a drift network fitted through posteriors as wide as those encountered here is close to affine over the posterior support, its expected Jacobian equals its Jacobian at the mean to $10^{-3}$, and the order-one correction is exact for it. This, and not the degree of the true drift, is why the Lorenz and Lotka--Volterra systems do not separate the two methods.

\textbf{Which oscillator.} With the true drift fixed, the separation depends on the oscillator. For $d = 2$ and isotropic $\Sigma = s^2 I$ the gauge excess of the order-one correction has the closed form $\delta^2 \det P / (2 s^2 \tr P)$, where $P$ is the posterior covariance and $\delta = \Delta J_{21} - \Delta J_{12}$ is twice the off-diagonal entry of the skew-symmetric part of $\Delta J = \E_{q(z)}[\nabla_z \priordrift(Z)] - \nabla_z \priordrift(m)$, the difference between the expected Jacobian and the Jacobian at the mean; the symmetric part of $\Delta J$ does not enter.

For the van der Pol oscillator $\priordrift(z) = (z_2, \, \mu (1 - z_1^2) z_2 - z_1)$ one finds $\delta = -2 \mu P_{12}$, which vanishes for the nearly isotropic posteriors of a well-observed limit cycle (in pilot runs with $\mu \in \{2, 3\}$ the excess of Helmholtz-\gls{sde} stayed below $0.001$ nats), and the Duffing oscillator behaves alike. We therefore use the Stuart--Landau oscillator with shear, $\priordrift(z) = (a_0 - |z|^2) z + (\omega + \kappa |z|^2) J z$, for which $\delta = 4 \kappa \tr P$ and the excess is $8 \kappa^2 \tr P \det P / s^2$.

\textbf{Data.} $a_0 = \omega = 1$, $\kappa \in \{2, 3\}$, additive noise $s = 0.5$, $z(0) \sim \Norm(0, I)$, on $[0, 10]$ by Euler--Maruyama with step $10^{-3}$; observations of both coordinates with Gaussian noise of standard deviation $0.5$ at $t = 0, 1, \ldots, 10$ ($N = 11$); 1024 training and 2048 test trials.

\textbf{Model.} The prior drift, the diffusion coefficient $s I$, the initial law $\Norm(0, I)$ and the observation model $C = I$, $c = 0$ are fixed at their true values.

\textbf{Posterior.} The shared architecture, with a single gauge entry $\gauge_{12}$ for Gauge Matching.

\textbf{Training.} $3 \cdot 10^4$ iterations, batches of 256 trials, learning rate $10^{-3}$, 5 seeds.

\textbf{Metrics.} Those of \cref{tab:lorenz} on the 2048 test trials. The negative \gls{elbo} is evaluated on a grid of 500 times with 4 noise samples each and excludes the rare test trials whose path-space \gls{kl} term exceeds $10^3$ nats: the cubic drift makes this term explode for a trial when the amortized posterior is badly off for it (one of $5 \times 2048$ trials for Helmholtz-\gls{sde} at $\kappa = 3$; three in total at $\kappa = 2$). The latent log-likelihood is evaluated at the observation times. The autocorrelation error uses lags in $[0, 4]$ and 2048 prior sample paths started from the posterior at $t = 0$; its Monte Carlo floor is $0.022$, and the true dynamics started from the fitted posterior score $0.040$--$0.044$ at $\kappa = 3$. The gauge excess uses the expected Jacobian of the true drift.

\textbf{Results.} At $\kappa = 3$ the mean posterior variance is $P_{11} \approx P_{22} \approx 0.06$, for which the closed form predicts a gauge excess of the order-one correction of $0.12$--$0.14$ nats per unit time; we measure $0.135 \pm 0.006$, against $0.014 \pm 0.003$ for Gauge Matching and $0.66$--$0.69$ for the fixed gauges (\cref{tab:hopf}, \cref{fig:hopf}). The negative \gls{elbo} of Gauge Matching is lower than that of Helmholtz-\gls{sde} by $1.6 \pm 0.3$ nats, paired over seeds and in the same direction for all five: it gains $3.0$ nats in the path-space \gls{kl} term and $0.1$ in the initial \gls{kl} term, and gives $1.5$ back in the reconstruction term through a slightly wider posterior, whose latent log-likelihood is accordingly $0.025$ lower. At $\kappa = 2$ the gauge excess of the order-one correction drops to $0.052 \pm 0.004$ ($0.028 \pm 0.003$ for Gauge Matching), roughly the factor $(2/3)^2$ that the closed form predicts, and the \gls{elbo} difference, $0.1 \pm 0.8$ nats, is within seed noise.

\begin{table}[t]
    \centering
    \caption{Stuart--Landau oscillator with shear $\kappa = 3$, $d = 2$, generative model known, posterior process learned. Columns as in \cref{tab:lorenz} without the time, with the autocorrelation error over lags in $[0, 4]$; the negative \gls{elbo} excludes held-out trials whose path-space \gls{kl} term exceeds $10^3$ nats (one of $5 \times 2048$ for Helmholtz-\gls{sde}, none otherwise). Mean $\pm 2$ standard errors over 5 seeds; the seed-paired \gls{elbo} difference between Gauge Matching and Helmholtz-\gls{sde} is $1.6 \pm 0.3$ nats.}
    \label{tab:hopf}
    \small
    \setlength{\tabcolsep}{5pt}
    \begin{tabular}{lcccc}
        \toprule
        Posterior process & Negative \gls{elbo} $\downarrow$ & Latent LL $\uparrow$ & Autocorr.\ error $\downarrow$ & Gauge excess $\downarrow$ \\
        \midrule
        \gls{sde} Matching (square root) & $64.5 \pm 0.1$ & $-1.89 \pm 0.01$ & $0.054 \pm 0.001$ & $0.685 \pm 0.020$ \\
        Cholesky gauge & $64.3 \pm 0.1$ & $-1.89 \pm 0.02$ & $0.053 \pm 0.002$ & $0.658 \pm 0.015$ \\
        Symmetric gauge & $64.5 \pm 0.2$ & $-1.89 \pm 0.01$ & $0.054 \pm 0.001$ & $0.686 \pm 0.017$ \\
        Helmholtz-\gls{sde} ($\ell = 1$) & $54.1 \pm 0.1$ & $-1.67 \pm 0.01$ & $0.043 \pm 0.002$ & $0.135 \pm 0.006$ \\
        Gauge Matching & $52.5 \pm 0.3$ & $-1.69 \pm 0.01$ & $0.044 \pm 0.004$ & $0.014 \pm 0.003$ \\
        \midrule
        Gauge Matching, diagonal $L_\phi$ & $52.9 \pm 0.2$ & $-1.71 \pm 0.01$ & $0.045 \pm 0.003$ & $0.012 \pm 0.003$ \\
        \bottomrule
    \end{tabular}
\end{table}

\begin{figure}[t]
    \centering
    \includegraphics[width=\linewidth]{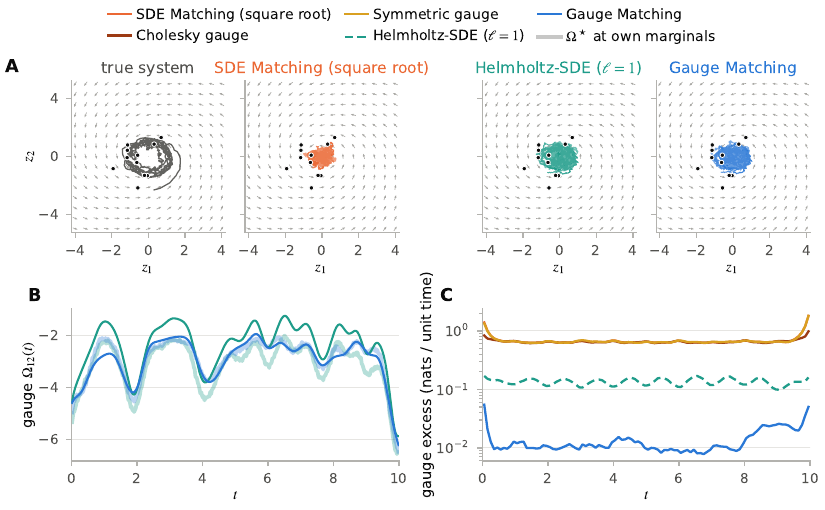}
    \caption{Stuart--Landau oscillator with shear $\kappa = 3$, known generative model. \emph{A:} the true drift field with the trajectory and observations (dots) of one test trial, and three sample paths of the true prior process started from the initial marginal $q_\phi(0, z \mid \obs)$ fitted by \gls{sde} Matching, Helmholtz-\gls{sde} and Gauge Matching; since the generative model is fixed at the truth, the panels differ only through the fitted initial marginal. \emph{B:} the gauge $\gauge_{12}(t)$ that each of the two gauge-corrected methods uses for this trial, against the optimal gauge \cref{eq:optimalgauge} at its own marginals (thick, light). \emph{C:} gauge excess $\mathcal{J}(\gauge) - \mathcal{J}(\gauge^\star)$ over time, averaged over 64 test trials and all seeds, for the five posterior processes (log scale): the order-one correction is off by the closed-form amount $8 \kappa^2 \tr P \det P / s^2 \approx 0.13$ nats per unit time, the learned gauge by a tenth of it.}
    \label{fig:hopf}
\end{figure}

\paragraph{State-Dependent Noise: Stochastic Lotka--Volterra}
\textbf{Data.} The system is $\diff x = (\alpha x - \beta x y) \diff t + \sigma x \diff W^1$, $\diff y = (\delta x y - \gamma y) \diff t + \sigma y \diff W^2$ with $(\alpha, \beta, \delta, \gamma) = (2/3, 4/3, 1, 1)$, $\sigma = 0.15$ and initial state $(0.9, 0.2)$ perturbed by $\Norm(0, 0.05^2 I)$, as in \citet{course2023amortized} and \citet{bartosh2025sde}. Trajectories are simulated by Euler--Maruyama in log coordinates with step $10^{-3}$ on $[0, 40]$, about five periods of the deterministic cycle through the initial state, which has period $8.0$. Observations are the two populations with Gaussian noise of standard deviation $0.1$ at $N = 30$ uniformly random times per trial; 2048 training and 512 test trials. In \emph{log coordinates} the same trajectories are transformed to $(\ell_1, \ell_2) = (\log x, \log y)$; by It\^o's lemma the diffusion coefficient is then the constant $\sigma I$ and the drift coefficient is $(\alpha - \beta e^{\ell_2} - \sigma^2/2, \; \delta e^{\ell_1} - \gamma - \sigma^2/2)$, and the observation noise is added to the log-populations.

\textbf{Model.} The latent \gls{sde} has $d = 2$ and the shared prior drift network. The observation model is fixed at the identity, $C = I$ and $c = 0$, with the true noise level, so that the latent state is the (log-)population vector. In log coordinates the diffusion coefficient is fixed at its true value $\sigma I$, as $\vol = I$ is for the Lorenz system. In the original coordinates it is learned and diagonal, $\vol_{ii}(z) = \sqrt{z_i^2 + \epsilon^2} \, s_i(z_i)$ with $\epsilon = 0.1$ and $s_i$ the sigmoid-bounded network of \citet{bartosh2025sde} (one hidden layer of width 100, softplus activations, output in $(0, 1)$, initialized at $0.3$). The noise thus scales with the population as the true $\sigma z_i$ does, and the floor $\epsilon$, of the size of the observation noise, keeps the path-space \gls{kl} term bounded where the Gaussian posterior puts mass at non-positive populations; without it training diverges. Helmholtz-\gls{sde} and the symmetric gauge are not applicable in the original coordinates.

\textbf{Posterior.} The shared architecture; since $T = 40$, the posterior MLP receives $t / T$ alongside $t$.

\textbf{Training.} $3 \cdot 10^4$ iterations, batches of 128 trials, learning rate $10^{-3}$, 10 seeds.

\textbf{Metrics.} All three metrics are computed from 1024 sample paths of the learned prior process, started from the learned initial distribution and simulated with step $2 \cdot 10^{-3}$ over a horizon of $80$ time units, about eight cycles. In the original coordinates the learned diffusion coefficient keeps the populations positive only approximately, and sample paths that leave the box $[-1, 20]^2$ are excluded ($7$--$8\%$ on average, up to $31\%$ for one seed, most of them beyond the training horizon $T = 40$; clipping the populations at zero instead changes no reported value).

The \emph{period} is the location of the first nonzero maximum of the autocorrelation of the prey coordinate, averaged over paths. The \emph{lag} is the location, within one period, of the maximum of the cross-correlation between predator and prey. Both are reported as relative errors against the values for the true system computed by the same procedure, which are period $10.2$ and lag $2.9$ in log coordinates and $9.8$ and $1.5$ in the original coordinates; these periods exceed the deterministic $8.0$ because the process noise drives trajectories onto wider orbits of the conserved-quantity family, whose periods are longer. The \emph{autocorrelation error} is computed as for the Lorenz system, with lags in $[0, 10]$.

\textbf{Results.} \Cref{tab:lv} reports the metrics. In log coordinates the gauge excess is $0.215 \pm 0.006$ for the fixed gauge, $0.000$ for Helmholtz-\gls{sde} and $0.004 \pm 0.001$ for Gauge Matching: the posterior covariance is small relative to the curvature of the exponential drift terms, so the Jacobian at the mean already gives the optimal gauge, and the two gauge-corrected methods agree within seed noise. In the original coordinates Gauge Matching improves all three metrics over the fixed gauge on every seed, with a gauge excess of $0.003$ against $0.081$ nats per unit time; the Cholesky gauge and Gauge Matching with a diagonal $L_\phi$ (\cref{sec:ablationgauge}) confirm that the gain is the gauge.

\textbf{Remarks.}
\begin{itemize}[nosep,leftmargin=*]
    \item An independent draw of the true system scores $0.01$--$0.02$ on all three metrics, which is the noise floor of the estimators.
    \item The negative \gls{elbo} on the training trials, evaluated with the test estimator, lies about $30$ nats below the test value for every method (about $45$ with 512 training trials), so the held-out metrics are not limited by overfitting.
    \item In the original coordinates the learned relative noise level $s_i$ is $0.25$--$0.30$ at population $1$, against the true $\sigma = 0.15$, which widens the sample paths relative to the true system (\cref{fig:lv}).
\end{itemize}

\begin{table}[t]
    \centering
    \caption{Stochastic Lotka--Volterra system with sparse noisy observations, $d = 2$, in log coordinates, where the diffusion coefficient is state-independent and fixed at its true value, and in the original coordinates, where it is multiplicative and learned, and Helmholtz-\gls{sde} does not apply. Columns: relative error of the period of the cycle; relative error of the lag of the predator behind the prey; autocorrelation error as in \cref{tab:lorenz}, lags in $[0, 10]$, all for the learned prior process against the true system; gauge excess in nats per unit time as in \cref{tab:lorenz}, in the original coordinates from the state-dependent linear system of \cref{sec:optimalgaugeconsequences} estimated by Monte Carlo ($10^3$ samples at $100$ times on $64$ test trials). Mean $\pm 2$ standard errors over 10 seeds.}
    \label{tab:lv}
    \small
    \setlength{\tabcolsep}{4pt}
    \begin{tabular}{llcccc}
        \toprule
        Coord. & Posterior process & Period err.\ $\downarrow$ & Lag err.\ $\downarrow$ & Autocorr.\ err.\ $\downarrow$ & Gauge excess $\downarrow$ \\
        \midrule
        \multirow{3}{*}{log}
        & \gls{sde} Matching (square root) & $0.12 \pm 0.04$ & $0.09 \pm 0.05$ & $0.119 \pm 0.019$ & $0.215 \pm 0.006$ \\
        & Helmholtz-\gls{sde} ($\ell = 1$) & $0.09 \pm 0.03$ & $0.07 \pm 0.03$ & $0.100 \pm 0.015$ & $0.000 \pm 0.000$ \\
        & Gauge Matching & $0.07 \pm 0.02$ & $0.05 \pm 0.03$ & $0.093 \pm 0.018$ & $0.004 \pm 0.001$ \\
        \midrule
        \multirow{4}{*}{original}
        & \gls{sde} Matching (square root) & $0.22 \pm 0.05$ & $0.42 \pm 0.03$ & $0.191 \pm 0.022$ & $0.081 \pm 0.002$ \\
        & Cholesky gauge & $0.22 \pm 0.05$ & $0.44 \pm 0.04$ & $0.192 \pm 0.024$ & $0.082 \pm 0.004$ \\
        & Gauge Matching, diagonal $L_\phi$ & $0.13 \pm 0.05$ & $0.35 \pm 0.03$ & $0.139 \pm 0.022$ & $0.002 \pm 0.001$ \\
        & Gauge Matching & $0.13 \pm 0.03$ & $0.34 \pm 0.02$ & $0.140 \pm 0.015$ & $0.003 \pm 0.001$ \\
        \bottomrule
    \end{tabular}
\end{table}

\begin{figure}[t]
    \centering
    \includegraphics[width=\linewidth]{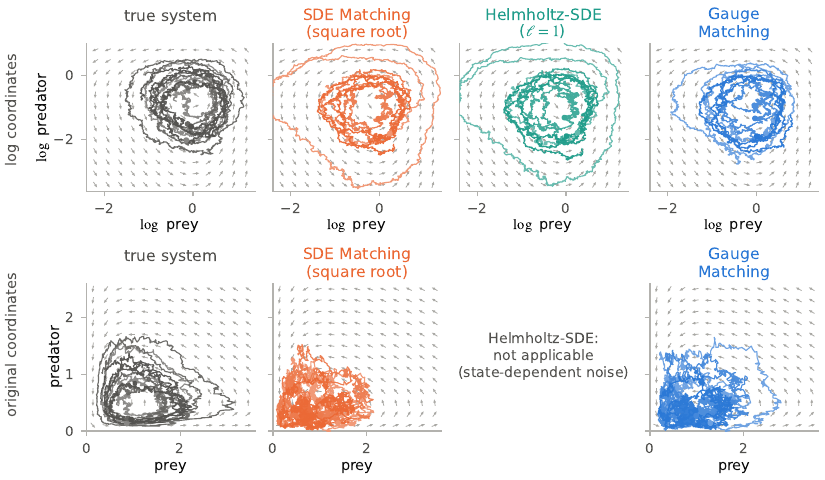}
    \caption{Learned prior dynamics of the stochastic Lotka--Volterra system. Drift field and three sample paths of the true system and of the models learned with each posterior process, in log coordinates with the diffusion coefficient fixed at its true value (top) and in the original coordinates with a learned state-dependent diffusion coefficient (bottom). In log coordinates all three posterior processes learn a drift field that sustains the cycle. In the original coordinates the learned multiplicative noise is larger than the true one, which widens the sample paths, and the cycle of the fixed gauge is smaller and less regular than that of the learned gauge.}
    \label{fig:lv}
\end{figure}

\paragraph{State-Dependent Noise in Higher Dimension}
To test the state-dependent case beyond $d = 2$, where the optimal gauge is not a Jacobian but the solution of the $d(d-1)/2$-dimensional linear system of \cref{sec:optimalgaugeconsequences}, we use the rotating Ornstein--Uhlenbeck process of the dimension sweep with a multiplicative diffusion coefficient.

\textbf{Data.} $\diff Z = (-\alpha I + \omega K) Z \diff t + \sigma \operatorname{diag}\bigl(\sqrt{1 + Z_i^2}\bigr) \diff W$ in $d = 8$ with $\alpha = 0.5$, $\omega = \pi$, $\sigma = 0.6$ and $K = (G - G^\top)/\|G - G^\top\|_2$ for a standard normal $G$, which for the system used rotates four planes at angular velocities $\pi |\lambda_j(K)| = 3.14, 1.70, 0.69, 0.02$. Since $Z^\top K Z = 0$, the process is stationary for $\sigma^2 < 2\alpha$ with second moment $s_0^2 = \sigma^2 / (2\alpha - \sigma^2) = 0.5625$ per coordinate, and $Z_0 \sim \Norm(0, s_0^2 I)$. Trajectories are simulated on $[0, 5]$ by Euler--Maruyama with step $10^{-3}$; all $d$ coordinates are observed through the identity every $\Delta t = 0.25$ ($N = 21$) with Gaussian noise of standard deviation $0.3$; 1024 training, 2048 test trials and 4096 independent reference trajectories, one system $K$ shared by all methods and seeds.

\textbf{Model.} Two protocols. \emph{Known generative model:} drift $A z$, diffusion coefficient $\sigma \operatorname{diag}(\sqrt{1 + z_i^2})$, initial law $\Norm(0, s_0^2 I)$ and observation model $C = I$, $c = 0$ fixed at their true values, so that only the posterior process is learned, as for the Stuart--Landau oscillator. \emph{Learned generative model:} the shared prior drift network, the state-dependent diagonal diffusion network of \citet{bartosh2025sde}, $\vol_{ii}(z) = 2 \operatorname{sigmoid}(\mathrm{MLP}_i(z_i))$ with one hidden layer of width 100 and softplus activations, a Gaussian initial law with learned mean and diagonal covariance, and the observation model fixed at the identity with the true noise level, as for the Lotka--Volterra system. Helmholtz-\gls{sde} and the symmetric gauge do not apply in either protocol.

\textbf{Posterior.} The shared architecture; four posterior processes: \gls{sde} Matching (diagonal $L_\phi$), the Cholesky gauge (full $L_\phi$), Gauge Matching (full $L_\phi$ and the 28 entries of $\gauge_\phi$) and Gauge Matching with a diagonal $L_\phi$ (\cref{sec:ablationgauge}).

\textbf{Training.} $3 \cdot 10^4$ iterations, batches of 256 trials, learning rate $10^{-3}$, 5 seeds.

\textbf{Metrics.} Negative \gls{elbo} and latent log-likelihood as in \cref{tab:lorenz}; the gauge excess with $\gauge^\star$ from the linear system $Q \gamma = b$ of \cref{sec:optimalgaugeconsequences}, whose coefficients and $\mathcal{J}$ are estimated from the same Monte Carlo samples as above, the residual $r_0$ including the divergence term $\frac{1}{2} \nabla \cdot \Sigma$ of \cref{eq:driftwhitened}. For the learned generative model also the autocorrelation error of 2048 prior sample paths over lags in $[0, 2]$ against the reference trajectories, the relative root-mean-square error of the learned diffusion coefficient against $\sigma \sqrt{1 + z_i^2}$ over $10^4$ stationary states, and the rotation rates of the learned drift, the imaginary parts of the eigenvalues of its Jacobian averaged over stationary states, against $\pi |\lambda_j(K)|$.

\textbf{Results.} \Cref{tab:rousd}. With the generative model known, Gauge Matching gains $6.6 \pm 0.8$ nats of negative \gls{elbo} over \gls{sde} Matching and $6.9 \pm 0.6$ over the Cholesky gauge, paired over seeds and on every seed. The gain is in the path-space \gls{kl} term ($94$--$95$ for the fixed gauges against $84$--$85$ for the learned ones), of which the gauge-corrected models give $1$--$2$ nats back in the reconstruction term ($53$--$54$ against $55$) through a slightly wider posterior, whose latent log-likelihood is accordingly $0.05$ lower. The gauge excess of the fixed gauges is $1.29$--$1.40$ nats per unit time against $0.041$--$0.049$ for the learned ones, whose $\|\gauge_\phi\|_F = 5.2$--$5.3$ matches $\|\gauge^\star\|_F = 5.3$ at their own marginals; the optimal gauge at the fixed gauges' marginals has the same norm, so the fixed gauges miss a rotation of that size.

With the generative model learned, the picture is the same at a smaller scale: a paired gain of $4.8 \pm 0.3$ nats over \gls{sde} Matching and $5.1 \pm 0.5$ over the Cholesky gauge ($8$ in the path-space term, $3$ given back in the reconstruction term), and a gauge excess of $0.85$--$0.88$ against $0.040$--$0.046$. The learned dynamics improve with the gauge: the autocorrelation error of the prior process drops from $0.072$ to $0.040$ (Monte Carlo floor $0.029$), and the rotation rates of the learned drift are $2.94, 1.60, 0.61$ against the true $3.14, 1.70, 0.69$ for the learned gauges but $2.55$--$2.58, 1.41, 0.54$ for the fixed ones, which absorb part of the rotation into a more contractive drift, as in the dimension sweep. The learned diffusion coefficient is the same for all four posterior processes and only approximately right (relative root-mean-square error $0.15$): the sigmoid network captures the level of the noise at $z_i = 0$ but underestimates its growth with $|z_i|$, equally for every gauge. Diagonal and full $L_\phi$ agree within a nat in both protocols, as for the Stuart--Landau and Lotka--Volterra systems (\cref{sec:ablationgauge}).

\textbf{Remarks.}
\begin{itemize}[nosep,leftmargin=*]
    \item The negative \gls{elbo} excludes test trials whose total exceeds $10^3$ nats (the criterion of \cref{tab:hopf}, applied to the total instead of the path-space \gls{kl} term): with $d = 8$ and a wide posterior, the amortized posterior is far off for $0$--$3$ of the $2048$ test trials per seed (one trial of the Cholesky gauge with a learned model at $9 \cdot 10^4$ nats), which dominates the plain mean without changing the ranking; the plain means are $158.5$, $156.3$, $149.0$, $148.4$ (known) and $150.5$, $159.7$, $145.7$, $146.1$ (learned) in the row order of \cref{tab:rousd}. The path-space and reconstruction terms quoted above are plain means and sum to the plain negative \gls{elbo}.
    \item Since the rotation is fully observed ($C = I$, eight observations per period of the fastest plane), the learned prior recovers most of it, unlike in the partially observed dimension sweep; the gain over the fixed gauges is smaller than with the known model ($4.8$ against $6.6$ nats) because the learned drift co-adapts to the gauge it is trained with.
    \item The state-dependent optimal gauge is never computed during training, only for the gauge-excess metric.
\end{itemize}

\begin{table}[t]
    \centering
    \caption{Rotating Ornstein--Uhlenbeck process in $d = 8$ with multiplicative noise, where Helmholtz-\gls{sde} and the symmetric gauge do not apply. Top: generative model fixed at the truth, posterior process learned. Bottom: prior drift, diffusion coefficient and initial law learned jointly with the posterior process. Columns as in \cref{tab:lorenz} without the time, with the gauge excess evaluated with the state-dependent optimal gauge of \cref{sec:optimalgaugeconsequences} and, for the learned model, the autocorrelation error over lags in $[0, 2]$; the negative \gls{elbo} excludes held-out trials above $10^3$ nats ($0$--$3$ of $2048$ per seed). Mean $\pm 2$ standard errors over 5 seeds.}
    \label{tab:rousd}
    \small
    \setlength{\tabcolsep}{3pt}
    \begin{tabular}{llcccc}
        \toprule
        Model & Posterior process & Neg.\ \gls{elbo} $\downarrow$ & Latent LL $\uparrow$ & Autocorr.\ err.\ $\downarrow$ & Gauge excess $\downarrow$ \\
        \midrule
        \multirow{4}{*}{known}
        & \gls{sde} Matching (square root) & $154.8 \pm 0.8$ & $-1.03 \pm 0.04$ & --- & $1.40 \pm 0.02$ \\
        & Cholesky gauge & $155.1 \pm 0.8$ & $-1.04 \pm 0.04$ & --- & $1.29 \pm 0.02$ \\
        & Gauge Matching & $148.2 \pm 0.8$ & $-1.09 \pm 0.03$ & --- & $0.049 \pm 0.002$ \\
        & Gauge Matching, diagonal $L_\phi$ & $147.3 \pm 0.5$ & $-1.08 \pm 0.03$ & --- & $0.041 \pm 0.001$ \\
        \midrule
        \multirow{4}{*}{learned}
        & \gls{sde} Matching (square root) & $150.5 \pm 0.5$ & $-1.05 \pm 0.02$ & $0.072 \pm 0.003$ & $0.88 \pm 0.02$ \\
        & Cholesky gauge & $150.8 \pm 0.2$ & $-1.06 \pm 0.01$ & $0.071 \pm 0.003$ & $0.85 \pm 0.02$ \\
        & Gauge Matching & $145.7 \pm 0.3$ & $-1.06 \pm 0.01$ & $0.040 \pm 0.002$ & $0.046 \pm 0.001$ \\
        & Gauge Matching, diagonal $L_\phi$ & $145.4 \pm 0.2$ & $-1.06 \pm 0.02$ & $0.040 \pm 0.002$ & $0.040 \pm 0.002$ \\
        \bottomrule
    \end{tabular}
\end{table}

\paragraph{Ablation: Gauge versus Covariance}
\label{sec:ablationgauge}
Gauge Matching differs from \gls{sde} Matching in two respects: the covariance factor $L_\phi$ is a full lower-triangular matrix rather than a diagonal one, and the gauge $\gauge_\phi$ is learned rather than fixed at zero. To attribute the improvement, we add the missing corner of the $2 \times 2$ design, \emph{Gauge Matching with a diagonal $L_\phi$}: the marginals of \gls{sde} Matching ($d$ means and $d$ positive standard deviations) together with the $d(d-1)/2$ entries of a learned skew-symmetric $\gauge_\phi$, all read off the same posterior network; with $L_\phi$ diagonal, the triangular solves in \cref{eq:driftwhitened} are divisions. The Cholesky gauge of the main tables is the other corner, a full $L_\phi$ with $\gauge \equiv 0$.

\textbf{Data, Model, Training.} Identical to the rows the variant joins: the Lorenz protocol of \cref{tab:lorenz}, the Stuart--Landau protocol of \cref{tab:hopf} ($\kappa = 3$), the Lotka--Volterra protocol of \cref{tab:lv} in the original coordinates, the two protocols of \cref{tab:rousd}, the identifiable dimension sweep of \cref{fig:dimsweepident}, and the conjugate setting at $\sigma = 0.3$, $\omega = 2\pi$, $N \in \{5, 10, 20, 50\}$ with the $5 \cdot 10^4$-iteration budget of the two gauge-corrected models. The variant has the parameter count of the Cholesky gauge ($54\,141$ at $d = 4$ against $53\,535$ for \gls{sde} Matching and $54\,747$ for Gauge Matching).

\textbf{Results.} On the Lorenz system the variant reaches a negative \gls{elbo} of $55.3 \pm 1.9$ nats (one seed at $58.8$, the others $53.5$--$55.6$), against $77.9$ for \gls{sde} Matching, $75.5$ for the Cholesky gauge and $53.2$ for Gauge Matching: of the $24.7$ nats between \gls{sde} Matching and Gauge Matching, the gauge accounts for $22.5$ and the covariance shape for $2.2 \pm 1.8$, paired over seeds and entirely in the path-space \gls{kl} term. Its gauge excess is $0.022 \pm 0.001$. On the Stuart--Landau oscillator the variant is within $0.4 \pm 0.2$ nats of Gauge Matching, the difference lying in the reconstruction term, with a gauge excess of $0.012 \pm 0.003$, while the Cholesky gauge stays with the fixed gauges at $64.3$. In the original coordinates of the Lotka--Volterra system, where the diffusion coefficient is learned and state-dependent, the Cholesky gauge scores like \gls{sde} Matching and the variant like Gauge Matching on all three metrics (paired differences within $0.02$ over 10 seeds; \cref{tab:lv}), the variant beating the Cholesky gauge on every seed.

In the conjugate setting the picture differs: the gap of the variant is $2.4$, $2.6$, $2.9$ and $8.6$ nats at $N = 5, 10, 20, 50$, against $0.8$, $0.8$, $1.2$, $3.9$ for Gauge Matching and $19$--$26$ for \gls{sde} Matching; the mean part of the gap is the same for the two learned-gauge models, and the whole difference is in the covariance part ($2.3$--$7.9$ against $0.7$--$3.0$).

\textbf{Remarks.} The smoother covariance of the rotating Ornstein--Uhlenbeck process is anisotropic in the fixed frame, and a diagonal $L_\phi$ can rotate its samples but not its covariance, which is why the gap is largest in the conjugate setting and at large $N$, where the posterior is narrowest and most anisotropic; on the nonlinear systems the learned prior and observation model absorb most of the anisotropy. The gauge head is insensitive to its initialization and learning rate: on the Lorenz system, initializing it at the library default instead of zero, or training it with a learning rate of $0.3$ or $3$ times that of the rest of the network, changes the negative \gls{elbo} of Gauge Matching by $+0.3$ to $+0.6$ nats (3 seeds, paired, within $2$ standard errors of zero) and leaves the gauge excess at $0.045$--$0.049$ and $\|\gauge_\phi\|_F$ within $0.5\%$ of $\|\gauge^\star\|_F$.

\paragraph{Motion Capture}
\textbf{Data.} We follow \citet{li2020scalable} and \citet{bartosh2025sde} exactly: the 50-dimensional recordings of subject 35 of \citet{gan2015deep} with the preprocessing of \citet{wang2007gaussian}; 16 training, 3 validation and 4 test sequences of 300 observations each; $d = 6$; and the architectures and hyperparameters of \citet{li2020scalable}.

\textbf{Training.} 10 random seeds instead of the default 5.

\textbf{Metrics.} The test mean squared error is defined as in \citet[Table 2]{bartosh2025sde}; Gauge Matching reaches $4.52 \pm 0.29$ (mean $\pm$ standard deviation over the 10 seeds), while the values quoted for \gls{sde} Matching ($4.50 \pm 0.32$) and the latent \gls{sde} of \citet{li2020scalable} ($4.03 \pm 0.20$) are those reported by \citet{bartosh2025sde} under the same protocol.

\textbf{Remarks.} With 300 nearly noise-free observations per sequence the posterior concentrates on the trajectory and the path-space \gls{kl} term is small for every gauge, so the gauge has little to contribute; the experiment checks that the additional $d(d-1)/2 = 15$ gauge outputs cost nothing in accuracy.

\end{document}